%% file: main_paper_arxiv.tex
\documentclass[11pt]{article}

\usepackage[margin=1in]{geometry}
\usepackage{setspace}
\usepackage{tgtermes}
\usepackage{newtxtext}
\usepackage{newtxmath}
\usepackage{bm}

\usepackage{amsmath, amssymb, amsthm}
\usepackage{mathtools}

\usepackage{graphicx}
\usepackage{booktabs}
\usepackage{multirow}
\usepackage{subcaption}
\usepackage{algorithm}
\usepackage{algorithmic}

\usepackage[round]{natbib}
\usepackage[colorlinks=true,
            linkcolor=blue!60!black,
            citecolor=blue!60!black,
            urlcolor=blue!60!black]{hyperref}

\theoremstyle{plain}
\newtheorem{theorem}{Theorem}
\newtheorem{lemma}[theorem]{Lemma}
\newtheorem{corollary}[theorem]{Corollary}

\theoremstyle{definition}
\newtheorem{definition}[theorem]{Definition}
\newtheorem{example}[theorem]{Example}
\newtheorem{assumption}[theorem]{Assumption}

\theoremstyle{remark}

\newcommand{\E}{\mathbb{E}}
\newcommand{\dis}{\mathrm{d}}

\title{Selection of the Best Policy under Fairness Constraints for Subpopulations}

\author{
  Tingyu Zhu
  \quad
  Yuhang Wu
  \quad
  Zeyu Zheng\\
  \normalsize Department of Industrial Engineering and Operations Research \& 
  Berkeley Artificial Intelligence Research Lab\\ University of California Berkeley\\
}

\date{\today}  

\begin{document}
\maketitle

\begin{abstract}
Many high-stakes decisions in health care, public policy, and clinical development require committing to a single policy that will be applied uniformly across a heterogeneous population. Regulatory and fairness standards sometime requires that the chosen policy performs adequately in every pre-specified subpopulation, not only on average. We formalize this as a Selection of the Best with Fairness Constraints (SBFC) problem, in order to identify the policy with the highest average performance among those policies that meet a minimum per-subpopulation threshold. We establish an instance-specific lower bound on sample complexity of the SBFC problem. We then develop a Track-and-Stop with Constraints on Subpopulation (T-a-S-CS) algorithm that achieves the lower bound asymptotically. We extend the framework to general closed-set and penalty-based fairness specifications with matching guarantees. Numerical experiments and a case study using the International Stroke Trial demonstrate substantial efficiency gains over policy-level allocation baselines.

\end{abstract}


\input{sections/1_introduction/intro.tex}

\input{sections/2_problem/1_sbfc_problem.tex}

\input{sections/2_problem/2_sample_algorithms.tex}

\input{sections/2_problem/3_optimality_feasibility.tex}
\input{sections/2_problem/4_lower_bound.tex}

\section{Track-and-Stop with Constraints on Subpopulation}
\label{sec: algorithm}

In this section, we introduce the T-a-S-CS algorithm, a $\delta$-PAC algorithm  that achieves the asymptotic lower bound given in equation (\ref{eq:lower_bound})  of Theorem \ref{thm:lower_bound}. The rest of this section is organized as follows. Section \ref{subsec:counter_set} solves the internal optimization problem in (\ref{eq:T_star})—that is, finding the ``hardest-to-distinguish" instance in the counter set $\mathcal{A}(\boldsymbol{\mu})$ for fixed weights $\mathbf{w}$. Section \ref{subsec: opt_weights} then solves the external optimization problem to derive the optimal weights $\mathbf{w}^*$. Section \ref{subsec:sample_alloc} presents a C-tracking rule that translates $\mathbf{w}^*$ into a sample allocation strategy for T-a-S-CS. Section \ref{subsec:stop_rule} provides the stopping rule and summarizes the complete algorithm. Finally, Section \ref{subsec:sample_complexity} proves that T-a-S-CS achieves the asymptotic lower bound (\ref{eq:lower_bound}). 

\input{sections/3_algorithm/1_counter_set.tex}
\input{sections/3_algorithm/2_optimal_weights.tex}
\input{sections/3_algorithm/3_sample_allocation.tex}
\input{sections/3_algorithm/4_stopping_rule.tex}
\input{sections/3_algorithm/5_sample_complexity.tex}

\section{Generalizations of the SBFC problem}
\label{sec:extensions}

\input{sections/4_extensions/1_generalized_constraints.tex}
\input{sections/4_extensions/2_penalty_fairness.tex}
\input{sections/4_extensions/3_conclusion.tex}

\input{sections/5_numerical/1_aspects.tex}
\input{sections/5_numerical/2_experiments.tex}
\input{sections/5_numerical/2_experiments/1_asymptotic.tex}
\input{sections/5_numerical/2_experiments/2_robustness.tex}
\input{sections/5_numerical/2_experiments/3_extension_comparison.tex}
\input{sections/5_numerical/3_real_demonstration.tex}

\bibliographystyle{plainnat}  
\bibliography{references}

\appendix

\section{Proof of the Lower bound Results}
\label{appendix:lowerbound}
\input{appendices/A_lowerbound/1_theorem_proof.tex}
\input{appendices/A_lowerbound/2_extension_proof.tex}
\input{appendices/A_lowerbound/3_extension2_proof.tex}

\section{Proof of the Lemmas of Counter Set Calculation}
\label{appendix:counterset}
\input{appendices/B_counterset/1_lemma_proof.tex}
\input{appendices/B_counterset/2_extension_proof.tex}
\input{appendices/B_counterset/3_extension2_proof.tex}

\section{Proof of The Lemmas in Section \ref{subsec: opt_weights}}
\label{appendix:lemmas}
\input{appendices/C_lemmas/1_weight_infeas.tex}
\input{appendices/C_lemmas/2_subgradient.tex}
\input{appendices/C_lemmas/3_optimization_convergence.tex}

\section{Proof of the Sample Complexity Convergence Theorems}
\label{appendix:convergence}
\input{appendices/D_convergence/1_theorem_proof.tex}
\input{appendices/D_convergence/2_lemma_proof.tex}
\input{appendices/D_convergence/3_extension1_cor_proof.tex}
\input{appendices/D_convergence/4_extension2_cor_proof.tex}

\section{Numerical Experiment Supplementary Materials}
\label{appendix:numerical}
\input{appendices/E_numerical/1_gaussian_sbfc.tex}
\input{appendices/E_numerical/2_bernoulli_subproblem.tex}
\input{appendices/E_numerical/3_variance_constraint.tex}

\input{appendices/E_numerical/4_soft_constraint.tex}
\input{appendices/E_numerical/5_asymptotic_details.tex}
\input{appendices/E_numerical/6_extension_details.tex}

\section{Real-World Application Supplementary Materials}
\input{appendices/F_real_scenario/IST_supplementary.tex}

\end{document}

%% file: sections/1_introduction/intro.tex
\section{Introduction}

Classical \textit{Selection of the Best} problems aim at identifying and deploying the best policy, one that has the highest expected performance under a certain metric. For applications where this selected policy applies to a general population, it can happen that the policy incurs difference performance on different subpopulations. A policy that achieves the best performance for one subpopulation could  perform badly for another subpopulation. In an ideal selection and deployment situation, one could select the best policy correspondingly for each different subpopulation and deploy different policies for different subpopulations. However, a number of high-stakes operational policy decisions require choosing a single policy that applies universally to the entire population. 

Several considerations would cause the decision maker to only deploy one policy, instead of different policies to different populations. For example, deploying a policy at scale can entail substantial fixed costs in infrastructure, training, and supply chain logistics, and replicating these investments for each subpopulation can multiply expenses beyond available budgets. In some other situations, subpopulation status is unobservable or unavailable at the moment of deployment (not at the experiment), leaving a single policy rule as the only implementable option. In addition, a uniform policy is often easier to communicate, audit, and defend to oversight bodies, which is particularly valuable in high-stakes settings where transparency and procedural consistency are themselves operational requirements. 

In the meanwhile, regulatory and fairness standards may demand that the selected and deployed policy performs adequately well above some given threshold for each subpopulation. The selection objective is therefore to optimize overall population-average performance subject to baseline adequacy requirements on each subpopulation. In this work, we formulate and refer to this problem setting as a  \emph{Selection of the Best with Fairness Constraints} (SBFC) problem. Given $K$ candidate policies evaluated across $L$ subpopulations with stochastic outcomes, the goal is to find the policy that maximizes weighted population-level performance subject to the constraint that every subpopulation's expected outcome remains acceptable. Formally, the policy's subpopulation-level performance must satisfy a pre-specified feasibility requirement. The policy's performance in each subpopulation is unknown and must be learned sequentially from samples.

Methodologically, SBFC sits at the intersection of ranking and selection (R\&S), best-arm identification (BAI), and algorithmic fairness. Classical R\&S seeks to identify the best among a finite set of stochastic systems. We refer to \cite{hong2021review} as an overview. For work on fixed-precision paradigm that discusses guarantees on the Probability of Correct Selection (PCS) we refer to  \citep{bechhofer1954single, kim2001fully, ma2017efficient} among others. Within the fixed-precision stream, some recent advances include large-scale parallel procedures \citep{zhong2022knockout}, two-stage refinements \citep{fan2025first}, screening with functional properties \citep{eckman2022plausible}, and extensions to input uncertainty \citep{wu2024data}. Alternatively, the fixed-budget paradigm formulates the problem as optimal resource allocation, and has progressed with parallel knockout-tournament \citep{hong2022solving}, context-dependent budget allocation \citep{li2024efficient}, and OCBA-based selection under parametric uncertainty \citep{kim2025selection}, among others. Related budget-allocation ideas have also been adapted beyond classical R\&S, for instance, OCBA-style exploration in reinforcement learning \citep{zhu2024uncertainty}.

A growing stream of modern R\&S literature has also pursued IZ-free procedures to eliminate reliance on worst-case configurations \citep{fan2016indifference}. Parallel to these efforts, the BAI literature has established instance-optimal algorithms that provide fixed-confidence guarantees without requiring a pre-specified minimum gap \citep{garivier2016optimal, kaufmann2016complexity}. We formulate SBFC using the fixed-confidence BAI framework to advance the fundamental goals of fixed-precision R\&S. This approach replaces the restrictive IZ assumption with BAI's instance-optimal stopping architectures. Consequently, we provide an IZ-free, gap-dependent algorithmic framework for constrained ranking and selection.

Our work is related to but different from work in the constrained simulation optimization or constrained R\&S literature. When secondary performance measures must be evaluated, constrained R\&S procedures have been developed \citep{andradottir2010fully}, alongside penalty-based approaches that adaptively soften hard constraints \citep{park2015penalty}. Parallel efforts in constrained BAI address multi-dimensional feasibility \citep{katzsamuels2019topfeasible} and grouped thresholds \citep{sundaram2024grouped}. The critical distinction between these existing works and our SBFC work lies in the subpopulation setting. Traditional constrained frameworks typically evaluate multiple metrics for a single system, meaning the sampling budget is still allocated over a one-dimensional population space of candidate policies; no subpopulations are considered. In SBFC, however, the constraints are defined at the subpopulation level. Samples must be allocated across a two-dimensional policy-by-subpopulation space, necessitating a globally coupled optimization that classical methods cannot directly accommodate.

The fairness criteria in SBFC also structurally diverge from those typically studied in online learning and selection. From a fairness perspective, algorithmic fairness in multi-armed bandits often focuses on procedural fairness; it constrains the {sampling policy itself} to ensure adequate exposure rates for different arms \citep{wang2021fairness}, or enforces balanced group-level representation in online secretary and prophet problems \citep{correa2025fairness}. In contrast, SBFC draws inspiration from structural fairness criteria in supervised learning \citep{hardt2016equality, barocas2023fairness}, shifting the focus from the exploration process to the final deployed outcome. The sampling algorithm in SBFC is free to allocate samples, provided the final outcome policy satisfies the subpopulational fairness constraints.


Our specific contributions are summarized as follows:

{First}, we introduce and formalize the SBFC problem, and we establish an instance-specific information-theoretic lower bound on the expected sample complexity of any $\delta$-probably-approximately-correct ($\delta$-PAC) algorithm for user-specified risk $\delta$.

{Second}, we develop the Track-and-Stop with Constraints on Subpopulations (T-a-S-CS) algorithm for the SBFC problem. We prove that T-a-S-CS is $\delta$-PAC, and asymptotically matches the lower bound. Therefore, the T-a-S-CS algorithm is the first instance-optimal algorithm for sequential selection under fairness constraints.

{Third}, we generalize the framework along two practically important dimensions. The $\mathcal{M}$-SBFC extension replaces the coordinate-wise linear threshold with an arbitrary closed fairness set, accommodating variance-based fairness, ratio constraints, and quota-style requirements. The $\gamma$-SBFC extension converts hard constraints into Lagrangian penalties that trade off overall performance against subpopulation shortfalls. For both extensions, we establish matching lower bounds, present the corresponding versions of T-a-S-CS algorithms, and prove asymptotic optimality.

{Fourth}, we validate the framework numerically and in a real-world case study. Numerical experiments verify the predicted $O(\log(1/\delta))$ scaling, and demonstrate that at operational scale ($K$ up to $100$) the sample allocation of T-a-S-CS remains efficient. We also demonstrate the framework on the International Stroke Trial, a landmark study of antithrombotic therapy in acute ischaemic stroke, where the fairness constraint protecting elderly patients overturns the unconstrained ranking.

The remainder of the paper is organized as follows. Section~\ref{sec: problem and framework} formalizes the SBFC problem, introduces the sequential sampling framework, illustrates the optimality-feasibility trade-off through examples, and derives the instance-specific lower bound. Section~\ref{sec: algorithm} develops the T-a-S-CS algorithm and establishes its asymptotic optimality. Section~\ref{sec:extensions} presents the $\mathcal{M}$-SBFC and $\gamma$-SBFC extensions with the corresponding versions of T-a-S-CS algorithm, and their asymptotic optimality guarantees. Section~\ref{sec:computational_numerical} reports practical implementation strategies, and numerical experiments. All proofs and experiment details are provided in the appendix.

%% file: sections/2_problem/1_sbfc_problem.tex
\section{Problem and Framework}
\label{sec: problem and framework}

In this section, we introduce the Selection-of-the-Best with Fairness Constraints (SBFC) problem. We start by introducing the problem setting in section \ref{subsect:SBFC setting}. Then, in Section \ref{subsec:sample_recommend_algorithms},  we introduce the framework of sequential sample-recommendation algorithms for the SBFC problem. In Section
\ref{subsec:opt_feas_tradeoff}, we provide examples to demonstrate a key characteristic of the SBFC problem: an optimality-feasibility trade-off in sample allocation. In Section \ref{subsec:lowerbound}, we provide a universal lower bound on sample complexity of a class of algorithms that solves the SBFC problem.

\subsection{The SBFC problem}
\label{subsect:SBFC setting}

Suppose that the decision maker is given the number of competing policies $K\geq 2$ and the number of subpopulations $L$. Every time a policy $k$ is applied to subpopulation $l$, the decision maker obtains an observation of performance. The performance is a random variable $X_{k,l}$ with distribution $\mathcal{P}_{k,l}$ and expectation $\mathbb{E}X_{k,l}=\mu_{k,l}, k\in[K], l\in[L]$. Both $\mathcal{P}_{k,l}$ and $\mathbb{E}X_{k,l}=\mu_{k,l}, k\in[K], l\in[L]$ are unknown and need to be estimated.
We refer to $\mu_{k,l}$ as the expected performance of policy $k$ applied to subpopulation $l$, and $\mathbf{\mu}_{k}=(\mu_{k,1},\ldots,\mu_{k,L})$ as the vector of subpopulational expected performances of policy $k$. With a fixed vector $\mathbf{q}=(q_1,\cdots,q_L)\in \mathbb{R}^L,~\mathbf{q}\geq 0$ representing the weight of the subpopulations, the populational expected performance of policy $k$ is defined as $\bar{\mu}_k = \sum_{l=1}^{L} q_l \mu_{k,l}, k\in[K]$. 

A standard selection-of-the-best problem aims to find the policy $k$ with the maximum expected populational performance, i.e., $k_{\text{SB}}=\arg\max_{k\in[K]}\bar{\mu}_{k}$. Additionally requiring the selected policy to be fair to all subpopulations, we introduce the concept of a feasible set $\mathcal{C}$ consisting of policies that satisfy the fairness constraint,
\begin{equation}\label{def:feasible}
    k\in \mathcal{C}= \{k\in[K]\big\vert \mu_{k,l}\ge C_{\text{min}},\forall~ l\in[L]\},
\end{equation}
where the constant $C_{\text{min}}$ is the minimum expected performance allowed for a policy in any of the sub-populations $l\in[L]$. With slight abuse of notation, we can equivalently write $\mathbf{\mu}_k\in\mathcal{C}$ for $k\in\mathcal{C}$, where $\mathcal{C}$ now stands for the feasible domain of $\mathbf{\mu}_k$ for policy $k$ to be in the feasible set $\mathcal{C}$. Our problem is to select the best policy $k^*$ in the feasible set $\mathcal{C}$, i.e., 
\begin{equation}\label{def:goal}
    k^*=\arg\max_{k\in\mathcal{C}}\bar{\mu}_{k}.
\end{equation}
Additionally, we define $k^*=0$ if there is no feasible policy, i.e., $\mathcal{C}=\emptyset$. The following assumption is imposed throughout the paper:
\begin{assumption}
    The distributions $\mathcal{P}_{k,l},k\in[K],l\in[L]$ belong to the same one-parameter exponential family, $\mathcal{P}:=\{\nu_\theta: d\nu_\theta/d\xi=\exp(\theta x-b(\theta))\}$.
\end{assumption}
Every probability distribution in a given one-parameter exponential family is entirely defined by its mean (\cite{cappe2013kullback}). Therefore, we may identify any instance of the SBFC problem with its matrix of means $\boldsymbol{\mu}\in\mathbb{R}^{K\times L}$. We hence use $\boldsymbol{\mu}$ to denote the problem instance throughout the paper.

%% file: sections/2_problem/2_sample_algorithms.tex
\subsection{Sample-Recommendation Algorithms}
\label{subsec:sample_recommend_algorithms}

To select $k^*$ as defined in (\ref{def:goal}), the decision maker can use a sequential sample-recommendation algorithm. At each time step $t$, the algorithm decides on a policy $A_t\in[K]$ and a subpopulation $I_t\in[L]$ to obtain sample from. The outcome $X_{A_t,I_t}\sim \mathcal{P}_{A_t,I_t}$ is then incorporated to update overall information. Upon reaching some stopping criterion specified by the decision maker, the algorithm stops sampling after time step $\tau$ and produces a recommendation for the best policy, $\hat{k}^*$. Considering the probabilistic guarantee of the final outcome, we focus on the {fixed-precision setting} (see \cite{hong2021review}) with risk level $\delta$. To formally introduce the fixed-precision setting in our selection problem, we first give the following definition of $\mathcal{S}$.

\begin{definition}
\label{def:S}
    Denote by $\mathcal{S}$ a set of instances that, each instance $\boldsymbol{\mu}$ in $\mathcal{S}$ and the feasible set $\mathcal{C}(\boldsymbol{\mu})$ satisfy \textit{either} of the following: 
    \begin{enumerate}
        \item There is a unique optimal feasible policy with all constraints strictly satisfied, i.e., $\exists~k^*(\boldsymbol{\mu})\in \mathcal{C}(\boldsymbol{\mu})$, such that $\mu_{k^*(\boldsymbol{\mu}),l}>C_{\text{min}},\forall~l\in[L]$; and  $\forall~k\in\mathcal{C}(\boldsymbol{\mu})$, $k\not=k^*(\boldsymbol{\mu})$, we have $\bar{\mu}_k<\bar{\mu}_{k^*(\boldsymbol{\mu})}$.
        \item All policies have at least one subpopulation constraint strictly violated. In this case we have  $\mathcal{C}(\boldsymbol{\mu})=\emptyset $ and $k^*(\boldsymbol{\mu})=0$.
    \end{enumerate}
\end{definition}

Definition \ref{def:S} imposes two key structural requirements on the instance set $\mathcal{S}$. First, the optimal policy $k^*(\boldsymbol{\mu})$ is always unique. Second, the strict inequalities in the definition ensure that the optimal policy satisfies feasibility constraints with positive margin ($\mu_{k^*,l}>C_{\text{min}}$ rather than $\mu_{k^*,l}\geq C_{\text{min}}$) and dominates other feasible policies with positive gap ($\bar{\mu}_{k^*}>\bar{\mu}_k$ rather than $\bar{\mu}_{k^*}\geq\bar{\mu}_k$). These gap conditions are strictly weaker than the classical indifference-zone (IZ) assumption \citep{hong2021review}. Instead of requiring the true best to lead by a pre-specified parameter $\delta > 0$, our setting allows the gap to be arbitrarily close to zero, effectively placing our problem in the indifference-zone-free (IZ-free) regime. We focus on instances in $\mathcal{S}$ throughout this paper.

We next give the definition of \textbf{$\delta$-probably approximately correct ($\delta$-PAC)} algorithms with respect to $\mathcal{S}$.

\begin{definition}
\label{def:delta-PAC}
   An algorithm is $\delta$-PAC if it gives a stopping time $\tau_\delta$ with respect to $\mathcal{F}_t$, a $\mathcal{F}_{\tau_\delta}$-measurable recommendation $\hat{k}^*_{\tau_\delta}\in\{0\}\cup [K]$, and $\forall \boldsymbol{\mu}\in\mathcal{S}$, 
\begin{equation}\label{def:delta_PAC}
\begin{aligned}
       & \mathbb{P}_ {\boldsymbol{\mu}}(\tau_\delta<+\infty)=1,\\
       & \mathbb{P}_ {\boldsymbol{\mu}}(\hat{k}^*_{\tau_\delta}\neq k^*(\boldsymbol{\mu}) )\leq \delta.
\end{aligned}
\end{equation}
\end{definition}

%% file: sections/2_problem/3_optimality_feasibility.tex
\subsection{The Optimality-Feasibility Trade-off of the SBFC Problem}
\label{subsec:opt_feas_tradeoff}

In this section, we provide an example of the SBFC problem to gain a deeper understanding of the optimality-feasibility trade-off in SBFC problems. We point out that requiring both optimality and feasibility of the selected outcome induces a trade-off for the sample allocation rule. The trade-off lies between allocating samples to verify optimality and allocating samples verify for feasibility.

\begin{example}
\label{example:2}
Suppose we set $K=3$ representing three different treatments, $L=2$ for two different subpopulations. We have $\mathbf{q}=(1/2,1/2)$, which gives $\bar{\mu}_k=\left(\mu_{k,1}+\mu_{k,2}\right)/2, ~k=1,2,3$. Suppose that $\mathcal{P}_{k,l}$ belongs to the Gaussian family with known variance $1$, and the minimum allowed subpopulational expected performance  $C_{\mathrm{min}}$ is set as $0$. The following Table \ref{tab:example_2} gives the expected performance values of the three treatments on two subpopulations. We set $0<\varepsilon<1$.
\begin{table}[ht]
    \centering
    \begin{tabular}{c|cc|c|c}
        \hline
         & Subpopulation $1$  & Subpopulation $2$ & Entire Population & Feasibility \\
        \hline
        Treatment A & $\mu_{1,1}=1$ & $\mu_{1,2}=1$ & $\bar{\mu}_1=1$ & Yes$(*)$ \\
        
        Treatment B & $\mu_{2,1}=1$ & $\mu_{2,2}=-\varepsilon$ & $\bar{\mu}_2=1/2-\varepsilon/2$ & No \\
        
        Treatment C & $\mu_{3,1}=4$ & $\mu_{3,2}=-1$ & $\bar{\mu}_3=3/2$& No \\
        \hline
    \end{tabular}
    \caption{Expected subpopulational performance and populational performance of treatments A, B and C, Example \ref{example:2}. We use $(*)$ to indicate the optimal feasible treatment.}
    \label{tab:example_2}
\end{table}
\end{example}

The interesting thing in Example \ref{example:2} is that it is not necessary for the algorithm to identify the feasible set $\mathcal{C}$ before finding $k^*=1$. Indeed, when $\varepsilon$ is close to $0$, an algorithm that fails to determine whether $\mu_{2,2}\geq0$ can still generate the correct recommendation $k^*=1$, because $\bar{\mu}_2<\bar{\mu}_1-1/2$. Therefore, treatment B can be excluded from the competing policies before the algorithm gains enough information to determine its feasibility. From this example we can see that a naive algorithm that attempts to find the feasible set $\mathcal{C}$ before searching for the best policy in $\mathcal{C}$ can be inefficient in general. Therefore, the {SBFC} problem is not just a straightforward combination of a standard selection-of-the-best problem and the identification of the feasible set.


In the rest of this work, we first provide an instance-specific lower bound on sample complexity that integrates the feasibility and optimality gaps. This lower bound therefore characterizes the difficulty of a SBFC problem. We then provide a sample-recommendation algorithm that is designed based on the lower bound. We demonstrate that the sample allocation rule of the algorithm implicitly addresses the optimality-feasibility trade-off, and achieves the lower bound asymptotically.

%% file: sections/2_problem/4_lower_bound.tex
\subsection{A Lower Bound on Sample Complexity}
\label{subsec:lowerbound}

In this section, we establish a lower bound on the stopping time $\tau_\delta$ for any $\delta$-PAC algorithm. We remark that,  $\delta$-PAC requires a probabilistic guarantee for the algorithmic outcome over all the instances in $\mathcal{S}$, while the lower bound we derive is instance-specific, i.e., a function of the specific instance $\boldsymbol{\mu}$. This instance-specific lower bound represents the best achievable sample complexity that can be achieved on a specific instance $\boldsymbol{\mu}$, by any algorithm that meets the universal requirement of $\delta$-PAC over $\mathcal{S}$.

We first introduce the counter set
\begin{equation}\label{def:alt}
    \mathcal{A}(\boldsymbol{\mu}):=\{{\boldsymbol{\lambda}}\in \mathcal{S}\big\vert k^*(\boldsymbol{\mu})\neq k^*({\boldsymbol{\lambda}})\},
\end{equation}
the set of instances in $\mathcal{S}$ where the optimal feasible policy is different from that in $\boldsymbol{\mu}$. The definition of $\mathcal{A}(\boldsymbol{\mu})$ originates from \cite{lai1985asymptotically}, which popularizes the use of changes of distributions to establish instance-specific lower bounds. 
We have the following theorem:

\begin{theorem}
    \label{thm:lower_bound}
Let $\delta\in(0,1)$. For any $\delta$-PAC algorithm and any instance $\boldsymbol{\mu}\in\mathcal{S}$, 
\begin{equation}\label{eq:lower_bound}
\begin{aligned}
    &\mathbb{E}_{\boldsymbol{\mu}}[\tau_\delta]\geq T^*(\boldsymbol{\mu})\mathrm{kl}(\delta,1-\delta),\\ 
    &\liminf_{\delta\to 0}\frac{\mathbb{E}_{\boldsymbol{\mu}}[\tau_\delta]}{\ln(1/\delta)}\geq T^*(\boldsymbol{\mu}),
\end{aligned}
\end{equation}
where $\mathrm{kl}(\delta,1-\delta)$ is the Kullback–Leibler (KL) divergence of two Bernoulli distributions of parameter $\delta$ and $1-\delta$, and
\begin{equation}\label{eq:T_star}
    \begin{aligned}
        &T^*(\boldsymbol{\mu})^{-1}\\
        =&\sup_{\mathbf{w}\in \Sigma_{K\times L}}\inf_{{\boldsymbol{\lambda}}\in \mathcal{A}(\boldsymbol{\mu})}\sum_{k\in[K]}\sum_{l\in[L]}w_{k,l}\mathrm{d}\left(\mu_{k,l},\lambda_{k,l}\right).\\
    \end{aligned}
\end{equation}
where $\Sigma_{K\times L}=\{\mathbf{w}\in \left(\mathbb{R}_+\cup \{0\}\right)^{K\times L}\big\vert w_1+\cdots+w_{KL}=1\}$, and $\mathrm{d}(\mu,\lambda)$ is the KL divergence of two distributions from the same one-parameter exponential family, parameterized by $\mu$ and $\lambda$ as mean values respectively.
\end{theorem}
The proof of Theorem \ref{thm:lower_bound} is provided in Section \ref{appendix:thm_lb_proof}. 

\noindent\textbf{Interpretation of $T^*(\boldsymbol{\mu})^{-1}$:} The quantity $T^*(\boldsymbol{\mu})$ measures the difficulty of the selection problem for instance $\boldsymbol{\mu}$. When some ${\boldsymbol{\lambda}}\in\mathcal{A}(\boldsymbol{\mu})$ is close to $\boldsymbol{\mu}$ under a weighted KL-divergence metric, $T^*(\boldsymbol{\mu})^{-1}$ becomes small (making $T^*(\boldsymbol{\mu})$ large) by definition (\ref{eq:T_star}). This aligns with the intuition that distinguishing $\boldsymbol{\mu}$ from nearby policies requires more time steps, indicating higher problem difficulty. Section \ref{subsec:counter_set} provides explicit characterization of $\mathcal{A}(\boldsymbol{\mu})$.

The optimal weight vector $\mathbf{w}^*$ achieving the supremum in (\ref{eq:T_star}) represents the optimal sampling proportions across policies and subpopulations (existence guaranteed by compactness). Specifically, $\mathbf{w}^*$ defines a metric that maximizes the worst-case gap between $\boldsymbol{\mu}$ and $\mathcal{A}(\boldsymbol{\mu})$, thereby guiding the sample allocation rule for the Track-and-Stop with Constraints on Subpopulations (T-a-S-CS) algorithm, which achieves asymptotically optimal sample complexity (presented in the next section).

%% file: sections/3_algorithm/1_counter_set.tex
\subsection{Calculating the Alternative Instance in the Counter Set}
\label{subsec:counter_set}

In building up the algorithm that allocates samples according to the optimal sampling proportion $\mathbf{w}^*$, we discuss how to solve the optimization problems in \eqref{eq:T_star}. 

We begin with the internal optimization problem, which is to solve for fixed $\mathbf{w}$,
\begin{equation}
 \label{eq:internal_opt}   f_{\boldsymbol{\mu}}^*(\mathbf{w})=\inf_{{\boldsymbol{\lambda}}\in \mathcal{A}(\boldsymbol{\mu})}\sum_{k\in[K]}\sum_{l\in[L]}w_{k,l}\mathrm{d}\left(\mu_{k,l},\lambda_{k,l}\right).
\end{equation}
Note that $\boldsymbol{\mu}=\{\mu_{k,l}\}_{k\in[K],l\in[L]}$ are considered as fixed parameters, and ${\boldsymbol{\lambda}}=\{\lambda_{k,l}\}_{k\in[K],l\in[L]}$ are the only decision variables. We now present the following lemma regarding problem (\ref{eq:internal_opt}).

\begin{lemma}
\label{lemma: counter_set_calculation}
The calculation of $f^*_{\boldsymbol{\mu}}(\mathbf{w})$ in (\ref{eq:internal_opt}) can be decomposed into optimization subproblems as follows:

If all policies are infeasible, i.e., $\mathcal{C}(\boldsymbol{\mu})=\emptyset$, which gives $k^*(\boldsymbol{\mu})=0$, we have $f_{\boldsymbol{\mu}}^*(\mathbf{w})=\min\left\{f_{\boldsymbol{\mu}}^{1}(\mathbf{w}),\ldots,f_{\boldsymbol{\mu}}^{K}(\mathbf{w})\right\}$, where $\forall~ k\in[K]$,

\begin{equation}
\label{eq:f_star_k_0}
\begin{aligned}
f_{\boldsymbol{\mu}}^k(\mathbf{w})=\inf_{\lambda_{k,1},\ldots,\lambda_{k,L}\in\mathbb{R}}&~\sum_{l\in[L]}w_{k,l}\mathrm{d}(\mu_{k,l},\lambda_{k,l}),\\
s.t.\quad&\lambda_{k,l}\geq C_{\min},~\forall~ l\in[L].
\end{aligned}
\end{equation}

If $\mathcal{C}(\boldsymbol{\mu})\not=\emptyset$, w.l.o.g. assume $k^*(\boldsymbol{\mu})=1$, we have $f^*_{\boldsymbol{\mu}}(\mathbf{w})=\min\left\{f^\mathrm{opt}_{\boldsymbol{\mu}}(\mathbf{w}),f^\mathrm{feas}_{\boldsymbol{\mu}}(\mathbf{w})\right\}$.
Further,
$f^\mathrm{opt}_{\boldsymbol{\mu}}(\mathbf{w})=\min\left\{f^{\mathrm{opt}(2)}_{\boldsymbol{\mu}}(\mathbf{w}),\ldots,f^{\mathrm{opt}(K)}_{\boldsymbol{\mu}}(\mathbf{w})\right\}$, where $\forall~k\in[K]$ and $k\geq 2$,
\begin{equation}
\label{eq:f_mu_opt}
\begin{aligned}
&f^{\mathrm{opt}(k)}_{\boldsymbol{\mu}}(\mathbf{w})\\
=&\inf_{\lambda_{i,l}\in\mathbb{R}, i\in\{1,k\}, l\in[L]}
\left(\sum_{i\in\{1,k\}}\sum_{l\in[L]}w_{i,l}\mathrm{d}(\mu_{i,l},\lambda_{i,l})\right),\\
&\text{s.t.}\quad \lambda_k=\sum_{l\in[L]}q_l\lambda_{k,l}>\lambda_1= \sum_{l\in[L]}q_l\lambda_{1,l},\\
&\phantom{\text{s.t.}}\quad \lambda_{k,l}\geq C_{\mathrm{min}},~\forall~l\in[L],
\end{aligned}
\end{equation}
and
\begin{equation}
\label{eq:f_infeas}
f^\mathrm{feas}_{\boldsymbol{\mu}}(\mathbf{w})=\min_{l\in[L]}w_{1,l}\mathrm{d}(\mu_{1,l},C_\mathrm{min}).
\end{equation}
\end{lemma}

The proof of Lemma \ref{lemma: counter_set_calculation} is given in Appendix \ref{appendix:thm_cs_proof}. The optimization problems (\ref{eq:f_star_k_0}), (\ref{eq:f_mu_opt}), and (\ref{eq:f_infeas}) correspond to three distinct scenarios for alternative instances in $\mathcal{A}(\boldsymbol{\mu})$: When $\mathcal{C}(\boldsymbol{\mu})=\emptyset$, problem (\ref{eq:f_star_k_0}) identifies alternative instances where an infeasible policy becomes feasible. When $\mathcal{C}(\boldsymbol{\mu})\neq\emptyset$, problem (\ref{eq:f_mu_opt}) identifies alternative instances where a different policy becomes optimal (while remaining feasible), and problem (\ref{eq:f_infeas}) identifies alternative instances where the current optimal policy becomes infeasible.

We now interpret Lemma \ref{lemma: counter_set_calculation}. This lemma identifies the instance ${\boldsymbol{\lambda}}^*(\mathbf{w})\in\mathcal{A}(\boldsymbol{\mu})$ that is closest to $\boldsymbol{\mu}$ under the weighted distance defined by $\mathbf{w}$. Intuitively, ${\boldsymbol{\lambda}}^*(\mathbf{w})$ is the ``most confusing" alternative instance to $\boldsymbol{\mu}$. Specifically, suppose the true instance is $\boldsymbol{\mu}$ and we sample according to proportions $\mathbf{w}$. Then ${\boldsymbol{\lambda}}^*(\mathbf{w})$ is the alternative instance that: (a) yields a different best policy, i.e., $k^*(\boldsymbol{\mu})\not=k^*({\boldsymbol{\lambda}}^*(\mathbf{w}))$, and (b) is hardest to distinguish from $\boldsymbol{\mu}$ among all such alternative instances.

%% file: sections/3_algorithm/2_optimal_weights.tex
\subsection{Solving the Optimal Weights}
\label{subsec: opt_weights}
Section \ref{subsec:counter_set} solved the internal optimization problem for fixed $\mathbf{w}$, yielding $\boldsymbol{\lambda}^*(\mathbf{w})\in\mathcal{A}(\boldsymbol{\mu})$. We now use this solution to find the optimal weights $\mathbf{w}^*$.

The general idea is to solve $\boldsymbol{\lambda}^*$ from the internal optimization problem (\ref{eq:internal_opt}) at each optimization loop, acquire the subgradients w.r.t. $\mathbf{w}$, and perform projected subgradient ascent to solve the external optimization problem 
\begin{equation}
    \label{eq:external_opt}
    \mathbf{w}^*({\boldsymbol{\mu}})=\arg\max_{\mathbf{w}\in\Sigma_{K\times L}}f_{{\boldsymbol{\mu}}}^*(\mathbf{w}), 
\end{equation}
where the notation $\mathbf{w}^*(\boldsymbol{\mu})$ illustrates that the optimal weight allocation $\mathbf{w}^*$ relies on the instance value $\boldsymbol{\mu}$. In the following, we sometimes use only $\mathbf{w}^*$ for simplicity.

We first consider the calculation of $\mathbf{w}^*(\boldsymbol{\mu})$ in a special case: if $\mathcal{C}({\boldsymbol{\mu}})=\emptyset$, we have a shortcut to calculate $\mathbf{w}^*$, provided in the following lemma.
\begin{lemma}
\label{lemma:weight_infeas}
    If $\mathcal{C}({\boldsymbol{\mu}})=\emptyset$, we have for each $k\in[K]$,
    \begin{equation}
\label{eq:weights_emptyset_overall}
\begin{aligned}
&\left(\mathbf{w}^*\right)_{k,l(k)}=\frac{1}{\mathrm{d}_k}\left(\sum_{k\in[K]}\frac{1}{\mathrm{d}_k}\right)^{-1};\\
&(\mathbf{w}^*)_{k,l}=0,\quad \text{for}~l\in[L],~l\neq l(k), 
\end{aligned}
\end{equation} 
where $\mathrm{d}_k$ is short for $\mathrm{d}(\mu_{k,l(k)},C_\mathrm{min})$, $l(k)=\arg\max_{l\in[L],\mu_{k,l}<C_\mathrm{min}}d(\mu_{k,l},C_\mathrm{min})$.
\end{lemma}
The proof of Lemma \ref{lemma:weight_infeas} is given in Appendix \ref{appendix:lemma_1_proof}. The key idea is that at optimality, $f_{\boldsymbol{\mu}}^k(\mathbf{w}^*)$ should be equal across all policies $k\in [K]$. For fixed $k$, the sum
$$\sum_{l\in[L],\mu_{k,l}<C_\mathrm{min}}(\mathbf{w}^*)_{k,l}\mathrm{d}(\mu_{k,l},C_\mathrm{min})$$
is linear in $\{(\mathbf{w}^*)_{k,l}\}_{l\in[L]}$. To maximize a linear function, we place all weight for policy $k$ on the subpopulation $l(k)$ with the largest distance to $C_{\mathrm{min}}$. This gives equation (\ref{eq:weights_emptyset_overall}) under the constraint $\sum_{k\in[K],l\in[L]}w_{k,l}=1$.

The optimization problem becomes more complicated when $\mathcal{C}({\boldsymbol{\mu}})\neq\emptyset$. In this case, we use the projected gradient method (\cite{boyd2003subgradient}) to solve the external optimization problem (\ref{eq:external_opt}). We present the following lemma of calculating a subgradient of $f^*_{\boldsymbol{\mu}}(\mathbf{w})$.
\begin{lemma}
    When $\mathcal{A}({\boldsymbol{\mu}})\neq\emptyset$, we can construct a subgradient  $\mathbf{c}(\mathbf{w})\in\partial f^*_{\boldsymbol{\mu}}(\mathbf{w})$ as follows:
If $f^\mathrm{opt}_{\boldsymbol{\mu}}(\mathbf{w})<f^\mathrm{feas}_{\boldsymbol{\mu}}(\mathbf{w})$, we can construct a subgradient $\mathbf{c}(\mathbf{w})=\{c_{k,l}\}_{k\in[K],l\in[L]}$ of  $f^*_{\boldsymbol{\mu}}(\mathbf{w})$ (i.e., $\mathbf{c}(\mathbf{w})\in \partial f^*_{\boldsymbol{\mu}}(\mathbf{w})$) at $\mathbf{w}$ by setting

\begin{equation*}
    \begin{aligned}
    c_{1,l}&= \mathrm{d}(\mu_{1,l},\lambda^*_{1,l}),\quad \forall~l\in[L];\\
    c_{k,l}&=\mathrm{d}(\mu_{k,l},\lambda^*_{k,l}),\quad \forall~l\in[L];\\
    c_{i,l}&=0,\quad \text{for}~i\in[K],~i\not=1,k;
\end{aligned}
    \end{equation*}
where $k=\arg\min_i\{f^{\mathrm{opt}(i)}_{\boldsymbol{\mu}}(\mathbf{w})\}_{i\in[K],i\not=1}$, and $(\{\lambda^*_{1,l}\}_{l\in[L]},\{\lambda^*_{k,l}\}_{l\in[L]})$ is the optimal solution of optimization problem (\ref{eq:f_mu_opt}), corresponding to $f^{\mathrm{opt}(k)}_{\boldsymbol{\mu}}(\mathbf{w})$ with index $k$.

Otherwise if $f^\mathrm{opt}_{\boldsymbol{\mu}}(\mathbf{w})\geq f^\mathrm{feas}_{\boldsymbol{\mu}}(\mathbf{w})$, we can construct a subgradient $\mathbf{c}(\mathbf{w})=\{c_{k,l}\}_{k\in[K],l\in[L]}$ of  $f^*_{\boldsymbol{\mu}}(\mathbf{w})$ at $\mathbf{w}$ by setting
\begin{equation*}
    \begin{aligned}
        c_{1,l^*}&=\mathrm{d}(\mu_{1,l^*},C_{\mathrm{min}}),
    \end{aligned}
\end{equation*}
and setting the rest of $\{c_{k,l}\}_{k\in[K],l\in[L],(k,l)\not=(1,l^*)}$ to $0$. Here $l^*$ is the optimal index of the optimization problem  (\ref{eq:f_infeas}), corresponding to $f_{\boldsymbol{\mu}}^{\mathrm{feas}}(\mathbf{w})$.
\label{lemma:subgradient}
\end{lemma}

In Appendix \ref{appendix:lemma_2_proof}, we prove that $\mathbf{c}(\mathbf{w})$ is a valid subgradient of $f^*_{\boldsymbol{\mu}}(\mathbf{w})$. The basic idea of proof comes from the calculation of  $f^*_{{\boldsymbol{\mu}}}(\mathbf{w})$ given in Lemma \ref{lemma: counter_set_calculation}. We know that either there exist $\boldsymbol{\lambda}\in\mathcal{S},~k\in[K]$ such that $f^*_{{\boldsymbol{\mu}}}(\mathbf{w})=\sum_{l\in[L]}w_{1,l}\mathrm{d}(\mu_{1,l},\lambda_{1,l})+\sum_{l\in[L]}w_{k,l}\mathrm{d}(\mu_{k,l},\lambda_{k,l})$, or there exists $l\in[L]$ such that $f^*_{{\boldsymbol{\mu}}}(\mathbf{w})=w_{1,l}\mathrm{d}(\mu_{1,l},C_\mathrm{min})$. In both cases we can write $f^*_{{\boldsymbol{\mu}}}(\mathbf{w})=\mathbf{c}(\mathbf{w})^T{\mathbf{w}}$ for some $\mathbf{c}(\mathbf{w})\in\mathbb{R}^{K \times L}$, so we can obtain a subgradient of $f^*_{{\boldsymbol{\mu}}}(\mathbf{w})$ given in Lemma \ref{lemma:subgradient}. 

Now, by performing projected subgradient method, we can solve the maximization problem (\ref{eq:external_opt}) by updating $\mathbf{w}^{(n+1)}=\mathbf{P}_{\Sigma_{K\times L}}\left(\mathbf{w}^{(n)}+\alpha_{n}\mathbf{c}(\mathbf{w}^{(n)})\right)$ iteratively with the projection operator $\mathbf{P}_{\Sigma_{K\times L}}(\cdot)$ and some proper stepsizes $\{\alpha_n\}$. The following Lemma shows that the projected subgradient algorithm applied to the external optimization problem (\ref{eq:external_opt}) yields an optimal solution  $\mathbf{w}^*$.
\begin{lemma}
\label{lemma:optimization_convergence}
    For the sequence $\{\mathbf{w}^{(n)}\}$ starting from an arbitrary point $\mathbf{w}^{(0)}$ and updated by
    \begin{equation*}
        \mathbf{w}^{(n+1)}=\mathbf{P}_{\Sigma_{K\times L}}\left(\mathbf{w}^{(n)}+\alpha_{n}\mathbf{c}(\mathbf{w}^{(n)})\right),
    \end{equation*}
    where $\mathbf{c}(\mathbf{w})$ is calculated as in Lemma \ref{lemma:subgradient}, if the sequence of step sizes $\{\alpha_n\}$ satisfies $\lim_{n\to\infty}\alpha_n=0,~\sum_{n=1}^\infty\alpha_k=\infty$,
    we have
    \begin{equation*}
        \lim_{n\to\infty}f_{\boldsymbol{\mu}}^*(\mathbf{w}^{(n)})=T^{*}(\boldsymbol{\mu})^{-1}.
    \end{equation*}
\end{lemma}
The proof of Lemma \ref{lemma:optimization_convergence} is given in Appendix \ref{appendix:lemma_3_proof}.

%% file: sections/3_algorithm/3_sample_allocation.tex
\subsection{Calculating the Sample Allocation Rule}
\label{subsec:sample_alloc}
In this section, we further apply the solution of the external optimization problem (\ref{eq:external_opt}) to derive the sample allocation rule of the T-a-S-CS algorithm in each round $t=1,2,\cdots$.

First, the algorithm computes the estimated expected performance of all policies and all subpopulations, denoted by $\hat{\boldsymbol{\mu}}(t)\in\mathbb{R}^{K\times L}$. Specifically, we have
\begin{equation}
\label{eq:est_performance}
\hat{\mu}_{k,l}(t)=\frac{1}{N_{k,l}(t)}\sum_{s=1}^t X_s\cdot \boldsymbol{1}\{(A_s,I_s)=(k,l)\},
\end{equation}
where $N_{k,l}(t)$ is the number of samples drawn from policy and subpopulation $(k,l)$ up to time step $t$, namely, $N_{k,l}(t)=\sum_{s=1}^t\boldsymbol{1}\{(A_s,I_s)=(k,l)\}$. 

Next, the maximizer $\mathbf{w}^*(\hat{\boldsymbol{\mu}}(t))$ of problem \eqref{eq:T_star} is then computed with $\boldsymbol{\mu}$ replaced by $\hat{\boldsymbol{\mu}}(t)$. We additionally note that the solution of (\ref{eq:external_opt}) is not necessarily unique. Strictly speaking, we acquire one solution  $\mathbf{w}_t\in \mathbf{w}^*(\hat{\boldsymbol{\mu}}(t))$.

Given the estimated optimal weights $\mathbf{w}_t$ at time $t$, the T-a-S-CS algorithm uses the C-tracking rule from \cite{garivier2016optimal} to select the next policy-subpopulation pair $(A_{t+1},I_{t+1})$. Specifically, let $\mathbf{w}^{(\varepsilon)}_t$ be the $L^{\infty}$ projection of $\mathbf{w}_t$ onto $\Sigma_{K\times L}^{\varepsilon}=\{(w_1,\ldots,w_{KL})\in[\varepsilon,1]^{K\times L}\mid w_1+\cdots+w_{KL}=1\}$, with $\varepsilon_t=\left(K^2L^2+t\right)^{-1/2}/2$. The algorithm selects
\begin{equation}\label{eq:c_tracking}
\left(A_{t+1},I_{t+1}\right)\in\arg\max_{(k,l)}\left(\sum_{s=0}^tw_{k,l}^{\varepsilon_s}(\hat{\boldsymbol{\mu}}(s))-N_{k,l}(t)\right).    
\end{equation}
This rule ensures $N_{k,l}(t)$ tracks $\sum_{s=0}^tw_{k,l}^{\varepsilon_s}(\hat{\boldsymbol{\mu}}(s))$, which converges to $tw^*_{k,l}({\boldsymbol{\mu}})$. As will be shown in Section \ref{subsec:sample_complexity}, this yields asymptotically optimal sample allocation that achieves the lower bound \eqref{eq:lower_bound}. 

%% file: sections/3_algorithm/4_stopping_rule.tex
\subsection{Stopping Rule and Recommendation}
\label{subsec:stop_rule}

In this section, we discuss the stopping rule of the T-a-S-CS algorithm and then provide the whole framework of the T-a-S-CS algorithm.
Stopping conditions of the ranking and selection algorithms have been discussed in literature such as \cite{hong2021review}. In our work, following the idea of \cite{russac2021b},  we use the Chernoff's Generalized Likelihood Ratio statistic:
\begin{equation}
\label{eq:stopping_statistics}
Z(t)=\inf_{{\boldsymbol{\lambda}}\in \mathcal{A}({\boldsymbol{\hat\mu(t)}})}\sum_{k\in[K]}\sum_{l\in[L]}N_{k,l}(t)d\left(\hat\mu_{k,l}(t),\lambda_{k,l}\right).
\end{equation}
We define the stopping time $\tau_\delta$ as 
\begin{equation*}
\tau_\delta = \inf_{t\in \mathbb{N}} \{Z(t)>\beta(t,\delta)\},
\end{equation*}
for a given risk level $0<\delta<1$.
Note that the calculation of $Z(t)$ has a similar form as the internal optimization problem (\ref{eq:internal_opt}). Denote the empirical weights as $(\hat{\mathbf{w}}_t)_{k,l}:=N_{k,l}(t)/t$, and $Z(t)$ can be written as $
Z(t)=tf^*_{\hat{\boldsymbol{\mu}}(t)}(\hat{\mathbf{w}}_t)$, therefore, $Z(t)$ can be calculated the same way as in Section \ref{subsec:counter_set}. 

By Proposition 15 of \cite{kaufmann2021mixture}, the threshold $\beta(t,\delta)$ can be set as $$\beta(t,\delta)=O\left(L\ln\ln t+\ln\frac{K}{\delta}\right)$$ to ensure the algorithm to be $\delta$-PAC. In practice, as suggested by \cite{russac2021b}, we use instead the stylized
$\ln({(1+\ln t)}/{\delta})$, which is less conservative. The final recommendation $\hat{k}^*_{\tau_\delta}$ is the best feasible policy of $\hat{\boldsymbol{\mu}}(\tau_\delta)$, i.e. 
$$    \hat{k}^*_{\tau_\delta}=\arg\max_{k\in{\mathcal{C}(\hat{\boldsymbol{\mu}}(\tau_\delta))}}\hat{\mu}_{k}.$$
We take $\hat{k}^*_{\tau_\delta}=0$ if $\mathcal{C}(\hat{\boldsymbol{\mu}}(\tau_{\delta}))=\emptyset$. The overall framework of the T-a-S-CS algorithm is provided in Algorithm \ref{alg:TaSCS}.

\begin{algorithm}[!t] 
\caption{T-a-S-CS algorithm framework}
\label{alg:TaSCS}
\begin{algorithmic}[1]
\REQUIRE Problem: $K,L,C_\mathrm{min},(q_1,\ldots,q_L),\delta$; Parameters: $T_\mathrm{max},N_\mathrm{max},\{\alpha_n\},N_0,\hat{\mathbf{w}}_0$
\FOR{$k\in[K], l\in[L]$} 
\STATE $N_{k,l}(0)\leftarrow N_0$, $S_{k,l}(0)\leftarrow(\hat{\mathbf{w}}_0)_{k,l}-N_0$, draw $N_0$ samples for $\hat{\mu}_{k,l}(0)$
\ENDFOR
\WHILE{$t\leq T_\mathrm{max}$}
\IF{$\mathcal{C}(\hat{\boldsymbol{\mu}}(t))=\emptyset$}
\STATE Calculate $\hat{\mathbf{w}}_t$ using (\ref{eq:weights_emptyset_overall})
\ELSE
\STATE $\mathbf{w}^{(0)}=\hat{\mathbf{w}}_{t-1}$
\FOR{$n=1,\ldots,N_\mathrm{max}$}
\STATE Compute $\mathbf{c}(\mathbf{w}^{(n)})\in\partial f^*_{\hat{\boldsymbol{\mu}}(t)}(\mathbf{w}^{(n)})$ via Algorithm \ref{alg:internal}
\STATE $\mathbf{w}^{(n+1)}\leftarrow\mathbf{P}_{\Sigma_{K\times L}}(\mathbf{w}^{(n)}+\alpha_n\mathbf{c}(\mathbf{w}^{(n)}))$ 
\ENDFOR
\STATE $\hat{\mathbf{w}}_t\leftarrow\mathbf{w}^{(N_\mathrm{max})}$
\ENDIF
\STATE Compute $\varepsilon_t\leftarrow (K^2L^2+t)^{-1/2}/2$, project $\hat{\mathbf{w}}_t^{(\varepsilon)}\leftarrow \mathbf{P}_{\Sigma^{\varepsilon_t}_{K\times L}}(\hat{\mathbf{w}}_t)$
\STATE Update $\mathbf{S}(t)\leftarrow\mathbf{S}(t-1)+\hat{\mathbf{w}}_t^{(\varepsilon)}$
\STATE \textbf{Sample} $(A_{t+1},I_{t+1})=\arg\max_{k,l}S_{k,l}(t)$, obtain $X_t$
\STATE Update $N_{A_t,I_t}(t+1),S_{A_t,I_t}(t+1),\hat{\mu}_{A_t,I_t}(t+1)$
\STATE Compute $Z(t+1)$ via Algorithm \ref{alg:internal} with $\hat{\mathbf{w}}_t=\mathbf{N}(t+1)/(t+1)$
\IF{$Z(t+1)>\ln((1+\ln(t+1))/\delta)$}
\RETURN $k^*(\hat{\boldsymbol{\mu}}(t+1))$
\ENDIF
\ENDWHILE
\end{algorithmic}
\end{algorithm}

\begin{algorithm}[!t]
\caption{Internal optimization algorithm}
\label{alg:internal}
\begin{algorithmic}[1]
\REQUIRE $K,L,\mathrm{d}(\cdot),\mathbf{w},\boldsymbol{\mu}$
\STATE Initialize $\mathbf{c}=\mathbf{0}$
\STATE Compute $f_1\leftarrow f^\mathrm{feas}_{\boldsymbol{\mu}}(\mathbf{w})$ and $l^*$ via (\ref{eq:f_infeas})
\FOR{$k=2,\ldots, K$}
\STATE Compute $f_k\leftarrow f^{\mathrm{opt}(k)}_{\boldsymbol{\mu}}(\mathbf{w})$, $\mathbf{\lambda}_1(k)$, $\mathbf{\lambda}(k)$ via (\ref{eq:f_mu_opt})
\ENDFOR
\STATE $k^*\leftarrow \arg\min_k f_k$
\IF{$k^*=1$}
\STATE $\mathbf{c}_{1,l^*}=\mathrm{d}(\mu_{1,l^*},C_\mathrm{min})$
\ELSE
\STATE $\mathbf{c}_1\leftarrow\mathrm{d}(\mathbf{\mu}_1,\mathbf{\lambda}_1(k^*))$, $\mathbf{c}_{k^*}\leftarrow\mathrm{d}(\mathbf{\mu}_{k^*},\mathbf{\lambda}(k^*))$
\ENDIF
\RETURN $\mathbf{c}$, $f_{k^*}$
\end{algorithmic}
\end{algorithm}

We briefly discuss the computational tractability of the nested optimization required by T-a-S-CS. The internal optimization (\ref{eq:internal_opt}) is a constrained minimization of a weighted sum of KL-divergences, whose structure depends on the exponential family $\mathcal{P}$. For the Gaussian family (known variance), the KL-divergence $\mathrm{d}(\mu,\lambda)=(\mu-\lambda)^2/(2\sigma^2)$ is quadratic, so each internal subproblem reduces to a convex quadratic program (QP) with linear constraints. For other exponential families (e.g., Bernoulli, Poisson, gamma), the KL-divergence is convex but nonlinear in $\lambda$, so the internal problems remain convex programs solvable by general-purpose methods (e.g., interior-point algorithms). The external optimization (\ref{eq:external_opt}) is a concave maximization over the simplex $\Sigma_{K\times L}$, solved by projected subgradient ascent regardless of the choice of $\mathcal{P}$. Detailed formulations for the Gaussian case and other representative cases are provided in Appendix~\ref{appendix:gaussian_sbfc} and Appendix~\ref{appendix:bernoulli_subproblem}.

%% file: sections/3_algorithm/5_sample_complexity.tex
\subsection{Sample Complexity of the T-a-S-CS Algorithm}
\label{subsec:sample_complexity}
We now give the convergence result of the T-a-S-CS algorithm, which matches the asymptotic optimal lower bound given by \eqref{eq:lower_bound}:

\begin{theorem}\label{thm:opt_convergnce}
For every instance ${\boldsymbol{\mu}}\in\mathcal{S}$, the T-a-S-CS algorithm is $\delta$-PAC and achieves the asymptotic instance-specific lower bound, i.e,
\begin{equation}
    \lim_{\delta\to 0}\frac{\mathbb{E}_{\boldsymbol{\mu}}[\tau_\delta]}{\ln(1/\delta)}=T^*({\boldsymbol{\mu}}).
\end{equation}
\end{theorem}

The proof of Theorem \ref{thm:opt_convergnce} is given in Appendix \ref{appendix:thm_conv_proof}. The main idea of the proof is briefly stated as follows: The convergence result presented in the theorem is based on the convergence of $\hat{\boldsymbol{\mu}}(t)\to{\boldsymbol{\mu}}$, which further guarantees the convergence of $\inf_{\mathbf{w}\in \mathbf{w}^*({\boldsymbol{\mu}})}\big\Vert\hat{\mathbf{w}}_t-\mathbf{w}\big\Vert_\infty\to 0$ because of the continuity of ${\boldsymbol{\mu}}\mapsto \mathbf{w}^*({\boldsymbol{\mu}})$. Further, the sampling strategy ensures that $N_{k,l}(t)\approx \sum_{s=0}^{t-1}w_{k,l}^*(\hat{\boldsymbol{\mu}}(s))$. We can then use the proof of Theorem 7 of \cite{degenne2019pure}, which uses the above convergence results of $\hat{\boldsymbol{\mu}}(t)$, $\hat{\mathbf{w}}_t$ and $N_{k,l}(t)$ to derive an upper bound of the stopping time $\tau_\delta$.

%% file: sections/4_extensions/1_generalized_constraints.tex
\subsection{Generalized Fairness Constraints: \texorpdfstring{$\mathcal{M}$-SBFC}{M-SBFC}}
\label{subsec:first_extension}

We generalize the linear fairness constraint $\mu_{k,l}\geq C_{\text{min}},~\forall l\in[L]$ to a set-based constraint $\mathbf{\mu}_k\in\mathcal{M}$, where $\mathcal{M} \subseteq \mathbb{R}^L$ is a closed set. The feasible set becomes:
\begin{equation*}
    \mathcal{C}(\boldsymbol{\mu},\mathcal{M})=\left\{k\in[K]\left|~\mathbf{\mu}_k\in\mathcal{M}\right.\right\}.
\end{equation*}
This enables flexible fairness specifications (e.g., quadratic constraints, ratio-based requirements, or ``at least $p\%$ of subpopulations exceed threshold'') while maintaining the theoretical framework of Section 3. Analogous to Definition \ref{def:S}, we impose the structural requirement that the optimal policy's mean performance $\mathbf{\mu}_{k^*}$ lies strictly in the interior of $\mathcal{M}$ (with positive distance from the boundary) and dominates other feasible policies by a positive gap. These conditions generalize the strict inequalities in Definition \ref{def:S} and are necessary for the convergence analysis in Appendix \ref{appendix:thm_conv_ext1_proof}.

\subsubsection{Counter Set Calculation}

The key modification is in the internal optimization problem, where the counter set becomes $\mathcal{A}(\boldsymbol{\mu}, \mathcal{M})$ and problem (\ref{eq:internal_opt}) becomes:
\begin{equation}
    \label{eq:internal_opt_ext1}
    f_{\boldsymbol{\mu},\mathcal{M}}^*(\mathbf{w})=\inf_{{\boldsymbol{\lambda}}\in \mathcal{A}(\boldsymbol{\mu},\mathcal{M})}\sum_{k\in[K]}\sum_{l\in[L]}w_{k,l}\mathrm{d}\left(\mu_{k,l},\lambda_{k,l}\right).
\end{equation}

\begin{lemma}
\label{lemma: counter_set_calculation_ext1}
The calculation of $f^*_{\boldsymbol{\mu},\mathcal{M}}(\mathbf{w})$ in (\ref{eq:internal_opt_ext1}) parallels Lemma \ref{lemma: counter_set_calculation} with $\mathcal{M}$ replacing the linear constraints:

If $\mathcal{C}(\boldsymbol{\mu})=\emptyset$, we have $f_{\boldsymbol{\mu},\mathcal{M}}^*(\mathbf{w})=\min\left\{f_{\boldsymbol{\mu},\mathcal{M}}^{1}(\mathbf{w}),\ldots,f_{\boldsymbol{\mu},\mathcal{M}}^{K}(\mathbf{w})\right\}$, where
\begin{equation}
\label{eq:f_star_k_0_ext1}
\begin{aligned}
f_{\boldsymbol{\mu},\mathcal{M}}^k(\mathbf{w})
&=\inf_{\lambda_{k,1},\ldots,\lambda_{k,L}\in\mathbb{R}}~\sum_{l\in[L]}w_{k,l}\mathrm{d}(\mu_{k,l},\lambda_{k,l}),\\
&\quad\text{s.t.}~{\mathbf{\lambda}}_{k}\in\mathcal{M}.
\end{aligned}
\end{equation}

If $\mathcal{C}(\boldsymbol{\mu})\not=\emptyset$ with $k^*(\boldsymbol{\mu})=1$, we have $f^*_{\boldsymbol{\mu},\mathcal{M}}(\mathbf{w})=\min\left\{f^\mathrm{opt}_{\boldsymbol{\mu},\mathcal{M}}(\mathbf{w}),f^\mathrm{feas}_{\boldsymbol{\mu},\mathcal{M}}(\mathbf{w})\right\}$, where
$f^\mathrm{opt}_{\boldsymbol{\mu},\mathcal{M}}(\mathbf{w})=\min\left\{f^{\mathrm{opt}(2)}_{\boldsymbol{\mu},\mathcal{M}}(\mathbf{w}),\ldots,f^{\mathrm{opt}(K)}_{\boldsymbol{\mu},\mathcal{M}}(\mathbf{w})\right\}$ with
\begin{equation}
\label{eq:f_mu_opt_ext1}
\begin{aligned}
&f^{\mathrm{opt}(k)}_{\boldsymbol{\mu},\mathcal{M}}(\mathbf{w}) \\
&=\inf_{\substack{\lambda_{i,l}\in\mathbb{R}\\ i\in\{1,k\}, l\in[L]}}
\left(\sum_{i\in\{1,k\}}\sum_{l\in[L]}w_{i,l}\mathrm{d}(\mu_{i,l},\lambda_{i,l})\right),\\
&\quad\text{s.t.}~\sum_{l\in[L]}q_l\lambda_{k,l}>\sum_{l\in[L]}q_l\lambda_{1,l},~
\boldsymbol{\lambda}_k\in\mathcal{M},
\end{aligned}
\end{equation}
and
\begin{equation}
\label{eq:f_infeas_ext1}
\begin{aligned}
f^\mathrm{feas}_{\boldsymbol{\mu},\mathcal{M}}(\mathbf{w})
&=\inf_{\lambda_{1,1},\ldots,\lambda_{1,L}\in\mathbb{R}}~\sum_{l\in[L]}w_{1,l}\mathrm{d}(\mu_{1,l},\lambda_{1,l}),\\
&\quad\text{s.t.}~{\mathbf{\lambda}}_{1}\not\in\mathcal{M}.
\end{aligned}
\end{equation}
\end{lemma}

The proof of Lemma \ref{lemma: counter_set_calculation_ext1} is given in Appendix \ref{appendix:lemma_cs_ext1_proof}. The proof follows the same structure as Lemma \ref{lemma: counter_set_calculation}, replacing constraint $\lambda_{k,l} \geq C_{\min}$ with $\boldsymbol{\lambda}_k \in \mathcal{M}$. The three optimization problems (\ref{eq:f_star_k_0_ext1}), (\ref{eq:f_mu_opt_ext1}), and (\ref{eq:f_infeas_ext1}) correspond to the same three scenarios as in Section \ref{subsec:counter_set}: making an infeasible policy feasible, making a different policy optimal, and making the current optimal policy infeasible.

We briefly discuss the computational tractability of the $\mathcal{M}$-SBFC internal optimization. The structure of the subproblems (\ref{eq:f_mu_opt_ext1})--(\ref{eq:f_infeas_ext1}) depends on the geometry of $\mathcal{M}$: for constraint sets defined by polynomial inequalities (e.g., quadratic variance bounds), the subproblems are quadratically constrained quadratic programs (QCQPs) or admit semidefinite programming (SDP) relaxations, both solvable in polynomial time. More generally, the framework maintains tractability whenever the constraint set $\mathcal{M}$ preserves convexity of the optimality subproblem (\ref{eq:f_mu_opt_ext1}); the feasibility subproblem (\ref{eq:f_infeas_ext1}) may be non-convex due to the complementary constraint $\boldsymbol{\lambda}_1 \notin \mathcal{M}$, but often admits efficient solutions via relaxation techniques. A concrete example with variance-based fairness constraints and the detailed formulations are provided in Appendix~\ref{appendix:variance_constraint}.

\subsubsection{Sample Complexity and Algorithm}

The generalization from $\mu_{k,l} \ge C_{\min}$ to $\mathbf{\mu}_k \in \mathcal{M}$ does not alter the theoretical results—only the internal optimization changes. The following corollaries parallel Theorem \ref{thm:lower_bound} and Theorem \ref{thm:opt_convergnce}.

\begin{corollary}[Lower Bound for $\mathcal{M}$-SBFC]
    \label{cor:lower_bound_m_sbfc}
For any $\delta$-PAC algorithm applied to the $\mathcal{M}$-SBFC problem and any instance $\boldsymbol{\mu}\in\mathcal{S}$ (with $\mathcal{C}(\boldsymbol{\mu}) = \{ k \mid \mathbf{\mu}_k \in \mathcal{M} \}$),
\begin{equation}\label{eq:lower_bound_m_sbfc}
    \liminf_{\delta\to 0}\frac{\mathbb{E}_{\boldsymbol{\mu}}[\tau_\delta]}{\ln(1/\delta)}\geq T^*(\boldsymbol{\mu},\mathcal{M}),
\end{equation}
where $T^*(\boldsymbol{\mu},\mathcal{M})^{-1}=\sup_{\mathbf{w}\in \Sigma_{K\times L}}\inf_{{\boldsymbol{\lambda}}\in \mathcal{A}(\boldsymbol{\mu},\mathcal{M})}\sum_{k,l}w_{k,l}\mathrm{d}(\mu_{k,l},\lambda_{k,l})$ with counter set $\mathcal{A}(\boldsymbol{\mu},\mathcal{M}) = \{ \boldsymbol{\lambda} \in \mathcal{S} \mid k^*(\boldsymbol{\lambda}, \mathcal{M}) \neq k^*(\boldsymbol{\mu}, \mathcal{M}) \}$.
\end{corollary}

The proof of Corollary \ref{cor:lower_bound_m_sbfc} is given in Appendix \ref{appendix:thm_lb_ext1_proof}. The proof follows that of Theorem \ref{thm:lower_bound} with $\mathcal{A}(\boldsymbol{\mu})$ replaced by $\mathcal{A}(\boldsymbol{\mu},\mathcal{M})$.

The T-a-S-CS algorithm applies directly with minimal modification: compute $\mathbf{w}_t \in \mathbf{w}^*(\hat{\boldsymbol{\mu}}(t), \mathcal{M})$ using projected subgradient ascent on $f^*_{\hat{\boldsymbol{\mu}}(t),\mathcal{M}}(\mathbf{w})$, use the C-tracking rule (\ref{eq:c_tracking}), and apply the stopping rule with $Z(t) = t f^*_{\hat{\boldsymbol{\mu}}(t),\mathcal{M}}(\hat{\mathbf{w}}_t)$. The only change is in computing subgradients of $f^*$ via (\ref{eq:f_star_k_0_ext1})--(\ref{eq:f_infeas_ext1}) instead of the corresponding problems in Section \ref{subsec: opt_weights}.

\begin{corollary}[Asymptotic Optimality of $\mathcal{M}$-T-a-S-CS]
    \label{cor:opt_convergnce_m_sbfc}
For every instance ${\boldsymbol{\mu}}\in\mathcal{S}$ (with $\mathcal{C}(\boldsymbol{\mu}) = \{ k \mid \mathbf{\mu}_k \in \mathcal{M} \}$), the $\mathcal{M}$-T-a-S-CS algorithm is $\delta$-PAC and achieves
\begin{equation}
    \lim_{\delta\to 0}\frac{\mathbb{E}_{\boldsymbol{\mu}}[\tau_\delta]}{\ln(1/\delta)}=T^*({\boldsymbol{\mu}}, \mathcal{M}).
\end{equation}
\end{corollary}

The proof of Corollary \ref{cor:opt_convergnce_m_sbfc} is given in Appendix \ref{appendix:thm_conv_ext1_proof}. The proof mirrors that of Theorem \ref{thm:opt_convergnce} with modified internal optimization.

%% file: sections/4_extensions/2_penalty_fairness.tex
\subsection{Soft Fairness Constraints: $\gamma$-SBFC}
\label{subsec:second_extension}

We consider a second extension that relaxes the hard fairness constraints $\mu_{k,l}\geq C_{\text{min}}$ to soft constraints by incorporating penalties for violations. This enables trade-offs between performance and fairness through penalty parameters $\boldsymbol{\gamma}=(\gamma_1,\ldots,\gamma_L)\in\mathbb{R}_{\geq 0}^L$.

\subsubsection{Penalized Performance}

For a policy $k$ with subpopulation performances $\boldsymbol{\mu}_k=(\mu_{k,1},\ldots,\mu_{k,L})$, we define the \textit{penalized performance}:
\begin{equation}\label{eq:linear_relaxed}
\mu_k^{\gamma}=\sum_{l=1}^Lq_l\left(\mu_{k,l}+\gamma_{l}\min\left(\mu_{k,l}-C_\text{min},0\right)\right).   
\end{equation}
This can be viewed as a Lagrangian relaxation where violations of $\mu_{k,l}\geq C_{\text{min}}$ reduce the effective performance. The $\gamma$-SBFC problem seeks:
\begin{equation}\label{eq:relaxed_goal}
    k^{\gamma,*}=\arg\max_{k\in[K]}\mu_k^\gamma.
\end{equation}
Note that as $\gamma_l\to+\infty$ for all $l\in[L]$, we recover the original SBFC problem where $k^{\gamma,*}\to k^*$. Analogous to Definition \ref{def:S}, we impose the structural requirement that $k^{\gamma,*}(\boldsymbol{\mu})$ is unique and dominates other policies in penalized performance by a positive gap (i.e., $\mu_{k^{\gamma,*}}^\gamma > \mu_k^\gamma + \varepsilon$ for all $k \neq k^{\gamma,*}$ and some $\varepsilon>0$). This condition generalizes the strict inequalities in Definition \ref{def:S} to the penalized setting and is necessary for the convergence analysis in Appendix \ref{appendix:thm_conv_ext2_proof}.

\subsubsection{Counter Set Calculation}

The counter set becomes $\mathcal{A}(\boldsymbol{\mu}, \boldsymbol{\gamma})$ and the internal optimization problem (\ref{eq:internal_opt}) becomes:
\begin{equation}
    \label{eq:internal_opt_ext2}
    f_{\boldsymbol{\mu},\boldsymbol{\gamma}}^{*}(\mathbf{w})=\inf_{{\boldsymbol{\lambda}}\in\mathcal{A}(\boldsymbol{\mu}, \boldsymbol{\gamma})}\sum_{k\in[K]}\sum_{l\in[L]}w_{k,l}\mathrm{d}\left(\mu_{k,l},\lambda_{k,l}\right).
\end{equation}

\begin{lemma}
\label{lemma: counter_set_calculation_ext2}
The calculation of $f^*_{\boldsymbol{\mu},\boldsymbol{\gamma}}(\mathbf{w})$ in (\ref{eq:internal_opt_ext2}) parallels Lemma \ref{lemma: counter_set_calculation}. Assuming $k^{\gamma,*}(\boldsymbol{\mu})=1$, we have $f_{\boldsymbol{\mu},\boldsymbol{\gamma}}^{*}(\mathbf{w})=\min\left\{f^{\mathrm{opt}(2)}_{\boldsymbol{\mu},\boldsymbol{\gamma}}(\mathbf{w}),\ldots,f^{\mathrm{opt}(K)}_{\boldsymbol{\mu},\boldsymbol{\gamma}}(\mathbf{w})\right\}$, where $\forall~k\in\{2,\ldots,K\}$,
\begin{equation}
\label{eq:f_mu_opt_ext2}
\begin{aligned}
f^{\mathrm{opt}(k)}_{\boldsymbol{\mu},\boldsymbol{\gamma}}(\mathbf{w})
&=\inf_{\substack{\lambda_{i,l}\in\mathbb{R}\\ i\in\{1,k\}, l\in[L]}}
\sum_{i\in\{1,k\}}\sum_{l\in[L]}w_{i,l}\mathrm{d}(\mu_{i,l},\lambda_{i,l}),\\
&\quad\text{s.t.}~\lambda_k^\gamma>\lambda_1^\gamma,
\end{aligned}
\end{equation}
where $\lambda_k^\gamma=\sum_{l=1}^Lq_l\left(\lambda_{k,l}+\gamma_{l}\min(\lambda_{k,l}-C_\text{min},0)\right)$ and $\lambda_1^\gamma= \sum_{l=1}^Lq_l\left(\lambda_{1,l}+\gamma_{l}\min(\lambda_{1,l}-C_\text{min},0)\right)$.
\end{lemma}

The proof of Lemma \ref{lemma: counter_set_calculation_ext2} is given in Appendix \ref{appendix:thm_cs_ext2_proof}. Note that the constraint $\lambda_k^\gamma\geq\lambda_1^\gamma$ in (\ref{eq:f_mu_opt_ext2}) involves the non-smooth $\min(\cdot,0)$ function, making the problem non-convex in general. We briefly discuss the computational tractability of the $\gamma$-SBFC internal optimization. The non-smooth penalty can be handled via mixed-integer reformulation: by introducing auxiliary variables to represent $\min(\lambda_{i,l}-C_{\min},0)$ and binary indicator variables with Big-M linearization, the internal subproblem becomes a mixed-integer quadratic program (MIQP) with $2L$ binary variables under the Gaussian family, efficiently solvable by standard solvers such as Gurobi. The detailed MIQP formulation is provided in Appendix~\ref{appendix:soft_constraint}.

\subsubsection{Sample Complexity and Algorithm}

The following corollaries parallel Theorem \ref{thm:lower_bound} and Theorem \ref{thm:opt_convergnce}.

\begin{corollary}[Lower Bound for $\gamma$-SBFC]
    \label{cor:lower_bound_ext2}
For any $\delta$-PAC algorithm applied to the $\gamma$-SBFC problem and any instance $\boldsymbol{\mu}\in\mathcal{S}$, 
\begin{equation}\label{eq:lower_bound_ext2}
    \liminf_{\delta\to 0}\frac{\mathbb{E}_{\boldsymbol{\mu}}[\tau_\delta]}{\ln(1/\delta)}\geq T^*(\boldsymbol{\mu},\boldsymbol{\gamma}),
\end{equation}
where $T^*(\boldsymbol{\mu},\boldsymbol{\gamma})^{-1}=\sup_{\mathbf{w}\in \Sigma_{K\times L}}\inf_{{\boldsymbol{\lambda}}\in \mathcal{A}(\boldsymbol{\mu},\boldsymbol{\gamma})}\sum_{k,l}w_{k,l}\mathrm{d}(\mu_{k,l},\lambda_{k,l})$ with counter set $\mathcal{A}(\boldsymbol{\mu},\boldsymbol{\gamma}) = \{ \boldsymbol{\lambda} \in \mathcal{S} \mid k^{\gamma,*}(\boldsymbol{\lambda}) \neq k^{\gamma,*}(\boldsymbol{\mu}) \}$.
\end{corollary}

The proof of Corollary \ref{cor:lower_bound_ext2} is given in Appendix \ref{appendix:thm_lb_ext2_proof}. The proof follows that of Theorem \ref{thm:lower_bound} with penalized performance replacing hard constraints.

The $\gamma$-T-a-S-CS algorithm applies with minimal modification: compute $\mathbf{w}_t \in \mathbf{w}^*(\hat{\boldsymbol{\mu}}(t), \boldsymbol{\gamma})$ using projected subgradient ascent on $f^*_{\hat{\boldsymbol{\mu}}(t),\boldsymbol{\gamma}}(\mathbf{w})$ via (\ref{eq:f_mu_opt_ext2}), use the C-tracking rule (\ref{eq:c_tracking}), and apply the stopping rule with $Z(t) = t f^*_{\hat{\boldsymbol{\mu}}(t),\boldsymbol{\gamma}}(\hat{\mathbf{w}}_t)$.

\begin{corollary}[Asymptotic Optimality of $\gamma$-T-a-S-CS]
\label{cor:opt_convergnce_ext2}
For every instance ${\boldsymbol{\mu}}\in\mathcal{S}$, the $\gamma$-T-a-S-CS algorithm is $\delta$-PAC and achieves
\begin{equation}
    \lim_{\delta\to 0}\frac{\mathbb{E}_{\boldsymbol{\mu}}[\tau_\delta]}{\ln(1/\delta)}=T^*({\boldsymbol{\mu}},\boldsymbol{\gamma}).
\end{equation}
\end{corollary}

The proof of Corollary \ref{cor:opt_convergnce_ext2} is given in Appendix \ref{appendix:thm_conv_ext2_proof}. The proof mirrors that of Theorem \ref{thm:opt_convergnce} with modified internal optimization.

%% file: sections/4_extensions/3_conclusion.tex
\subsection{Conclusion}
\label{subsec:ext_conclusion}
In this section, we summarize the two extensions from the viewpoint of the counter set. The use of the counter set in the T-a-S-CS algorithmic framework is indeed powerful in that it flexibly allows various generalizations of the SBFC problem. To incorporate a new setting, the idea is to analyze any specific instance $\boldsymbol{\mu}\in\mathcal{S}$ under the setting by explicitly answering ``which policy is selected as the best", and ``how to modify $\boldsymbol{\mu}$ into $\boldsymbol{\lambda}$, yielding a different best policy", i.e., to have $k^*(\boldsymbol{\lambda})\not=k^*(\boldsymbol{\mu})$. The problem setting (i.e., the form of constraint and the level of fairness) is transformed and incorporated into calculation of the sample complexity and the construction of the algorithm, completely through the counter set analysis.

%% file: sections/5_numerical/1_aspects.tex
\section{Computational and Numerical Study}
\label{sec:computational_numerical}

In the previous sections, we established the theoretical foundations of the SBFC problem and its extensions, including asymptotic optimality guarantees and sample complexity lower bounds. This section focuses on practical implementation and empirical validation. Section \ref{subsec:practical_implementation} discusses practical implementation strategies for the T-a-S-CS algorithm, including finite-sample considerations, computational speedup techniques, and batch sampling. Section \ref{subsec:numerical_experiments} presents numerical experiments designed to verify theoretical results and demonstrate the algorithm's practical performance.

\subsection{Practical Implementation Strategies}
\label{subsec:practical_implementation}

While the T-a-S-CS algorithm provides strong theoretical guarantees, practical implementations benefit from computational speedup techniques. We first discuss the inherent finite-sample conservativeness of the stopping criterion, and then present two complementary speedup strategies, n-step update and batch sampling.

\textbf{Finite-sample Conservativeness and Choice of Stopping Criterion.}
\label{subsubsec:asymptotic_finite_sample}
Theorem~\ref{thm:opt_convergnce} establishes that T-a-S-CS attains the lower bound $T^*(\boldsymbol{\mu})$ asymptotically as $\delta \to 0$. In practice where we set finite values of $\delta$, however, Track-and-Stop procedures can stop noticeably later than $T^*(\boldsymbol{\mu})\ln(1/\delta)$. To explain, this is because the stopping thresholds required by the proofs, e.g., in \citet[Proposition~21]{kaufmann2021mixture}, carry lower-order slack terms ($\ln\ln t$, $\log K$, and constants). These slack terms are necessary for uniform correctness, become dominated by $\ln(1/\delta)$ as we take the limit $\delta \to 0$ and therefore do not affect the asymptotic properties, but are loose on any specific values of $\delta$. Further, the over-conservativeness of the resulting stopping criterion, even when the stylised stopping threshold of $\ln((\ln t+1)/\delta)$ is adopted, has been documented for the original Track-and-Stop algorithm \citep{garivier2016optimal,russac2021b}. Specifically, over-conservative means that the stopping criterion results in much lower empirical probability of error than the risk assessment, and therefore requiring more samples than necessary to stop the procedure.

Therefore, asymptotic optimality and finite-sample conservativeness coexist, describing different aspects of the same algorithm. We further conduct experiments in Section~\ref{subsubsec:asymptotic_optimality} to empirically disentangle the two aspects, and to identify which features of Theorem~\ref{thm:opt_convergnce} survive at practical precision.

\textbf{n-Step Weight Update.} In practice, it is not necessary to solve the external optimization problem (\ref{eq:external_opt}) to full convergence at each time step $t$. Instead, at each iteration, we execute three steps: (\textit{i}) \textbf{solve} the internal optimization problem (e.g., (\ref{eq:gaussian_internal_opt}) for Gaussian SBFC) once using the current estimate $\hat{\boldsymbol{\mu}}(t)$ to obtain a subgradient; (\textit{ii}) \textbf{update} the weights via 
$n$ step of projected subgradient ascent ($\mathbf{w}_{t-1} \to \mathbf{w}_t$); and (\textit{iii}) \textbf{sample} the next observation according to the C-Tracking rule using $\mathbf{w}_t$. This $n$-step update strategy reduces computational cost while maintaining convergence guarantees (as the estimates $\hat{\boldsymbol{\mu}}(t) \to \boldsymbol{\mu}$ drive $\mathbf{w}_t \to \mathbf{w}^*(\boldsymbol{\mu})$). Numerical experiments in Section \ref{subsubsec:asymptotic_optimality} confirm that this approach achieves desirable sample complexity.

\textbf{Batch Sampling.} The T-a-S-CS algorithm is inherently sequential: each sampling decision depends on the updated estimates from all previous observations, so the algorithm must wait for each sample outcome before determining the next action. In many practical settings---such as clinical trials, A/B testing, or field experiments---collecting a single observation can take considerable time, and this waiting cost dominates computational cost. Batch sampling mitigates the waiting cost by exploiting the fact that, when the estimate $\hat{\boldsymbol{\mu}}(t)$ is close to convergence, successive weight updates $\mathbf{w}_t \to \mathbf{w}_{t+1}$ become increasingly small, and since the sampling decision $(k_t, l_t)$ is discrete (determined by the C-Tracking rule), small changes in $\mathbf{w}$ may not alter the sampling decision. Rather than recomputing weights after every sample, we compute $\mathbf{w}_t$ at iteration $t$, draw a batch of $B$ samples \emph{in parallel} using the same weights, update estimates $\hat{\boldsymbol{\mu}}(t+B)$ after the batch, and recompute weights $\mathbf{w}_{t+B}$ for the next batch. Since samples within each batch are drawn independently given $\mathbf{w}_t$ and can be collected simultaneously, this reduces the number of sequential sampling rounds at the cost of a moderate increase in sample complexity. 

%% file: sections/5_numerical/2_experiments.tex
\subsection{Numerical Experiments}
\label{subsec:numerical_experiments}

This section presents numerical experiments designed to validate theoretical results and demonstrate the practical performance of the T-a-S-CS algorithm. Section~\ref{subsubsec:asymptotic_optimality} empirically verifies the $O(\ln(1/\delta))$ scaling rate predicted by Theorem~\ref{thm:opt_convergnce} and characterizes the finite-sample gap between the empirical stopping time and the asymptotic prediction $T^*(\boldsymbol{\mu})\ln(1/\delta)$ (cf.\ Section~\ref{subsubsec:asymptotic_finite_sample}). Section~\ref{subsubsec:robustness} shifts focus from the stopping rule to the allocation rule, comparing T-a-S-CS against the classical T-a-S algorithm in a fixed-budget regime at operationally relevant scales ($K$ up to 100) under distributional misspecification, to demonstrate the value of subgroup-level allocation and its implicit optimality-feasibility balancing. Section~\ref{subsubsec:extension_comparison} compares the SBFC and $\gamma$-SBFC problem formulations in terms of sample complexity, computational cost, and flexibility.

\textbf{General Setup.} Unless otherwise stated, all experiments use the following implementation: each policy-subpopulation pair $(k,\ell)$ is initialized with 5 samples; weights are updated via one step of projected subgradient ascent with stepsize $\alpha=1$ at each iteration (see Section \ref{subsec:practical_implementation}); the stopping rule uses the generalized likelihood ratio (GLR) statistic with threshold $\beta(t,\delta)=\ln((1+\ln t)/\delta)$; and each configuration is replicated 3000 times to compute the average stopping time $\hat{\tau}_\delta$ with $95\%$ confidence intervals and the empirical probability of correct selection $\hat{P}_{\boldsymbol{\mu}}$. \footnote{All experiments were run on a single AWS EC2 \texttt{c7i.8xlarge}
instance (32 vCPUs, 64~GB RAM) with Gurobi 12.0 (WLS license).
Each replication was executed as an independent process, parallelized
across vCPUs via GNU \texttt{parallel}. Aggregated over the six
settings reported in Sections~\ref{subsubsec:asymptotic_optimality}--%
\ref{subsubsec:extension_comparison}, the total $18{,}000$ replications required approximately $2{,}400$ CPU-hours, dominated by the MIQP-based $\gamma$-variant in Section~\ref{subsubsec:extension_comparison}
(mean $\approx 41$~min per replication) and the operational scale study in Section~\ref{subsubsec:robustness}
(mean $\approx 9$~min per replication); all other settings averaged under one minute per replication.}

%% file: sections/5_numerical/2_experiments/1_asymptotic.tex
\subsubsection{Asymptotic Scaling and the Finite-Sample Gap}
\label{subsubsec:asymptotic_optimality}

Theorem~\ref{thm:opt_convergnce} establishes that the T-a-S-CS algorithm achieves asymptotic optimality: $\lim_{\delta\to 0}\frac{\mathbb{E}_{\boldsymbol{\mu}}[\tau_\delta]}{\ln(1/\delta)}=T^*(\boldsymbol{\mu})$. This asymptotic statement characterizes the {rate} at which the expected stopping time grows with $\ln(1/\delta)$. As discussed in Section~\ref{subsubsec:asymptotic_finite_sample}, Track-and-Stop procedures tend to stop later than the asymptotic prediction at practical precision levels. This section empirically examines the interplay between the two aspects: we verify the asymptotic scaling rate and characterize the magnitude of the finite-sample gap.

\textbf{Problem Instance.} We consider an instance with $K=5$ policies and $L=3$ subpopulations under equal weights $q_l = 1/3$ and fairness constraint $C_{\min} = 0$. The mean matrix $\boldsymbol{\mu}$ is given in Appendix \ref{appendix: scaling_problem_instance}. Observations are Gaussian with $\sigma = 0.5$.

\textbf{Experiment Design.}
Using the instance above, we test 10 geometrically-spaced values $\delta \in [0.0005, 0.2]$ with 3000 independent replications each. For each $\delta$, we record the empirical mean stopping time $\hat{\tau}_\delta$ and examine the normalized ratio $\hat{\tau}_\delta / \ln(1/\delta)$, which Theorem~\ref{thm:opt_convergnce} predicts converges to $T^*(\boldsymbol{\mu})$ as $\delta \to 0$.

Figure \ref{fig:exp_a_scaling} presents the key scaling results. The left panel plots the normalized ratio $\hat{\tau}_\delta / \ln(1/\delta)$ against $\delta$. As $\delta$ decreases, the ratio declines toward the theoretical constant $T^*(\boldsymbol{\mu}) \approx 57.32$ (dashed line), consistent with the asymptotic convergence predicted by Theorem~\ref{thm:opt_convergnce}. At every finite $\delta$ tested, the empirical ratio remains above $T^*(\boldsymbol{\mu})$. This is expected: The stopping threshold $\beta(t,\delta) = \ln((1+\ln t)/\delta)$ carries lower-order terms ($\ln\ln t$ and constants) that inflate the stopping time beyond $T^*(\boldsymbol{\mu})\ln(1/\delta)$ at moderate precision levels. In addition, the $n$-step weight update described in Section~\ref{subsec:practical_implementation} can contribute to the gap between $\hat{\tau}_\delta/\ln(1/\delta)$ and $T^*(\boldsymbol{\mu})$. The right panel confirms that $\hat{\tau}_\delta$ scales linearly with $\ln(1/\delta)$, validating the $O(\ln(1/\delta))$ growth rate.

\begin{figure}[htbp]
  \centering
  \begin{minipage}[t]{0.48\textwidth}
    \centering
    \includegraphics[width=\textwidth]{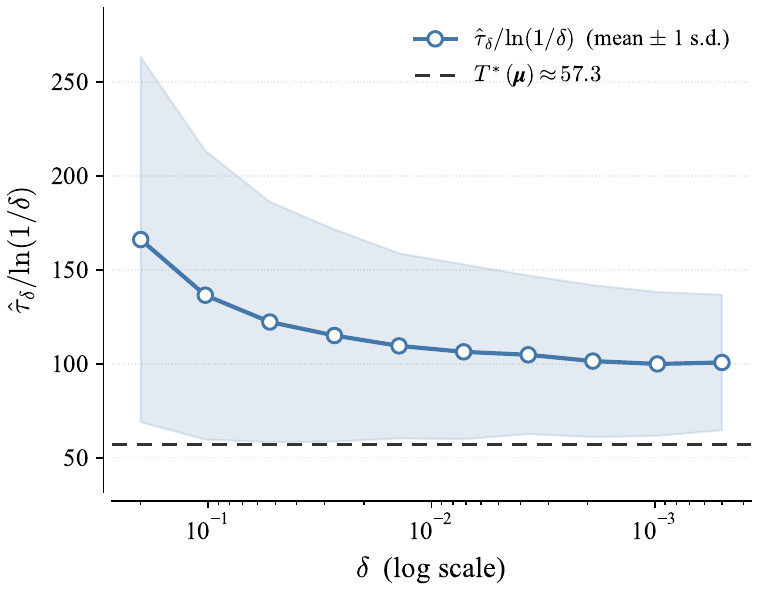}
  \end{minipage}
  \hfill
  \begin{minipage}[t]{0.48\textwidth}
    \centering
    \includegraphics[width=\textwidth]{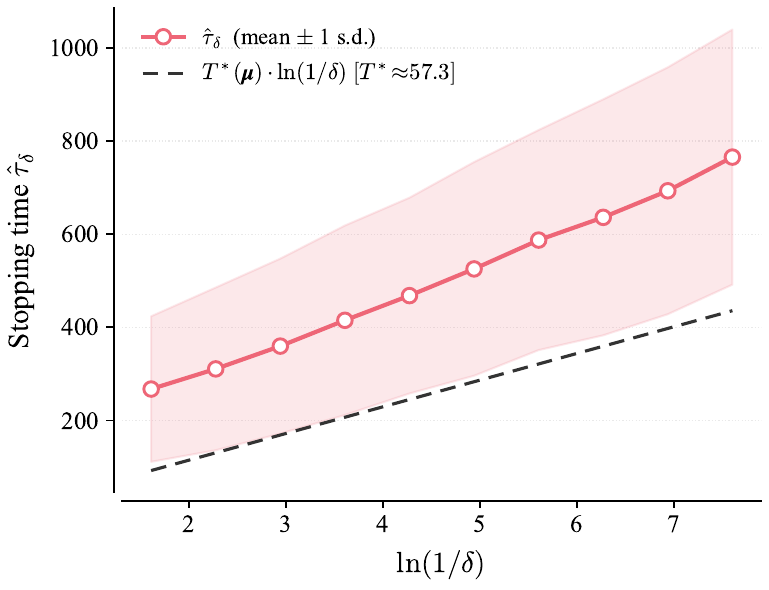}
  \end{minipage}
  \caption{Asymptotic scaling analysis. Left: Normalized stopping time $\hat{\tau}_\delta / \ln(1/\delta)$ versus $\delta$, with the asymptotic constant $T^*(\boldsymbol{\mu}) \approx 57.32$ (dashed line). The gap above $T^*(\boldsymbol{\mu})$ reflects finite-sample conservativeness (Section~\ref{subsubsec:asymptotic_finite_sample}). Right: Stopping time $\hat{\tau}_\delta$ versus $\ln(1/\delta)$, demonstrating the predicted linear scaling.}
  \label{fig:exp_a_scaling}
\end{figure}

Figure \ref{fig:exp_a_heatmap} displays the empirical sampling allocation at $\delta = 0.0005$. The algorithm concentrates samples on the two hardest-to-resolve comparisons: the pair $(a_3, \text{subpop.~1})$ receives the highest weight to verify the binding feasibility constraint, while $(a_1, \cdot)$ and $(a_2, \cdot)$ receive moderate weight for optimality discrimination. Policies $a_4$ and $a_5$ are allocated near-zero weight, consistent with their rapid elimination.

\begin{figure}[htbp]
  \centering
  \includegraphics[width=0.55\textwidth]{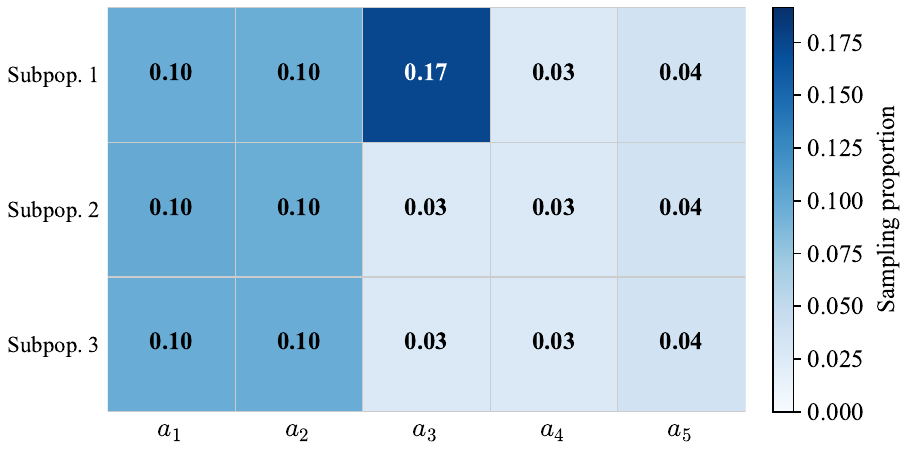}
  \caption{Empirical sampling proportions at $\delta = 0.0005$, displayed as a $5 \times 3$ heatmap (rows = policies, columns = subpopulations). Cell values are mean sampling proportions over 3000 replications.}
  \label{fig:exp_a_heatmap}
\end{figure}

Detailed numerical results are reported in Table~\ref{table:exp_a_results} (Appendix~\ref{appendix:asymptotic_scaling_details}). Most notably, all configurations achieve empirical correctness probability at least $\hat{P}_{\boldsymbol{\mu}} = 0.9997$, still well exceeding the required $1-\delta$ guarantee; this remains a direct manifestation of the over-conservativeness discussed in Section~\ref{subsubsec:asymptotic_finite_sample}. Additional results from a complementary perspective that involves varying the problem instance and changing the feasibility gap are provided in Appendix~\ref{appendix:instance_sensitivity}.

%% file: sections/5_numerical/2_experiments/2_robustness.tex
\subsubsection{Allocation Efficiency at Operational Scale}
\label{subsubsec:robustness}

The preceding experiment reveals that the stopping criterion of T-a-S-CS is over-conservative at practical precision levels. For large-scale problems with many candidate policies, this conservativeness can make the sequential stopping rule impractical: the algorithm may demand a prohibitively long sampling horizon before it stops. However, the over-conservative stopping rule does not diminish the value of the \emph{allocation rule} itself---the weight vector $\mathbf{w}^*$ that the algorithm computes and tracks still encodes an efficient distribution of sampling effort across policy-subpopulation pairs.  Moreover, as shown in Section~\ref{subsec:counter_set}, $\mathbf{w}^*$ implicitly balances the optimality-feasibility trade-off (Section~\ref{subsec:opt_feas_tradeoff}) by maximizing the minimum confidence between verifying that the selected policy is optimal and verifying that it is feasible. This section isolates and evaluates the allocation rule of T-a-S-CS at operationally relevant scales ($K$ up to 100), where the stopping criterion is set aside and performance is measured by the probability of correct selection after a pre-specified number of sampling rounds. As the baseline, we use the classical Track-and-Stop (T-a-S) algorithm \citep{kaufmann2016complexity}, which allocates samples at the policy level based only on overall performance gaps and does not incorporate subpopulation-level information. By comparing T-a-S-CS against T-a-S, we directly quantify the benefit of the subpopulation-level allocation and its implicit optimality-feasibility balancing.

\medskip
\noindent\textbf{Problem Instance.}
We use $L = 3$ subpopulations with equal weights $q_\ell = 1/3$ and fairness threshold $C_{\min} = 0.2$. True observations follow $X_{k,\ell} \sim \text{Bernoulli}(\mu_{k,\ell})$. The optimal policy has $\boldsymbol{\mu}_1 = (0.45, 0.45, 0.45)$, with small gaps to sub-optimal feasible competitors ($\mu_{k,\ell} \sim U[0.32, 0.42]$ for all $\ell$). Infeasible policies are drawn from two groups: those with high overall performance ($\mu_{k,1} \sim U[0.00, 0.10]$, $\mu_{k,\ell} \sim U[0.45, 0.70]$ for $\ell \geq 2$) and those with low overall performance ($\mu_{k,1} \sim U[0.00, 0.10]$, $\mu_{k,\ell} \sim U[0.20, 0.35]$ for $\ell \geq 2$). 

This instance also introduces a distributional mismatch: the true data-generating process is Bernoulli, while the T-a-S-CS algorithm uses the Gaussian KL-divergence ($\sigma^2 = 1$) for weight computation. This deliberate misspecification tests whether the allocation rule remains effective when the distributional family is not correctly specified---a common practical scenario where the Gaussian approximation is preferred for its computational tractability (Appendix~\ref{appendix:gaussian_sbfc}).

\medskip
\noindent\textbf{Experiment Design.}
We test $K \in \{50, 100\}$ and compare two allocation strategies. The first is T-a-S-CS with Gaussian KL-divergence, batch size $B = 10$, with 3 subgradient steps per batch. At each step, the tracking rule selects a specific policy-subpopulation pair $(k, \ell)$ to sample. The second is T-a-S with Gaussian KL-divergence, which allocates at the policy level using the classical Track-and-Stop weights, determined by overall performance gaps $\bar{\mu}_{k^*} - \bar{\mu}_k$ alone. At each step, the tracking rule selects a policy $k$ to sample, and a subpopulation $\ell$ is drawn uniformly at random. Both strategies are run for a fixed budget of $T = K \times L \times 40$ rounds: $T = 6{,}000$ at $K=50$ and $T = 12{,}000$ at $K=100$. At the end of $T$ rounds, each algorithm recommends $\hat{k}^* = \arg\max_{k \in \hat{\mathcal{C}}(\hat{\boldsymbol{\mu}}(T))} \hat{\bar{\mu}}_k$. Each configuration is replicated 3000 times.

\medskip
\noindent\textbf{Results.}
Table~\ref{table:part_b_pcs} reports the empirical probability of correct selection (PCS) at the end of the sampling horizon, with 95\% binomial confidence intervals. Figure~\ref{fig:part_b_pcs} shows the PCS trajectory over sampling rounds at $K \in \{50, 100\}$. The subpopulation-level allocation of T-a-S-CS dramatically outperforms the policy-level allocation of T-a-S at every checkpoint. At $K=50$, T-a-S-CS reaches $95.9\%$ PCS while T-a-S achieves only $52.5\%$; at $K=100$, the gap widens to $+57.3$ percentage points ($96.3\%$ vs.\ $39.0\%$). Notably, T-a-S performance \emph{degrades} as $K$ grows, because its policy-level weights spread effort across increasingly many policies without any mechanism to target the subpopulations where feasibility violations occur. In contrast, T-a-S-CS allocates at the $(k, \ell)$ level and implicitly balances the optimality-feasibility trade-off (Section~\ref{subsec:opt_feas_tradeoff}), concentrating samples on the subpopulations that matter most for resolving each policy's feasibility status.

\begin{table}[ht]
\centering
\small
\renewcommand{\arraystretch}{1.2}
\begin{tabular}{rrcccr}
\toprule
$K$ & $T$ & $\hat{P}_{\mathrm{CS}}$ T-a-S-CS ($\pm$ 95\% CI) &
             $\hat{P}_{\mathrm{CS}}$ T-a-S ($\pm$ 95\% CI) & Gain \\
\midrule
50  & $6{,}000$  & $0.959 \pm 0.007$ & $0.525 \pm 0.018$ & $+0.434$ \\
100 & $12{,}000$ & $0.963 \pm 0.007$ & $0.390 \pm 0.018$ & $+0.573$ \\
\bottomrule
\end{tabular}
\caption{Empirical PCS after $T$ sampling rounds (3000 replications, $\pm$ 95\% binomial CI). The ``Gain'' column reports $\hat{P}_{\mathrm{CS}}(\text{T-a-S-CS}) - \hat{P}_{\mathrm{CS}}(\text{T-a-S})$.}
\label{table:part_b_pcs}
\end{table}

\begin{figure}[htbp]
  \centering
  \includegraphics[width=\textwidth]{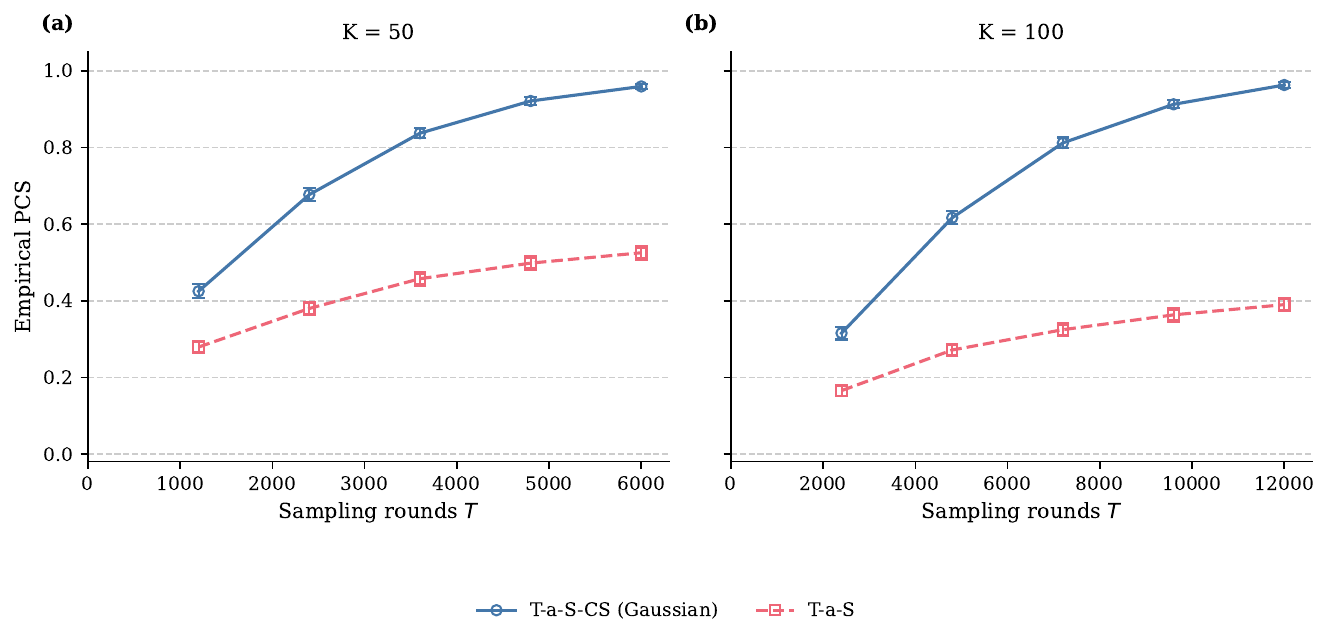}
  \caption{Empirical PCS of T-a-S-CS (Gaussian) and T-a-S over sampling rounds at $K=50$ (left) and $K=100$ (right). Error bars show 95\% binomial CIs at each checkpoint. The subpopulation-level allocation of T-a-S-CS dramatically outperforms the policy-level allocation of T-a-S, which cannot direct effort toward feasibility verification.}
  \label{fig:part_b_pcs}
\end{figure}

These results demonstrate that the subpopulation-level allocation derived from the lower-bound optimization retains substantial practical value even when (i) the stopping criterion is not used, (ii) the distributional family is misspecified, and (iii) the number of candidate policies is large. The policy-level allocation of classical T-a-S, which does not account for subpopulation structure or feasibility constraints, is fundamentally inadequate for the SBFC problem. The ability of T-a-S-CS to balance optimality verification and feasibility verification at the $(k, \ell)$ level is the key driver of its practical advantage.

%% file: sections/5_numerical/2_experiments/3_extension_comparison.tex
\subsubsection{Hard Constraints vs.\ Soft Constraints}
\label{subsubsec:extension_comparison}

Section \ref{subsec:second_extension} introduces the $\gamma$-SBFC formulation with soft fairness constraints as an alternative to the hard constraint formulation. Appendix \ref{appendix:soft_constraint} shows that under Gaussian observations, the SBFC subproblem is a convex QP while the $\gamma$-SBFC subproblem is a mixed-integer QP (MIQP); both can be efficiently solved using standard solvers such as Gurobi. This experiment empirically compares SBFC (hard constraints) and $\gamma$-SBFC (soft constraints) to address the question: How does the choice between hard and soft constraints affect sample complexity, and how does the penalty parameter $\gamma$ influence algorithm behavior?

\textbf{Problem Instance.} We consider $K=10$ policies and $L=3$ subpopulations with equal weights $q_\ell = 1/3$. Table~\ref{table:extension_instance} (Appendix~\ref{appendix:extension_instance}) reports the mean matrix. The fairness constraint requires $\mu_{k,\ell} \geq C_{\min} = 0.2$. Of the ten policies, five are feasible (policies 1--3, 9, 10) and five are infeasible (policies 4--8). Policy~1 is the optimal feasible policy with $\bar\mu_1 = 0.550$. Policy~4 achieves the highest overall performance $\bar\mu_4 = 0.700$ but is infeasible due to $\mu_{4,1} = 0.05 < 0.2$. Observations are Gaussian with $\sigma = 0.5$. We test $\delta = 0.2$ with 3{,}000 independent replications.

\textbf{Two Formulations.} We compare: (1) \textbf{SBFC (hard constraints)}, which requires $\mu_{k,\ell} \geq C_{\min}$ and selects policy~1 as optimal; and (2) \textbf{$\gamma$-SBFC (soft constraints)}, which penalizes violations through the modified performance $\mu_k^\gamma$ defined in \eqref{eq:linear_relaxed}. For policy~4, $\mu_4^\gamma = 0.700 - \tfrac{1}{3}\gamma \cdot 0.15$ and $\mu_1^\gamma = 0.550$, yielding a critical threshold $\gamma^* = 3.0$: for $\gamma < \gamma^*$, policy~4 is optimal under $\gamma$-SBFC; for $\gamma > \gamma^*$, policy~1 is optimal (same as SBFC). We test $\gamma \in \{0.5, 1.0, 5, 10\}$, placing two values below $\gamma^*$ and two above.

\textbf{Theoretical Sample Complexity $T^*$ vs.\ $\gamma$.}
Before examining empirical results, we analyze how the theoretical sample complexity lower bound $T^*(\boldsymbol{\mu}, \gamma)$ varies with the penalty parameter $\gamma$. For the instance above, we solve the external optimization problem $\sup_{\mathbf{w}} f^*_{\boldsymbol{\mu},\gamma}(\mathbf{w})$ from Corollary \ref{cor:lower_bound_ext2} using the MIQP formulation (\ref{eq:gamma_miqp}) with projected subgradient ascent. Figure \ref{fig:T_star_vs_gamma} shows $T^*(\boldsymbol{\mu}, \gamma)$ as a function of $\gamma$, with the SBFC baseline $T^*(\boldsymbol{\mu}) \approx 170$ for comparison. The plot reveals three regimes. For $\gamma < \gamma^* = 3.0$, the optimal penalized policy is $a_4$ (infeasible under hard constraints), and sample complexity increases as $\gamma \to \gamma^*$ because the penalized performance gap shrinks ($T^*(\boldsymbol{\mu}, 0.5) \approx 506$, $T^*(\boldsymbol{\mu}, 1.0) \approx 993$). At the critical threshold $\gamma = \gamma^*$, the penalized performances of policies~4 and~1 coincide, so $T^*(\boldsymbol{\mu}, \gamma) \to \infty$. For $\gamma > \gamma^*$, the optimal policy switches to $a_1$ (consistent with SBFC), and the gap widens as $\gamma$ grows, driving $T^*$ back down to approach (but stays above) the SBFC baseline $T^*(\boldsymbol{\mu}) = 170$.

\textbf{Results.} 
Table \ref{table:extension_comparison} reports empirical stopping times and correctness, and Figure \ref{fig:weights_extension} (Appendix~\ref{appendix:extension_weights}) shows the sampling proportions. All formulations achieve $\hat{P}_{\boldsymbol{\mu}} \geq 0.99 \geq 1 - \delta$, confirming valid $\delta$-PAC guarantees for both hard and soft constraints. Consistent with the theoretical analysis in Figure~\ref{fig:T_star_vs_gamma}, the empirical stopping times of $\gamma$-SBFC are uniformly larger than those of SBFC, and the penalty parameter $\gamma$ controls which policy is selected: for $\gamma < \gamma^*$, the algorithm selects the infeasible but high-performing policy~4, while for $\gamma > \gamma^*$ it selects policy~1, the same as SBFC. In summary, $\gamma$-SBFC incurs higher sample complexity and greater computational cost per iteration (MIQP vs.\ QP), but offers additional flexibility by allowing the decision-maker to calibrate, through $\gamma$, the degree to which fairness violations are tolerated.

\begin{figure}[ht]
\centering
\begin{minipage}[t]{0.46\textwidth}
    \centering
    \vspace{0pt}
    \includegraphics[width=\linewidth]{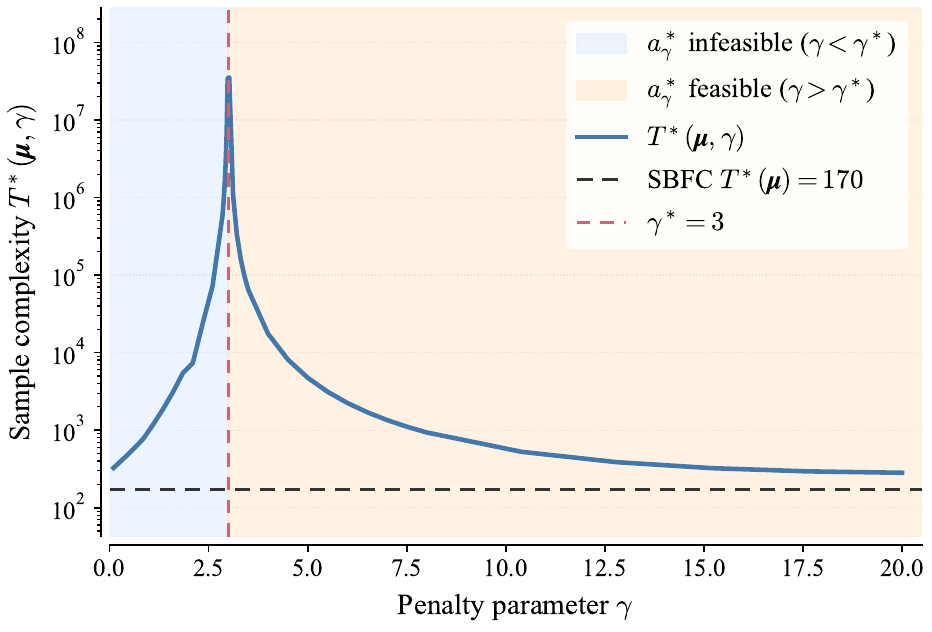}
    \captionof{figure}{Theoretical sample complexity $T^*(\boldsymbol{\mu}, \gamma)$ vs.\ penalty $\gamma$. Dashed line: SBFC baseline; vertical line at $\gamma^* = 3$: switch of the optimal penalized policy from policy~4 (infeasible) to policy~1 (feasible).}
    \label{fig:T_star_vs_gamma}
\end{minipage}\hfill
\begin{minipage}[t]{0.51\textwidth}
    \centering
    \vspace{0pt}
    \small
    \renewcommand{\arraystretch}{1.25}
    \setlength{\tabcolsep}{4pt}
    \begin{tabular}{lcccc}
    \toprule
    Formulation & $a^*$ & $\hat{\tau}_\delta$ & Std & $\hat{P}_{\boldsymbol{\mu}}$ \\
    \midrule
    SBFC & $a_1$ & $853 \pm 5$ & 271 & 0.999 \\
    $\gamma$-SBFC ($\gamma=0.5$) & $a_4$ & $969 \pm 10$ & 550 & 1.000 \\
    $\gamma$-SBFC ($\gamma=1.0$) & $a_4$ & $1{,}473 \pm 17$ & 946 & 0.998 \\
    $\gamma$-SBFC ($\gamma=5$) & $a_1$ & $3{,}179 \pm 38$ & 2{,}104 & 0.991 \\
    $\gamma$-SBFC ($\gamma=10$) & $a_1$ & $1{,}080 \pm 8$ & 440 & 0.997 \\
    \bottomrule
    \end{tabular}
    \vspace{0.6em}
    \captionof{table}{Performance comparison between SBFC and $\gamma$-SBFC formulations ($K=10$, $L=3$, $\delta = 0.2$, 3{,}000 replications). Column $a^*$ indicates the true optimal policy; $\hat{\tau}_\delta$ reports mean $\pm$ SE.}
    \label{table:extension_comparison}
\end{minipage}
\end{figure}

%% file: sections/5_numerical/3_real_demonstration.tex
\subsection{Real-World Demonstration: The International Stroke Trial}
\label{subsec:real_scenario}

As discussed in the Introduction, a commitment-stage decision-maker must choose a single, universal policy whose performance is auditable in every pre-specified subpopulation. Clinical regulators make this logic explicit: a recommended intervention must perform adequately in vulnerable subpopulations, not only on average~\citep{ema2019subgroup}. In this section, we apply our SBFC framework to the International Stroke Trial (IST)~\citep{ist1997lancet,sandercock2011ist}, a landmark study of antithrombotic therapy in acute ischaemic stroke. We show the effectiveness of T-a-S-CS in identifying the subpopulation-feasible optimum on genuine clinical data.

\medskip
\noindent\textbf{Problem framing.}
IST used a $2 \times 2$ factorial design, which we treat as $K = 4$ candidate policies: \emph{Aspirin only}, \emph{Heparin only}, \emph{Both}, and \emph{Neither}. Subpopulations are defined by age, giving $L = 2$: patients \emph{younger than 80} ($q_1 = 0.7177$) and \emph{aged 80 or older} ($q_2 = 0.2823$) from the full trial cohort ($N = 19{,}435$). The outcome $X_{k,\ell}$ is a binary indicator equal to $1$ for a safe 14-day outcome (alive, free of recurrent stroke and major bleeding), and $0$ otherwise. The fairness constraint requires $\mu_{k,\ell} \geq 5/6$ for all $\ell$. The original IST employed uniform randomization across all arms
without subpopulation targeting. We use the trial data as a
nonparametric simulator: at each round, when the algorithm
allocates a sample to cell~$(k,\ell)$, one patient record is
drawn with replacement from the corresponding
(treatment, age group) cell and its 14-day composite outcome is
returned. This setup asks a retrospective question: given the
same patient population, how much more efficiently could a
sequential, fairness-aware allocation have identified the
subpopulation-feasible optimum compared with the uniform
design actually used?

\medskip
\noindent\textbf{Cell means and feasibility structure.}
Table~\ref{table:ist_means} reports the empirical cell means $\mu_{k,\ell}$ and feasibility status under $C_{\min} = 5/6$. The policy with the highest overall mean, \emph{Neither}, fails the threshold for the older group ($\mu_{\text{Neither},\geq80} = 0.8248 < 5/6$). Only \emph{Aspirin only} clears the feasibility threshold for both groups. Thus, the subpopulation-feasible optimum is $k^* = \text{Aspirin only}$, illustrating how sub population fairness can overturn the unconstrained ranking (Case \ref{example:2} from Section~\ref{subsec:opt_feas_tradeoff}).

\begin{table}[ht]
\centering
\small
\renewcommand{\arraystretch}{1.15}
\begin{tabular}{lcccc}
\toprule
Policy & Age $<$ 80 ($q_1{=}0.7177$) & Age $\geq$ 80 ($q_2{=}0.2823$) & $\bar\mu_k$ & Feasible \\
\midrule
Aspirin only & $0.8925$ & $\mathbf{0.8429}$ & $0.8785$ & $\checkmark$ \\
Heparin only & $0.8751$ & $\mathbf{0.7980}$ & $0.8534$ & $\times$ \\
Both         & $0.8705$ & $\mathbf{0.7843}$ & $0.8462$ & $\times$ \\
Neither      & $0.9013$ & $\mathbf{0.8248}$ & $0.8797$ & $\times$ \\
\bottomrule
\end{tabular}
\caption{Cell means $\mu_{k,\ell}$ on the IST sample ($N = 19{,}435$). Only \emph{Aspirin only} satisfies the fairness constraint $C_{\min} = 5/6$ across both age groups. \emph{Neither} has the highest overall mean $\bar\mu_k$ but is subpopulation-infeasible.}
\label{table:ist_means}
\end{table}

\medskip
\noindent\textbf{Sampling model and experiment design.}
The performance gaps across policies in the IST data are small
(0.5--1.3 percentage points in the binding elderly subpopulation),
so the GLR stopping criterion requires a sample size substantially
larger than would be available in a realistic acute stroke trial.
We therefore evaluate the allocation rule under a fixed number of samples; the contribution of T-a-S-CS in this setting is the efficiency in sample allocation. To mimic sequential deployment, we simulate a sequential trial by sampling via bootstrap with replacement from the empirical cells. Since outcomes are known to be binary, here we use the correctly specified Bernoulli KL divergences for T-a-S-CS. We evaluate allocation rules over a fixed budget of $T = 6{,}000$ rounds, and use Uniform sampling and T-a-S as baselines. We measure empirical probability of correct selection (PCS) across $3{,}000$ replications at five checkpoints.

\medskip
\noindent\textbf{Results.}
Table~\ref{table:ist_pcs} shows the PCS trajectories. T-a-S-CS dominates both baselines at every checkpoint, finishing $7.1$ percentage points above arm-level T-a-S and $10.3$ percentage points above Uniform at $T = 6{,}000$. Equivalently, T-a-S-CS matches the terminal PCS of arm-level T-a-S (Uniform) using roughly $40\%$ ($50\%$) fewer samples.


\begin{table}[ht]
\centering
\small
\renewcommand{\arraystretch}{1.15}
\begin{tabular}{rccc}
\toprule
$T_c$ & $\hat P_{\mathrm{CS}}$ T-a-S-CS ($\pm$ 95\% CI) & $\hat P_{\mathrm{CS}}$ T-a-S (arm) ($\pm$ 95\% CI) & $\hat P_{\mathrm{CS}}$ Uniform ($\pm$ 95\% CI) \\
\midrule
$3{,}000$ & $0.620 \pm 0.017$ & $0.568 \pm 0.018$ & $0.554 \pm 0.018$ \\
$3{,}750$ & $0.667 \pm 0.017$ & $0.614 \pm 0.017$ & $0.572 \pm 0.018$ \\
$4{,}500$ & $0.692 \pm 0.017$ & $0.637 \pm 0.017$ & $0.596 \pm 0.018$ \\
$5{,}250$ & ${0.723 \pm 0.016}$ & $0.647 \pm 0.017$ & $0.615 \pm 0.017$ \\
$6{,}000$ & ${0.737 \pm 0.016}$ & $0.666 \pm 0.017$ & $0.634 \pm 0.017$ \\
\bottomrule
\end{tabular}
\caption{Empirical PCS after $T_c$ sampling rounds ($3{,}000$ replications, $\pm$ 95\% binomial CI). T-a-S-CS identifies the feasible optimal policy substantially faster than both baselines. See Appendix~\ref{appendix:ist} for the corresponding feasibility structure, data construction and sample allocation.}
\label{table:ist_pcs}
\end{table}

%% file: appendices/A_lowerbound/1_theorem_proof.tex
\subsection{Proof of Theorem \ref{thm:lower_bound}}
\label{appendix:thm_lb_proof}

The proof of Theorem \ref{thm:lower_bound} is based on the change of measure inequality presented in Lemma 1 of \cite{kaufmann2016complexity}. We first restate the lemma for clarity, adapting its binary relative entropy notation $\mathrm{d}(\cdot, \cdot)$ to $\mathrm{kl}(\cdot, \cdot)$.

\begin{lemma}[Lemma 1 in \cite{kaufmann2016complexity}]
\label{lem:lemma1_restated}
Let $\nu$ and $\nu'$ be two bandit models with $K$ arms such that for all $a$, the distributions $\nu_a$ and $\nu_a'$ are mutually absolutely continuous. For any almost-surely finite stopping time $\sigma$ with respect to $(\mathcal{F}_t)$,
$$
\sum_{a=1}^{K}\mathbb{E}_{\nu}[N_{a}(\sigma)]KL(\nu_{a},\nu_{a}^{\prime}) \ge \sup_{\mathcal{E}\in\mathcal{F}_{\sigma}} \mathrm{kl}(\mathbb{P}_{\nu}(\mathcal{E}),\mathbb{P}_{\nu^{\prime}}(\mathcal{E})),
$$
where $KL(\nu_a, \nu_a')$ is the Kullback-Leibler divergence between the arm distributions and $\mathrm{kl}(x,y) := x\log(x/y)+(1-x)\log((1-x)/(1-y))$ is the binary relative entropy (KL divergence between Bernoulli distributions).
\end{lemma}

We now prove our Theorem \ref{thm:lower_bound}. Let $\boldsymbol{\mu}$ be the true parameter instance (equivalent to $\nu$ in Lemma \ref{lem:lemma1_restated}) and let $\boldsymbol{\lambda} \in \mathcal{A}(\boldsymbol{\mu})$ be any alternative instance (equivalent to $\nu'$) for which the set of best arms is different from that of $\boldsymbol{\mu}$. Let $( (A_t, I_t), \tau_\delta, \hat{k}^*_{\tau_\delta})$ be a $\delta$-PAC algorithm, where $\tau_\delta$ is the stopping time.

Let $\mathcal{E}$ be the event of a correct recommendation, i.e., $\mathcal{E} = \{\hat{k}^*_{\tau_\delta} = k^*(\boldsymbol{\mu})\} \in \mathcal{F}_{\tau_\delta}$. By the definition of a $\delta$-PAC algorithm, we have two conditions:
\begin{itemize}
    \item $\mathbb{P}_{\boldsymbol{\mu}}(\mathcal{E}) \ge 1-\delta$ (correctness on the true model).
    \item $\mathbb{P}_{\boldsymbol{\lambda}}(\mathcal{E}) \le \delta$ (bounded error on the alternative model, as $\mathcal{E}$ is an incorrect event under $\boldsymbol{\lambda}$).
\end{itemize}

We adapt Lemma \ref{lem:lemma1_restated} to our $K \times L$ arm setting. The sum over $K$ arms becomes a double summation over $k \in [K]$ and $l \in [L]$. The KL divergence $KL(\nu_a, \nu_a')$ becomes $\mathrm{d}(\mu_{k,l}, \lambda_{k,l})$. Applying the lemma yields:
\begin{equation}
\begin{aligned}
&\sum_{l=1}^L\sum_{k=1}^K \E_{\boldsymbol{\mu}}\left[N_{k,l}(\tau_\delta)\right] \mathrm{d}(\mu_{k,l}, \lambda_{k,l})\\ &\ge ~\mathrm{kl}(\mathbb{P}_{\boldsymbol{\mu}}(\mathcal{E}), \mathbb{P}_{\boldsymbol{\lambda}}(\mathcal{E}))
\end{aligned}
\end{equation}
The function $\mathrm{kl}(x,y)$ is increasing in $x$ and decreasing in $y$ for $x > y$. Therefore, we can lower bound the right-hand side using our PAC conditions:
$$
\mathrm{kl}(\mathbb{P}_{\boldsymbol{\mu}}(\mathcal{E}), \mathbb{P}_{\boldsymbol{\lambda}}(\mathcal{E})) \ge \mathrm{kl}(1-\delta, \delta)
$$
As $\mathrm{kl}(1-\delta, \delta) = \mathrm{kl}(\delta, 1-\delta)$, we arrive at the inequality from the hint, which holds for any $\boldsymbol{\lambda} \in \mathcal{A}(\boldsymbol{\mu})$:
\begin{equation} \label{eq:appendix_hint}
\sum_{l=1}^L\sum_{k=1}^K\E_{\boldsymbol{\mu}}\left[N_{k,l}(\tau_\delta)\mathrm{d}(\mu_{k,l},\lambda_{k,l})\right]\geq\text{kl}(\delta,1-\delta).
\end{equation}

Therefore,
$$
\begin{aligned}
   \text{kl}(\delta,1-\delta)&\leq \inf_{\boldsymbol{\lambda}\in\mathcal{A}(\boldsymbol{\mu})} \sum_{l=1}^L\sum_{k=1}^K\E_{\boldsymbol{\mu}}\left[N_{k,l}(\tau_\delta)\right]\mathrm{d}(\mu_{k,l},\lambda_{k,l})\\
   &=\E_{\boldsymbol{\mu}}\left[\tau_\delta\right]\inf_{\boldsymbol{\lambda}\in\mathcal{A}(\boldsymbol{\mu})} \sum_{l=1}^L\sum_{k=1}^K\frac{\E_{\boldsymbol{\mu}}\left[N_{k,l}(\tau_\delta)\right]}{\E_{\boldsymbol{\mu}}\left[\tau_\delta\right]}\mathrm{d}(\mu_{k,l},\lambda_{k,l})\\
   &\leq \E_{\boldsymbol{\mu}}\left[\tau_\delta\right]\sup_{\mathbf{w}\in \Sigma_{K\times L}}\inf_{{\boldsymbol{\lambda}}\in \mathcal{A}(\boldsymbol{\mu})}\sum_{k\in[K]}\sum_{l\in[L]}w_{k,l}\mathrm{d}\left(\mu_{k,l},\lambda_{k,l}\right).\\
\end{aligned}$$

This gives us the first part of the theorem:
\begin{equation} \label{eq:appendix_first_result}
\mathbb{E}_{\boldsymbol{\mu}}[\tau_\delta] \ge T^*(\boldsymbol{\mu}) \mathrm{kl}(\delta, 1-\delta).
\end{equation}

\noindent\textbf{Asymptotic Limit:} Note that
\begin{equation*}
   \mathrm{kl}(\delta, 1-\delta)
= \left[ \delta \ln\left(\frac{\delta}{1-\delta}\right) + (1-\delta) \ln\left(\frac{1-\delta}{\delta}\right) \right]
\end{equation*}
gives
$$
\lim_{\delta \to 0} \frac{\mathrm{kl}(\delta, 1-\delta)}{\ln(1/\delta)} = 1.
$$
Taking the $\liminf$ of our inequality as $\delta \to 0$ gives:
$$
\liminf_{\delta\to 0}\frac{\mathbb{E}_{\boldsymbol{\mu}}[\tau_\delta]}{\ln(1/\delta)} \ge T^*(\boldsymbol{\mu}) \cdot \left( \lim_{\delta \to 0} \frac{\mathrm{kl}(\delta, 1-\delta)}{\ln(1/\delta)} \right)
$$
$$
\liminf_{\delta\to 0}\frac{\mathbb{E}_{\boldsymbol{\mu}}[\tau_\delta]}{\ln(1/\delta)} \ge T^*(\boldsymbol{\mu})
$$
This completes the proof.

%% file: appendices/A_lowerbound/2_extension_proof.tex
\subsection{Proof of Corollary \ref{cor:lower_bound_m_sbfc}}
\label{appendix:thm_lb_ext1_proof}

We prove that the lower bound in Theorem \ref{thm:lower_bound} extends to the $\mathcal{M}$-SBFC problem with generalized feasibility constraint $\mathcal{C}(\boldsymbol{\mu}, \mathcal{M}) = \{k \in [K] \mid \mathbf{\mu}_k \in \mathcal{M}\}$. The key insight is that the proof structure of Theorem \ref{thm:lower_bound} is \emph{agnostic} to the specific form of the feasibility constraint—it only requires the existence of a well-defined counter set.

\textbf{Proof Strategy.} We demonstrate that the proof of Theorem \ref{thm:lower_bound} remains valid by verifying that:
\begin{enumerate}
    \item The counter set $\mathcal{A}(\boldsymbol{\mu}, \mathcal{M})$ is well-defined under the generalized constraint.
    \item The change of measure inequality (Lemma \ref{lem:lemma1_restated}) applies identically.
    \item The $\sup \inf$ characterization of $T^*(\boldsymbol{\mu}, \mathcal{M})$ preserves its information-theoretic interpretation.
\end{enumerate}

\textbf{Step 1: Counter Set Definition.} 
Recall that the counter set contains all instances with a different optimal policy:
\begin{equation}
    \mathcal{A}(\boldsymbol{\mu}, \mathcal{M}) := \{{\boldsymbol{\lambda}}\in \mathcal{S}\big\vert k^*(\boldsymbol{\lambda}, \mathcal{M})\neq k^*(\boldsymbol{\mu}, \mathcal{M})\},
\end{equation}
where $k^*(\boldsymbol{\mu}, \mathcal{M})=\arg\max_{k\in\mathcal{C}(\boldsymbol{\mu}, \mathcal{M})}\bar{\mu}_{k}$ denotes the best feasible policy under constraint set $\mathcal{M}$. 

The generalization from linear constraints ($\mu_{k,l} \ge C_{\text{min}}$ for all $l \in [L]$) to set-based constraints ($\mathbf{\mu}_k \in \mathcal{M}$) changes the \emph{content} of $\mathcal{A}(\boldsymbol{\mu}, \mathcal{M})$—i.e., which specific alternative instances $\boldsymbol{\lambda}$ belong to the counter set. However, this does not affect the \emph{structural role} that $\mathcal{A}(\boldsymbol{\mu}, \mathcal{M})$ plays in the proof.

\textbf{Step 2: Invariance of the Change of Measure Argument.}
The proof of Theorem \ref{thm:lower_bound} applies the change of measure inequality (Lemma \ref{lem:lemma1_restated}) to establish that for any $\boldsymbol{\lambda} \in \mathcal{A}(\boldsymbol{\mu})$,
\begin{equation*}
\sum_{k,l} \mathbb{E}_{\boldsymbol{\mu}}[N_{k,l}(\tau_\delta)] \mathrm{d}(\mu_{k,l}, \lambda_{k,l}) \ge \mathrm{kl}(\delta, 1-\delta).
\end{equation*}
This inequality depends only on:
\begin{itemize}
    \item The KL divergence between the true instance $\boldsymbol{\mu}$ and alternative instance $\boldsymbol{\lambda}$.
    \item The $\delta$-PAC guarantee: $\mathbb{P}_{\boldsymbol{\mu}}(\mathcal{E}) \ge 1-\delta$ and $\mathbb{P}_{\boldsymbol{\lambda}}(\mathcal{E}) \le \delta$, where $\mathcal{E} = \{\hat{k}^* = k^*(\boldsymbol{\mu}, \mathcal{M})\}$.
\end{itemize}
Crucially, this argument is \emph{independent} of how $k^*(\boldsymbol{\mu}, \mathcal{M})$ is defined—whether through linear constraints, set membership, or any other feasibility rule. As long as $\boldsymbol{\lambda} \in \mathcal{A}(\boldsymbol{\mu}, \mathcal{M})$ (meaning $k^*(\boldsymbol{\lambda}, \mathcal{M}) \neq k^*(\boldsymbol{\mu}, \mathcal{M})$), the event $\mathcal{E}$ is correct under $\boldsymbol{\mu}$ and incorrect under $\boldsymbol{\lambda}$, which is sufficient for applying Lemma \ref{lem:lemma1_restated}.

\textbf{Step 3: Preservation of the $\sup \inf$ Characterization.}
Following the same derivation as in the proof of Theorem \ref{thm:lower_bound} (Appendix \ref{appendix:thm_lb_proof}), we obtain:
\begin{equation*}
\mathbb{E}_{\boldsymbol{\mu}}[\tau_\delta] \ge T^*(\boldsymbol{\mu}, \mathcal{M}) \cdot \mathrm{kl}(\delta, 1-\delta),
\end{equation*}
where
\begin{equation*}
\begin{aligned}
&T^*(\boldsymbol{\mu}, \mathcal{M})^{-1} \\
&= \sup_{\mathbf{w}\in \Sigma_{K\times L}}\inf_{{\boldsymbol{\lambda}}\in \mathcal{A}(\boldsymbol{\mu}, \mathcal{M})}\sum_{k\in[K]}\sum_{l\in[L]}w_{k,l}\mathrm{d}\left(\mu_{k,l},\lambda_{k,l}\right).
\end{aligned}
\end{equation*}
This characterization retains its information-theoretic interpretation: $T^*(\boldsymbol{\mu}, \mathcal{M})$ measures the fundamental difficulty of distinguishing $\boldsymbol{\mu}$ from its hardest-to-distinguish alternative instance $\boldsymbol{\lambda} \in \mathcal{A}(\boldsymbol{\mu}, \mathcal{M})$ under optimal sampling proportions $\mathbf{w}$. The generalization to set-based constraints merely redefines which instances constitute ``alternative instances", but does not change the information-theoretic metric itself.

\textbf{Conclusion.} Since all steps of the proof of Theorem \ref{thm:lower_bound} remain valid when $\mathcal{C}(\boldsymbol{\mu})$ is defined via $\mathbf{\mu}_k \in \mathcal{M}$ instead of $\mu_{k,l} \ge C_{\min}$ for all $l$, we conclude that the asymptotic lower bound holds for the $\mathcal{M}$-SBFC problem:
\begin{equation*}
\liminf_{\delta\to 0}\frac{\mathbb{E}_{\boldsymbol{\mu}}[\tau_\delta]}{\ln(1/\delta)}\geq T^*(\boldsymbol{\mu}, \mathcal{M}).
\end{equation*}
This completes the proof.

%% file: appendices/A_lowerbound/3_extension2_proof.tex
\subsection{Proof of Corollary \ref{cor:lower_bound_ext2}}
\label{appendix:thm_lb_ext2_proof}

The proof of this corollary follows the same structure as that of Theorem \ref{thm:lower_bound} in Appendix \ref{appendix:thm_lb_proof}. The key distinction lies in the definition of optimality and the counter set: under soft fairness constraints with penalty parameters $\boldsymbol{\gamma}$, the optimal policy is determined by maximizing penalized performance $\mu_k^\gamma = \sum_{l=1}^Lq_l\left(\mu_{k,l}+\gamma_{l}\min(\mu_{k,l}-C_\text{min},0)\right)$ rather than maximizing performance subject to hard feasibility constraints.

This change affects the proof in the following way: the counter set is now defined as:
$$
\mathcal{A}(\boldsymbol{\mu}, \boldsymbol{\gamma}) := \{{\boldsymbol{\lambda}}\in \mathcal{S}\mid k^{\gamma,*}(\boldsymbol{\lambda})\neq k^{\gamma,*}(\boldsymbol{\mu})\},
$$
where $k^{\gamma,*}(\boldsymbol{\mu})=\arg\max_{k\in[K]}\mu_k^\gamma$ is the policy with the highest penalized performance.

The remainder of the proof is identical to that in Appendix \ref{appendix:thm_lb_proof}. The change of measure inequality (Lemma \ref{lem:lemma1_restated}) applies with the event $\mathcal{E} = \{\hat{k}^*_{\tau_\delta} = k^{\gamma,*}(\boldsymbol{\mu})\}$. For any $\boldsymbol{\lambda} \in \mathcal{A}(\boldsymbol{\mu}, \boldsymbol{\gamma})$, the $\delta$-PAC guarantees give:
\begin{itemize}
    \item $\mathbb{P}_{\boldsymbol{\mu}}(\mathcal{E}) \ge 1-\delta$ (correctness under $\boldsymbol{\mu}$).
    \item $\mathbb{P}_{\boldsymbol{\lambda}}(\mathcal{E}) \le \delta$ (bounded error under $\boldsymbol{\lambda}$, since $k^{\gamma,*}(\boldsymbol{\lambda}) \neq k^{\gamma,*}(\boldsymbol{\mu})$).
\end{itemize}

Applying the change of measure inequality and following the same derivation yields:
$$
\mathbb{E}_{\boldsymbol{\mu}}[\tau_\delta] \ge T^*(\boldsymbol{\mu}, \boldsymbol{\gamma}) \cdot \mathrm{kl}(\delta, 1-\delta),
$$
where
$$
T^*(\boldsymbol{\mu}, \boldsymbol{\gamma})^{-1} = \sup_{\mathbf{w}\in \Sigma_{K\times L}}\inf_{{\boldsymbol{\lambda}}\in \mathcal{A}(\boldsymbol{\mu}, \boldsymbol{\gamma})}\sum_{k\in[K]}\sum_{l\in[L]}w_{k,l}\mathrm{d}\left(\mu_{k,l},\lambda_{k,l}\right).
$$

Taking $\liminf$ as $\delta \to 0$ gives the asymptotic lower bound. The $\sup \inf$ formulation retains its information-theoretic interpretation: $T^*(\boldsymbol{\mu}, \boldsymbol{\gamma})$ measures the difficulty of distinguishing $\boldsymbol{\mu}$ from its hardest-to-distinguish alternative instance in the penalized performance sense, under optimal sampling proportions.

This completes the proof.

%% file: appendices/B_counterset/1_lemma_proof.tex
\subsection{Proof of Lemma \ref{lemma: counter_set_calculation}}
\label{appendix:thm_cs_proof}

We prove the characterization of the counter set calculation under linear fairness constraints $\mu_{k,l} \ge C_{\min}$ for all $l \in [L]$. The proof constructs counter instances ${\boldsymbol{\lambda}} \in \mathcal{A}(\boldsymbol{\mu})$ by modifying the original instance $\boldsymbol{\mu}$ in ways that change the optimal policy, while minimizing the KL-divergence distance.

\textbf{Proof Strategy.} We consider two cases based on whether any feasible policies exist. In each case, we identify all possible ways to construct a counter instance and solve for the one with minimum distance.

\textbf{Case 1: No Feasible Policies ($\mathcal{C}(\boldsymbol{\mu})=\emptyset$).}

When all policies are infeasible under $\boldsymbol{\mu}$, we have $k^*(\boldsymbol{\mu})=0$. The counter set contains all instances with at least one feasible policy:
$$
\mathcal{A}(\boldsymbol{\mu}) = \{{\boldsymbol{\lambda}}\in \mathcal{S}\mid k^*({\boldsymbol{\lambda}})\neq 0\}.
$$

To construct a counter instance with minimum distance, we select one policy $k \in [K]$ and modify it to become feasible. For policy $k$ to satisfy $k \in \mathcal{C}(\boldsymbol{\lambda})$, we require $\lambda_{k,l} \ge C_{\min}$ for all $l \in [L]$. 

\textbf{Construction:} For each subpopulation $l \in [L]$:
\begin{itemize}
    \item If $\mu_{k,l} \ge C_{\min}$: set $\lambda_{k,l} = \mu_{k,l}$ (no modification needed, $\mathrm{d}(\mu_{k,l}, \lambda_{k,l}) = 0$).
    \item If $\mu_{k,l} < C_{\min}$: set $\lambda_{k,l} = C_{\min}$ (minimum modification to satisfy constraint).
\end{itemize}

The second choice is optimal because, by properties of the KL divergence in one-parameter exponential families, $\mathrm{d}(\mu_{k,l}, \lambda)$ is increasing in $\lambda$ on $[C_{\min}, +\infty)$ when $\mu_{k,l} < C_{\min}$. For all other policies $k' \neq k$, we set $\lambda_{k',l} = \mu_{k',l}$ to minimize distance.

This construction yields the optimization problem (\ref{eq:f_star_k_0}), and selecting the most efficient policy gives:
$$
f_{\boldsymbol{\mu}}^*(\mathbf{w})=\min\left\{f_{\boldsymbol{\mu}}^{1}(\mathbf{w}),\ldots,f_{\boldsymbol{\mu}}^{K}(\mathbf{w})\right\}.
$$

\textbf{Case 2: At Least One Feasible Policy ($\mathcal{C}(\boldsymbol{\mu})\neq\emptyset$).}

Without loss of generality, assume $k^*(\boldsymbol{\mu})=1$ is the current optimal policy. There are two ways to construct a counter instance:

\emph{Option 1: Activate a Different Policy.} Select a policy $k \ge 2$ and modify both $\mathbf{\mu}_1$ and $\mathbf{\mu}_k$ such that:
\begin{itemize}
    \item Policy $k$ becomes feasible: $\lambda_{k,l} \ge C_{\min}$ for all $l \in [L]$.
    \item Policy $k$ becomes optimal: $\bar{\lambda}_k > \bar{\lambda}_1$.
\end{itemize}

For all other policies $k' \notin \{1, k\}$, we set $\lambda_{k',l} = \mu_{k',l}$ to minimize distance contribution:
$$
\sum_{k'\notin\{1,k\}}\sum_{l\in[L]}w_{k',l}\mathrm{d}(\mu_{k',l},\lambda_{k',l})=0.
$$

This corresponds to the optimization problem (\ref{eq:f_mu_opt}). Selecting the most efficient alternative policy yields:
$$
f^\mathrm{opt}_{\boldsymbol{\mu}}(\mathbf{w})=\min\left\{f^{\mathrm{opt}(2)}_{\boldsymbol{\mu}}(\mathbf{w}),\ldots,f^{\mathrm{opt}(K)}_{\boldsymbol{\mu}}(\mathbf{w})\right\}.
$$

\emph{Option 2: Make Current Optimal Policy Infeasible.} Modify policy 1 to violate at least one constraint. The most efficient approach is to select one subpopulation $l \in [L]$ and set $\lambda_{1,l} = C_{\min}$ (just below the feasibility threshold). For all other policies and subpopulations, set $\lambda_{k,l} = \mu_{k,l}$.

This corresponds to the optimization problem (\ref{eq:f_infeas}). Selecting the most efficient subpopulation yields $f^\mathrm{feas}_{\boldsymbol{\mu}}(\mathbf{w})$.

Taking the option that incurs the smaller distance:
$$
f^*_{\boldsymbol{\mu}}(\mathbf{w})=\min\left\{f^\mathrm{opt}_{\boldsymbol{\mu}}(\mathbf{w}),f^\mathrm{feas}_{\boldsymbol{\mu}}(\mathbf{w})\right\}.
$$

\textbf{Conclusion.} The three optimization problems (\ref{eq:f_star_k_0}), (\ref{eq:f_mu_opt}), and (\ref{eq:f_infeas}) fully characterize $f^*_{\boldsymbol{\mu}}(\mathbf{w})$ by identifying the closest counter instance across all possible ways to change the optimal policy. This completes the proof.

%% file: appendices/B_counterset/2_extension_proof.tex
\subsection{Proof of Lemma \ref{lemma: counter_set_calculation_ext1}}
\label{appendix:lemma_cs_ext1_proof}

The proof of this lemma follows the same structure as that of Lemma \ref{lemma: counter_set_calculation} in Appendix \ref{appendix:thm_cs_proof}. The key distinction lies in how feasibility is defined and enforced: under the generalized constraint set $\mathcal{M}$, a policy $k$ is feasible if and only if its subpopulational performance vector satisfies $\mathbf{\mu}_k \in \mathcal{M}$, rather than the coordinate-wise linear constraints $\mu_{k,l} \ge C_{\min}$ for all $l \in [L]$ in the original SBFC setting.

This change in feasibility definition affects the counter instance construction in two ways:
\begin{enumerate}
    \item \textbf{Making a policy feasible}: Instead of setting individual coordinates to $C_{\min}$ to satisfy linear constraints, we must find the closest point in $\mathcal{M}$ to the original performance vector $\mathbf{\mu}_k$ under the weighted KL-divergence metric.
    \item \textbf{Making a policy infeasible}: Instead of setting a single coordinate $\lambda_{1,l} < C_{\min}$, we must move the performance vector $\mathbf{\lambda}_1$ outside the constraint set $\mathcal{M}$.
\end{enumerate}

We now provide the detailed arguments for both cases.

\textbf{Case 1: No Feasible Policies ($\mathcal{C}(\boldsymbol{\mu})=\emptyset$).}

The counter set definition and construction strategy are identical to Appendix \ref{appendix:thm_cs_proof}. The only difference is that instead of explicitly setting $\lambda_{k,l} = C_{\min}$ for violating coordinates, we solve the constrained optimization problem in equation (\ref{eq:f_star_k_0_ext1}) to find the closest point $\mathbf{\lambda}_k \in \mathcal{M}$. Selecting the most efficient policy yields $f_{\boldsymbol{\mu},\mathcal{M}}^*(\mathbf{w})=\min\{f_{\boldsymbol{\mu},\mathcal{M}}^{1}(\mathbf{w}),\ldots,f_{\boldsymbol{\mu},\mathcal{M}}^{K}(\mathbf{w})\}$.

\textbf{Case 2: At Least One Feasible Policy ($\mathcal{C}(\boldsymbol{\mu})\neq\emptyset$).}

Following the two-option structure in Appendix \ref{appendix:thm_cs_proof}: \emph{Option 1} modifies the feasibility constraint to $\mathbf{\lambda}_k \in \mathcal{M}$ in equation (\ref{eq:f_mu_opt_ext1}), yielding $f^\mathrm{opt}_{\boldsymbol{\mu},\mathcal{M}}(\mathbf{w})$. \emph{Option 2} requires finding the closest point $\mathbf{\lambda}_1 \notin \mathcal{M}$ in equation (\ref{eq:f_infeas_ext1}), yielding $f^\mathrm{feas}_{\boldsymbol{\mu},\mathcal{M}}(\mathbf{w})$. Taking the minimum gives $f^*_{\boldsymbol{\mu},\mathcal{M}}(\mathbf{w})=\min\{f^\mathrm{opt}_{\boldsymbol{\mu},\mathcal{M}}(\mathbf{w}),f^\mathrm{feas}_{\boldsymbol{\mu},\mathcal{M}}(\mathbf{w})\}$.

The three optimization problems (\ref{eq:f_star_k_0_ext1}), (\ref{eq:f_mu_opt_ext1}), and (\ref{eq:f_infeas_ext1}) thus fully characterize $f^*_{\boldsymbol{\mu},\mathcal{M}}(\mathbf{w})$ for the $\mathcal{M}$-SBFC setting, completing the proof.

%% file: appendices/B_counterset/3_extension2_proof.tex
\subsection{Proof of Lemma \ref{lemma: counter_set_calculation_ext2}}
\label{appendix:thm_cs_ext2_proof}

The proof parallels that of Lemma \ref{lemma: counter_set_calculation} (Appendix \ref{appendix:thm_cs_proof}) with one key difference: optimality is determined by penalized performance $\mu_k^\gamma$ rather than hard feasibility constraints. Since all policies contribute penalized performance regardless of constraint violations, we only need to construct counter instances where an alternative policy becomes optimal (the analogue of \emph{Option 1} in the hard constraint case).

Assume without loss of generality that $k^{\gamma,*}(\boldsymbol{\mu})=1$. To construct a counter instance where policy $k \ge 2$ becomes optimal, we set $\lambda_{k',l} = \mu_{k',l}$ for all $k' \notin \{1, k\}$ (minimizing their distance contribution) and solve:
\begin{align*}
f^{\mathrm{opt}(k)}_{\boldsymbol{\mu},\boldsymbol{\gamma}}(\mathbf{w})
&=\inf_{\substack{\lambda_{i,l}\in\mathbb{R}\\ i\in\{1,k\}, l\in[L]}}
\sum_{l\in[L]}\Big(w_{1,l}\mathrm{d}(\mu_{1,l},\lambda_{1,l})\\
&\qquad\qquad+w_{k,l}\mathrm{d}(\mu_{k,l},\lambda_{k,l})\Big)\\
&\quad\textit{s.t. }~\lambda_k^\gamma\geq\lambda_1^\gamma,
\end{align*}
where $\lambda_k^\gamma=\sum_{l=1}^Lq_l(\lambda_{k,l}+\gamma_{l}\min(\lambda_{k,l}-C_\text{min},0))$ and similarly for $\lambda_1^\gamma$.

Taking the minimum over all alternative policies yields equation (\ref{eq:f_mu_opt_ext2}):
$$
f_{\boldsymbol{\mu},\boldsymbol{\gamma}}^{*}(\mathbf{w})=\min_{k=2,\ldots,K}f^{\mathrm{opt}(k)}_{\boldsymbol{\mu},\boldsymbol{\gamma}}(\mathbf{w}).
$$

%% file: appendices/C_lemmas/1_weight_infeas.tex
\subsection{Proof of Lemma \ref{lemma:weight_infeas}}
\label{appendix:lemma_1_proof}

By Lemma \ref{lemma: counter_set_calculation}, if $C(\boldsymbol{\mu})=\emptyset$, then for any $k$ and $w_{k,l}~(l=1,\cdots,L)$, $f_{\boldsymbol{\mu}}^k(\mathbf{w})$ is calculated as in \eqref{eq:f_star_k_0}:
\begin{equation}
\begin{aligned}
f_{\boldsymbol{\mu}}^k(\mathbf{w})=\inf_{\lambda_{k,1},\ldots,\lambda_{k,L}\in\mathbb{R}}&~\sum_{l\in[L]}w_{k,l}\mathrm{d}(\mu_{k,l},\lambda_{k,l}),\\
s.t.\quad&\lambda_{k,l}\geq C_{\min},~\forall~ l\in[L].
\end{aligned}
\end{equation}
For this optimization problem, to minimize the objective function, for those $\mu_{k,l'}\geq C_\text{min}$, we can set $\lambda_{k,l'}=\mu_{k,l'}$ so that $\mathrm{d}(\mu_{k,l'},\lambda_{k,l'})=0$, as $\mathrm{d}(\cdot,\cdot)\geq 0$. Since $C(\boldsymbol{\mu})=\emptyset$, there must exist at least one $l''$ such that $\mu_{k,l''}<C_\text{min}$. For such $l''$, then $\mathrm{d}(\mu_{k,l''},\lambda)$ with constraint $\lambda\geq C_\text{min}$ is minimized 
at $\lambda=C_\text{min}$, as by the property of one-parameter exponential family $\mathrm{d}(\mu_{k,l''},\lambda)$ is increasing w.r.t. $\lambda$ on $[C_\text{min},+\infty)$ when $\mu_{k,l''}<C_\text{min}$. Thus, $f_{\boldsymbol{\mu}}^k(\mathbf{w})$ can be further simplified as:
\begin{equation}
f_{\boldsymbol{\mu}}^k(\mathbf{w})=\sum_{l:\mu_{k,l}<C_\text{min}}w_{k,l}\mathrm{d}(\mu_{k,l},C_\text{min}).    
\end{equation}
Now, if we define $\hat{\mathbf{w}}$ as follows:
\begin{equation}
\begin{aligned}
&\hat{\mathbf{w}}_{k,l(k)}=\frac{1}{\dis_k}\left(\sum_{k\in[K]}\frac{1}{\dis_k}\right)^{-1};\\
&\hat{\mathbf{w}}_{k,l}=0,\quad \text{for}~l\in[L],~l\neq l(k), 
\end{aligned}
\end{equation} 
where $\dis_k$ is short for $\dis(\mu_{k,l(k)},C_\mathrm{min})$,$l(k)=\arg\max_{l\in[L],\mu_{k,l}<C_\text{min}}\mathrm{d}(\mu_{k,l},C_\text{min})=\arg\min_{l\in[L]}\mu_{k,l}$. Note that $\sum_{k,l}\hat{\mathbf{w}}_{k,l}=1$, then we have:
\begin{equation}
\begin{aligned}
 f^*_{\boldsymbol{\mu}}(\mathbf{w})
    =& \min\left\{f_{\boldsymbol{\mu}}^{1}(\mathbf{w}),\ldots,f_{\boldsymbol{\mu}}^{K}(\mathbf{w})\right\}\\
    \leq& \sum_{k=1}^K\left(\hat{\mathbf{w}}\right)_{k,l(k)}f_{\boldsymbol{\mu}}^{k}(\mathbf{w})\\
    \leq& \left(\sum_{k\in[K]}\frac{1}{\dis_k}\right)^{-1}\sum_{k=1}^K\sum_{l:\mu_{k,l}<C_\text{min}}w_{k,l}\\
     \leq& \left(\sum_{k\in[K]}\frac{1}{\dis_k}\right)^{-1}.
\end{aligned}
\end{equation}
Since $l(k)$ are all singleton, the equality is attained if and only if $\mathbf{w}=\hat{\mathbf{w}}$. Thus, 
\begin{equation}
\hat{\mathbf{w}}=\arg\max_{\mathbf{w}\in\Sigma_{K\times L}}f_{{\boldsymbol{\mu}}}^*(\mathbf{w})=\mathbf{w}^*.
\end{equation}

%% file: appendices/C_lemmas/2_subgradient.tex
\subsection{Proof of Lemma \ref{lemma:subgradient}}
\label{appendix:lemma_2_proof}

We verify that the vector $\mathbf{c}(\mathbf{w})$ constructed in Lemma \ref{lemma:subgradient} satisfies the subgradient inequality for the concave function $f^*_{\boldsymbol{\mu}}(\mathbf{w})$ at $\mathbf{w}$. Recall $f^*_{\boldsymbol{\mu}}(\mathbf{w}) = \inf_{\boldsymbol{\lambda} \in \mathcal{A}(\boldsymbol{\mu})} \mathbf{c}(\boldsymbol{\lambda})^T \mathbf{w}$, where $c_{k,l}(\boldsymbol{\lambda}) = \mathrm{d}(\mu_{k,l}, \lambda_{k,l})$. A vector $\mathbf{g}$ is a subgradient of the concave function $f$ at $\mathbf{w}$ if $f(\mathbf{w}_1) \le f(\mathbf{w}) + \mathbf{g}^T (\mathbf{w}_1 - \mathbf{w})$ for all $\mathbf{w}_1$.

Assume $\mathcal{C}(\boldsymbol{\mu}) \neq \emptyset$ and $k^*(\boldsymbol{\mu})=1$. Lemma \ref{lemma: counter_set_calculation} states $f^*_{\boldsymbol{\mu}}(\mathbf{w})=\min\left\{f^\mathrm{opt}_{\boldsymbol{\mu}}(\mathbf{w}),f^\mathrm{feas}_{\boldsymbol{\mu}}(\mathbf{w})\right\}$.

\textbf{Case 1: $f^\mathrm{opt}_{\boldsymbol{\mu}}(\mathbf{w}) < f^\mathrm{feas}_{\boldsymbol{\mu}}(\mathbf{w})$}
Here, $f^*_{\boldsymbol{\mu}}(\mathbf{w}) = f^{\mathrm{opt}(k)}_{\boldsymbol{\mu}}(\mathbf{w})$ for $k = \arg\min_i \{f^{\mathrm{opt}(i)}_{\boldsymbol{\mu}}(\mathbf{w})\}_{i \neq 1}$. Let $\boldsymbol{\lambda}^* \in \mathcal{A}(\boldsymbol{\mu})$ be the minimizer attaining this value in \eqref{eq:f_mu_opt}. The vector $\mathbf{c}(\mathbf{w})$ defined in the lemma has components $c_{i,l} = \mathrm{d}(\mu_{i,l}, \lambda^*_{i,l})$ for $i \in \{1, k\}$ and $c_{i,l}=0$ otherwise. Thus, $\mathbf{c}(\mathbf{w}) = \mathbf{c}(\boldsymbol{\lambda}^*)$ restricted to relevant indices. By construction, $f^*_{\boldsymbol{\mu}}(\mathbf{w}) = \mathbf{c}(\mathbf{w})^T \mathbf{w}$.
For any $\mathbf{w}_1$, by definition of $f^*$ as an infimum:
$$ f^*_{\boldsymbol{\mu}}(\mathbf{w}_1) \le \mathbf{c}(\boldsymbol{\lambda}^*)^T \mathbf{w}_1 = \mathbf{c}(\mathbf{w})^T \mathbf{w}_1. $$
Substituting $f^*_{\boldsymbol{\mu}}(\mathbf{w}) = \mathbf{c}(\mathbf{w})^T \mathbf{w}$, we get $f^*_{\boldsymbol{\mu}}(\mathbf{w}_1) \le f^*_{\boldsymbol{\mu}}(\mathbf{w}) + \mathbf{c}(\mathbf{w})^T (\mathbf{w}_1 - \mathbf{w})$.

\textbf{Case 2: $f^\mathrm{opt}_{\boldsymbol{\mu}}(\mathbf{w}) \ge f^\mathrm{feas}_{\boldsymbol{\mu}}(\mathbf{w})$}
Here, $f^*_{\boldsymbol{\mu}}(\mathbf{w}) = f^\mathrm{feas}_{\boldsymbol{\mu}}(\mathbf{w}) = w_{1,l^*} \mathrm{d}(\mu_{1,l^*}, C_{\min})$ for $l^* = \arg\min_{l} w_{1,l} \mathrm{d}(\mu_{1,l}, C_{\min})$. The vector $\mathbf{c}(\mathbf{w})$ has $c_{1,l^*} = \mathrm{d}(\mu_{1,l^*}, C_{\min})$ and zeros elsewhere. Let $\boldsymbol{\lambda}^*$ be an instance in $\mathcal{A}(\boldsymbol{\mu})$ such that $\lambda^*_{1,l^*} = C_{\min}$ (making policy 1 infeasible) and other components match $\boldsymbol{\mu}$ or are chosen optimally; this corresponds to $\mathbf{c}(\mathbf{w}) = \mathbf{c}(\boldsymbol{\lambda}^*)$ restricted to relevant indices. By construction, $f^*_{\boldsymbol{\mu}}(\mathbf{w}) = \mathbf{c}(\mathbf{w})^T \mathbf{w}$.
For any $\mathbf{w}_1$, by definition of $f^*$ as an infimum:
$$ f^*_{\boldsymbol{\mu}}(\mathbf{w}_1) \le \mathbf{c}(\boldsymbol{\lambda}^*)^T \mathbf{w}_1 = \mathbf{c}(\mathbf{w})^T \mathbf{w}_1. $$
Substituting $f^*_{\boldsymbol{\mu}}(\mathbf{w}) = \mathbf{c}(\mathbf{w})^T \mathbf{w}$, we get $f^*_{\boldsymbol{\mu}}(\mathbf{w}_1) \le f^*_{\boldsymbol{\mu}}(\mathbf{w}) + \mathbf{c}(\mathbf{w})^T (\mathbf{w}_1 - \mathbf{w})$.

In both cases, $\mathbf{c}(\mathbf{w})$ satisfies the subgradient inequality. The argument is analogous if $\mathcal{C}(\boldsymbol{\mu}) = \emptyset$. Thus, $\mathbf{c}(\mathbf{w})$ defined in Lemma \ref{lemma:subgradient} is a valid subgradient.

%% file: appendices/C_lemmas/3_optimization_convergence.tex
\subsection{Proof of Lemma \ref{lemma:optimization_convergence}}
\label{appendix:lemma_3_proof}

For each fixed $\boldsymbol{\mu}$, $f^*_{\boldsymbol{\mu}}(\mathbf{w})$ is the infimum of a set of linear functions of $\mathbf{w}$, so it is concave. Then this lemma is just the convergence result of projected subgradient descent method with diminishing step size, see for example \cite{boyd2003subgradient}.

%% file: appendices/D_convergence/1_theorem_proof.tex
\subsection{Proof of Theorem \ref{thm:opt_convergnce}}
\label{appendix:thm_conv_proof}

Theorem \ref{thm:opt_convergnce} states that the T-a-S-CS algorithm (Algorithm \ref{alg:TaSCS}) is $\delta$-PAC and achieves the asymptotically optimal sample complexity $T^*(\boldsymbol{\mu})$ as defined in Theorem \ref{thm:lower_bound}. The proof consists of two main parts: establishing the $\delta$-PAC guarantee and proving the asymptotic optimality.

We begin with the $\delta$-PAC guarantee. The $\delta$-PAC property ensures that the algorithm returns a correct answer with probability at least $1-\delta$. This guarantee stems from the choice of the stopping rule and the threshold function $\beta(t,\delta)$.

The stopping statistic used in Algorithm \ref{alg:TaSCS} is $Z(t)$ defined in Equation \eqref{eq:stopping_statistics}:
$$
Z(t)=\inf_{{\boldsymbol{\lambda}}\in \mathcal{A}(\hat{\boldsymbol{\mu}}(t))}\sum_{k\in[K]}\sum_{l\in[L]}N_{k,l}(t)d\left(\hat\mu_{k,l}(t),\lambda_{k,l}\right).
$$
This statistic is an instance of the Generalized Likelihood Ratio (GLR) statistic $\hat{\Lambda}_t$ discussed in \cite{kaufmann2021mixture}, specifically for testing the composite hypothesis $\mathcal{H}_0: \boldsymbol{\mu} \in \mathcal{A}(\hat{\boldsymbol{\mu}}(t))$ against $\mathcal{H}_1: \boldsymbol{\mu} \notin \mathcal{A}(\hat{\boldsymbol{\mu}}(t))$ (i.e., $\boldsymbol{\mu}$ belongs to the partition element containing $\hat{\boldsymbol{\mu}}(t)$).

The algorithm stops at time $\tau_\delta = \inf_{t\in \mathbb{N}} \{Z(t)>\beta(t,\delta)\}$. To ensure the $\delta$-PAC property, proposition 21 of \cite{kaufmann2021mixture} provides a general threshold:
$$
\beta(t,\delta) = \hat{c}_t(\delta) = 3\sum_{k,l}\ln(1+\ln N_{k,l}(t)) + KL \cdot \mathcal{C}_{\text{exp}}\left(\frac{\ln(1/\delta)}{KL}\right).
$$
The practical threshold $\beta(t,\delta) = \ln((1+\ln t)/\delta)$ used in Algorithm \ref{alg:TaSCS} is motivated by these results and practical considerations. 

To prove $$\lim_{\delta\to 0} \frac{\mathbb{E}_{\boldsymbol{\mu}}[\tau_\delta]}{\ln(1/\delta)} = T^*(\boldsymbol{\mu}),$$
we leverage the framework established by \cite{degenne2019pure}, which analyzes the asymptotic optimality of Track-and-Stop strategies, including cases where the optimal weights $\mathbf{w}^*(\boldsymbol{\mu})$ might not be unique. The proof relies on the convergence of key quantities: convergence of estimates $\hat{\boldsymbol{\mu}}(t) \to \boldsymbol{\mu}$, properties of the optimal weight map $\boldsymbol{\mu} \mapsto \mathbf{w}^*(\boldsymbol{\mu})$, convergence of empirical weights $\hat{\mathbf{w}}_t$ to $\mathbf{w}^*(\boldsymbol{\mu})$, and finally the application of theorem 7 in \cite{degenne2019pure}.

We first show the convergence of estimates $\hat{\boldsymbol{\mu}}(t) \to \boldsymbol{\mu}$. Note that the C-Tracking rule \eqref{eq:c_tracking} involves projecting the target weights $\mathbf{w}_t$ onto $$\Sigma_{K\times L}^{\varepsilon}=\{(w_1,\ldots,w_{KL})\in[\varepsilon,1]^{K\times L}\mid w_1+\cdots+w_{KL}=1\},$$ with $\varepsilon_t=\left(K^2L^2+t\right)^{-1/2}/2$, ensuring a minimum sampling proportion $\varepsilon_t = O(1/\sqrt{t})$ for all $(k,l)$ pairs. Lemma 7 of \cite{garivier2016optimal} provides the lower bound on number of samples $$N_{k,l}(t) \ge \sqrt{t+K^2L^2} - 2KL$$ for this C-tracking rule. Since $N_{k,l}(t) \to \infty$ as $t \to \infty$ for all $(k,l)$, the Strong Law of Large Numbers ensures that the empirical means converge almost surely to the true means: $\hat{\boldsymbol{\mu}}(t) \to \boldsymbol{\mu}$ a.s.

Next, we state some properties of $\boldsymbol{\mu} \mapsto \mathbf{w}^*(\boldsymbol{\mu})$ that will be useful in later steps. The first property is the \textit{convexity of $\mathbf{w}^*(\boldsymbol{\mu})$}. To show this property, note that the function $f^*_{\boldsymbol{\mu}}(\mathbf{w})$ defined in \eqref{eq:internal_opt} is the infimum of functions linear in $\mathbf{w}$ (specifically, $\sum w_{k,l} d(\mu_{k,l}, \lambda_{k,l})$ for fixed $\boldsymbol{\lambda}$). Therefore, $f^*_{\boldsymbol{\mu}}(\mathbf{w})$ is concave in $\mathbf{w}$. The set of maximizers $\mathbf{w}^*(\boldsymbol{\mu}) = \arg\max_{\mathbf{w}\in\Sigma_{K\times L}} f^*_{\boldsymbol{\mu}}(\mathbf{w})$ is the set of maximizers of a concave function over a convex set ($\Sigma_{K\times L}$), and is therefore itself convex. 
    
The second property is the \textit{hemicontinuity of $\boldsymbol{\mu} \mapsto \mathbf{w}^*(\boldsymbol{\mu})$}. Note that the convergence proof of the T-a-S-CS algorithm relies on the upper hemicontinuity of the set-valued map $\boldsymbol{\mu} \mapsto \mathbf{w}^*(\boldsymbol{\mu})$. Recall that $\mathbf{w}^*(\boldsymbol{\mu})$ denotes the set of optimal solutions (optimal sampling proportions) for the external optimization problem \eqref{eq:external_opt}:
$$
\mathbf{w}^*({\boldsymbol{\mu}})=\arg\max_{\mathbf{w}\in\Sigma_{K\times L}}f_{{\boldsymbol{\mu}}}^*(\mathbf{w}).
$$
We establish the hemicontinuity property by applying Berge's Maximum Theorem, specifically Theorem 22 from \cite{degenne2019pure}.

\begin{theorem}[Berge's Maximum Theorem, adapted from Thm. 22 \cite{degenne2019pure}]
\label{thm:berge_adapted_outer}
Let $X$ and $Y$ be Hausdorff topological spaces. Assume that
\begin{enumerate}
    \item The map $\Phi: X \to \mathbb{K}(Y)$ mapping $\boldsymbol{\mu}$ to the constraint set for $\mathbf{w}$ is both lower and upper hemicontinuous, and compact-valued. $\mathbb{K}$ denotes non-empty compact subsets.
    \item The objective function $u: X \times Y \to \mathbb{R}$, here $u(\boldsymbol{\mu}, \mathbf{w}) = f^*_{\boldsymbol{\mu}}(\mathbf{w})$, is continuous.
\end{enumerate}
Then the value function $v(\boldsymbol{\mu}) = \max_{\mathbf{w} \in \Phi(\boldsymbol{\mu})} u(\boldsymbol{\mu}, \mathbf{w})$  is continuous, and the solution map $\Phi^*(\boldsymbol{\mu}) = \arg\max_{\mathbf{w} \in \Phi(\boldsymbol{\mu})} u(\boldsymbol{\mu}, \mathbf{w})$  is upper hemicontinuous and compact-valued.
\end{theorem}
In our context, we have $Y=\Sigma_{K \times L}$, $v(\boldsymbol{\mu}) = T^*(\boldsymbol{\mu})^{-1}$ and $\Phi^*(\boldsymbol{\mu}) = \mathbf{w}^*(\boldsymbol{\mu})$. We now verify the conditions of Theorem \ref{thm:berge_adapted_outer}. First, the parameter space $\mathcal{S} \subseteq \mathbb{R}^{KL}$ and the choice space $\Sigma_{K \times L} \subset \mathbb{R}^{KL}$ are standard Hausdorff topological spaces. Second, the constraint map for the maximization is $\Phi(\boldsymbol{\mu}) = \Sigma_{K \times L}$. This map is constant. The set $\Sigma_{K \times L}$ is a non-empty, closed, and bounded subset of $\mathbb{R}^{KL}$, hence it is compact. Thus, condition (a) of Theorem \ref{thm:berge_adapted_outer} is satisfied.

Next, we require the objective function $f^*_{\boldsymbol{\mu}}(\mathbf{w})$ to be jointly continuous in $(\boldsymbol{\mu}, \mathbf{w})$.
The continuity with respect to $\mathbf{w}$ follows from the definition of $f^*_{\boldsymbol{\mu}}(\mathbf{w})$ as the infimum of functions $L_{\boldsymbol{\lambda}}(\mathbf{w}) = \sum_{k,l} w_{k,l} d(\mu_{k,l}, \lambda_{k,l})$, which are linear in $\mathbf{w}$ for fixed $\boldsymbol{\mu}$ and $\boldsymbol{\lambda}$ \cite{degenne2019pure}. The infimum of linear functions is concave \cite{degenne2019pure}. Concave functions are continuous on the relative interior of their domain, which includes $\Sigma_{K \times L}$.
The continuity of $f^*_{\boldsymbol{\mu}}(\mathbf{w})$ with respect to $\boldsymbol{\mu}$ is more involved due to the dependence of the infimum's domain $\mathcal{A}(\boldsymbol{\mu})$ on $\boldsymbol{\mu}$. We establish this continuity in Lemma \ref{lemma:continuity_f_mu} below. The proof of Lemma \ref{lemma:continuity_f_mu} is in Section 
\ref{subsec: proof_continuity_lemma}.

\begin{lemma}
\label{lemma:continuity_f_mu}
For any fixed $\mathbf{w} \in \Sigma_{K \times L}$, the function $\boldsymbol{\mu} \mapsto f^*_{\boldsymbol{\mu}}(\mathbf{w})$ is continuous for all $\boldsymbol{\mu} \in \mathcal{S}$.
\end{lemma}

Based on the result of Lemma \ref{lemma:continuity_f_mu}, $f^*_{\boldsymbol{\mu}}(\mathbf{w})$ is jointly continuous in $(\boldsymbol{\mu}, \mathbf{w})$. Therefore, all conditions of Theorem \ref{thm:berge_adapted_outer} are met. We conclude that the optimal weight map $\boldsymbol{\mu} \mapsto \mathbf{w}^*(\boldsymbol{\mu})$ is upper hemicontinuous and has non-empty, compact values for all $\boldsymbol{\mu} \in \mathcal{S}$. Up to this point, we have finished proving the two properties of $\boldsymbol{\mu} \mapsto \mathbf{w}^*(\boldsymbol{\mu})$.

We now show the convergence of empirical weights $\hat{\mathbf{w}}_t$ to $\mathbf{w}^*(\boldsymbol{\mu})$. We need to show that the empirical sampling proportions $\hat{\mathbf{w}}_t = \mathbf{N}(t) / t$ converge to the set of optimal proportions $\mathbf{w}^*(\boldsymbol{\mu})$. This involves two steps:

The first step is the convergence of average tracked weights.  Let $\mathbf{w}_s \in \mathbf{w}^*(\hat{\boldsymbol{\mu}}(s))$ be the weight vector computed at step $s$. As proved in previous steps 1 and 2, $\hat{\boldsymbol{\mu}}(t) \to \boldsymbol{\mu}$ a.s. and the map $\boldsymbol{\mu} \mapsto \mathbf{w}^*(\boldsymbol{\mu})$ is upper hemicontinuous and convex-valued. Therefore, Lemma 6 of \cite{degenne2019pure} applies, and we have
$$ \lim_{t\to\infty} \inf_{\mathbf{w}\in \mathbf{w}^*({\boldsymbol{\mu}})}\left\Vert\frac{1}{t}\sum_{s=0}^{t-1}\mathbf{w}_s-\mathbf{w}\right\Vert_\infty = 0 \quad \text{a.s.}. $$ Note: We use $\mathbf{w}_s$ instead of the projected $\mathbf{w}_s^{\varepsilon_s}$ here; the effect of the projection vanishes asymptotically as $\varepsilon_s \to 0$.

The second step is to show the convergence of empirical weights. The C-Tracking rule \eqref{eq:c_tracking} ensures that the actual counts $N_{k,l}(t)$ closely follow the cumulative sum of the (projected) target weights. Lemma 7 of \cite{garivier2016optimal}, adapted to $KL$ arms, states:
$$ \max_{k,l}\left\vert N_{k,l}(t)-\sum_{s=0}^{t-1}w_{k,l}^{\varepsilon_s}(\hat{\boldsymbol{\mu}}(s))\right\vert\leq KL(1+\sqrt{t}) $$
Dividing by $t$, we get:
        $$ \max_{k,l}\left\vert \hat{w}_{k,l}(t)-\frac{1}{t}\sum_{s=0}^{t-1}w_{k,l}^{\varepsilon_s}(\hat{\boldsymbol{\mu}}(s))\right\vert \leq \frac{KL(1+\sqrt{t})}{t} \xrightarrow{t\to\infty} 0 $$
Since $\varepsilon_s \to 0$, the difference between $w_{k,l}^{\varepsilon_s}(\hat{\boldsymbol{\mu}}(s))$ and $w_{k,l}(\hat{\boldsymbol{\mu}}(s))$ also vanishes. Combining this with the convergence of the average tracked weights from the previous step, we conclude that the empirical weights converge to the optimal set:
    $$ \lim_{t\to\infty} \inf_{\mathbf{w}\in \mathbf{w}^*({\boldsymbol{\mu}})}\big\Vert\hat{\mathbf{w}}_t-\mathbf{w}\big\Vert_\infty = 0 \quad \text{a.s.} $$

The final step uses Theorem~7 of \cite{degenne2019pure}, which states that \emph{Track-and-Stop with C-Tracking} is $\delta$-correct and asymptotically optimal for every instance $\boldsymbol{\mu}$ such that the oracle answer set $i^F(\boldsymbol{\mu})$ is a singleton (in particular, this holds for single-answer pure exploration problems). 
In our proof, we invoke Theorem~7 through its underlying proof ingredients:
(i) a tracking property ensuring that the empirical sampling proportions converge to the oracle weight set, namely
\[
\lim_{t\to\infty}\inf_{\mathbf{w}\in\mathbf{w}^*(\boldsymbol{\mu})}
\bigl\|\hat{\mathbf{w}}_t-\mathbf{w}\bigr\|_\infty = 0
\qquad \text{a.s.},
\]
which is guaranteed here by C-Tracking together with the almost sure convergence $\hat{\boldsymbol{\mu}}(t)\to\boldsymbol{\mu}$ and the regularity (upper hemicontinuity/convexity) of the set-valued map $\boldsymbol{\mu}\mapsto \mathbf{w}^*(\boldsymbol{\mu})$;
and (ii) a GLR/KL-based stopping rule that stops when a statistic of the form $Z(t)$ exceeds a threshold $\beta(t,\delta)$ whose leading dependence on $\delta$ is of order $\ln(1/\delta)$ as $\delta\to 0$.
These two ingredients allow the Track-and-Stop analysis of \cite{degenne2019pure} to upper bound the stopping time and conclude the asymptotic optimality:
\[
\lim_{\delta\to 0}\frac{\mathbb{E}_{\boldsymbol{\mu}}[\tau_\delta]}{\ln(1/\delta)}=T^*({\boldsymbol{\mu}}).
\]

Combining the $\delta$-PAC guarantee (assuming an appropriate threshold) and the asymptotic optimality concludes the proof of Theorem \ref{thm:opt_convergnce}.

%% file: appendices/D_convergence/2_lemma_proof.tex
\subsection{Proof of Lemma \ref{lemma:continuity_f_mu}}
\label{subsec: proof_continuity_lemma}

We use a theorem from \cite{feinberg2014berge's}, which provides conditions for the continuity of an infimum value function when the constraint set depends on the parameter.

\begin{theorem}[\cite{feinberg2014berge's}]
\label{thm:berge_adapted_inner}
Let $X$ and $Y$ be Hausdorff topological spaces, with $X$ compactly generated. Assume that
\begin{enumerate}
    \item The map $\Phi: X \to \mathbb{S}(Y)$ is lower hemicontinuous.
    \item The objective function $u: X \times Y \to \mathbb{R}$ is K-inf-compact and upper semi-continuous on the graph $Gr_{X}(\Phi) = \{ (\boldsymbol{\mu}, \boldsymbol{\lambda}) \mid \boldsymbol{\lambda} \in \Phi(\boldsymbol{\mu})) \}$.
\end{enumerate}
Then the value function $v(\boldsymbol{\mu}) = \inf_{\boldsymbol{\lambda} \in \Phi(\boldsymbol{\mu})} u(\boldsymbol{\mu}, \boldsymbol{\lambda})$ is continuous.
\end{theorem}

In our context, we can take $\Phi(\boldsymbol{\mu}) = \mathcal{A}(\boldsymbol{\mu})$, which maps $\boldsymbol{\mu}$ to its counter set,  $u(\boldsymbol{\mu}, \boldsymbol{\lambda}) = L_{\boldsymbol{\lambda}}(\mathbf{w};\boldsymbol{\mu}) = \sum_{k,l} w_{k,l} d(\mu_{k,l}, \lambda_{k,l})$ for fixed $\mathbf{w}$, and  $v(\boldsymbol{\mu}) = f^*_{\boldsymbol{\mu}}(\mathbf{w})$. We now verify the conditions of Theorem \ref{thm:berge_adapted_inner}.

First, the parameter and choice spaces are compactly generated as it is a metric space, and therefore are Hausdorff.

Second, we consider the objective function $u(\boldsymbol{\mu}, \boldsymbol{\lambda}) = \sum_{k,l} w_{k,l} d(\mu_{k,l}, \lambda_{k,l})$. The KL divergence $d(x,y)$ is continuous in both arguments within its domain. The K-inf-compactness condition requires that for any compact set $C \subset \mathcal{M}$, the level sets $\{ (\boldsymbol{\mu}, \boldsymbol{\lambda}) \in C \times \mathcal{S} \mid \boldsymbol{\lambda} \in \mathcal{A}(\boldsymbol{\mu}), u(\boldsymbol{\mu}, \boldsymbol{\lambda}) \le y \}$ are compact. As argued in \cite{degenne2019pure}, the KL divergence grows sufficiently fast as $\boldsymbol{\lambda}$ approaches boundaries or infinity, ensuring these level sets are bounded, and thus compact (since they are also closed by continuity).

Third, we must show that the constraint map $\Phi(\boldsymbol{\mu}) = \mathcal{A}(\boldsymbol{\mu}) = \{ \boldsymbol{\lambda} \in \mathcal{S} \mid k^*(\boldsymbol{\lambda}) \neq k^*(\boldsymbol{\mu}) \}$ is lower hemicontinuous (LHC) for $\boldsymbol{\mu} \in \mathcal{S}$. We use the sequence definition of LHC. Let $\boldsymbol{\mu}_n \to \boldsymbol{\mu}_0$ with $\boldsymbol{\mu}_0 \in \mathcal{S}$, and let $\boldsymbol{\lambda}_0 \in \mathcal{A}(\boldsymbol{\mu}_0)$. This means $k^*(\boldsymbol{\lambda}_0) \neq k^*(\boldsymbol{\mu}_0)$. We need to find a sequence $\boldsymbol{\lambda}_n \in \mathcal{A}(\boldsymbol{\mu}_n)$ such that $\boldsymbol{\lambda}_n \to \boldsymbol{\lambda}_0$.

We use the properties guaranteed by Definition \ref{def:S} for $\boldsymbol{\mu}_0 \in \mathcal{S}$.
\begin{itemize}
    \item Case 1: $\mathcal{C}(\boldsymbol{\mu}_0) = \emptyset$, so $k^*(\boldsymbol{\mu}_0) = 0$. By Definition \ref{def:S}, all policies violate at least one constraint strictly ($\exists l$ s.t. $(\mu_0)_{k,l} < C_{\min}$). Due to continuity, for $n$ large enough, the same strict inequalities hold for $\boldsymbol{\mu}_n$, thus $\mathcal{C}(\boldsymbol{\mu}_n) = \emptyset$ and $k^*(\boldsymbol{\mu}_n) = 0$. In this case, $\mathcal{A}(\boldsymbol{\mu}_n) = \{ \boldsymbol{\lambda} \in \mathcal{S} \mid k^*(\boldsymbol{\lambda}) \neq 0 \}$ and $\mathcal{A}(\boldsymbol{\mu}_0) = \{ \boldsymbol{\lambda} \in \mathcal{S} \mid k^*(\boldsymbol{\lambda}) \neq 0 \}$. Since $\boldsymbol{\lambda}_0 \in \mathcal{A}(\boldsymbol{\mu}_0)$, we have $k^*(\boldsymbol{\lambda}_0) \neq 0$. Thus, $\boldsymbol{\lambda}_0 \in \mathcal{A}(\boldsymbol{\mu}_n)$ for large $n$. We can choose the sequence $\boldsymbol{\lambda}_n = \boldsymbol{\lambda}_0$, which converges to $\boldsymbol{\lambda}_0$. The LHC condition holds.

    \item Case 2: $k^*(\boldsymbol{\mu}_0)$ is unique and strictly feasible/optimal. Let $k_0 = k^*(\boldsymbol{\mu}_0)$. The strict inequalities $(\mu_0)_{k_0, l} > C_{\min}$ and $(\mu_0)_{k_0} > (\mu_0)_k$ for $k \in \mathcal{C}(\boldsymbol{\mu}_0) \setminus \{k_0\}$ hold. By continuity, for $n$ large enough, these same strict inequalities hold for $\boldsymbol{\mu}_n$, and feasibility/infeasibility of other policies also remains the same. Thus, $k^*(\boldsymbol{\mu}_n) = k_0$ for large $n$. In this case, $\mathcal{A}(\boldsymbol{\mu}_n) = \{ \boldsymbol{\lambda} \in \mathcal{S} \mid k^*(\boldsymbol{\lambda}) \neq k_0 \}$ and $\mathcal{A}(\boldsymbol{\mu}_0) = \{ \boldsymbol{\lambda} \in \mathcal{S} \mid k^*(\boldsymbol{\lambda}) \neq k_0 \}$. Since $\boldsymbol{\lambda}_0 \in \mathcal{A}(\boldsymbol{\mu}_0)$, we have $k^*(\boldsymbol{\lambda}_0) \neq k_0$. Thus, $\boldsymbol{\lambda}_0 \in \mathcal{A}(\boldsymbol{\mu}_n)$ for large $n$. We can choose the sequence $\boldsymbol{\lambda}_n = \boldsymbol{\lambda}_0$, which converges to $\boldsymbol{\lambda}_0$. The LHC condition holds.
\end{itemize}
In both cases permitted by Definition \ref{def:S}, the map $\boldsymbol{\mu} \mapsto k^*(\boldsymbol{\mu})$ is locally constant around $\boldsymbol{\mu}_0 \in \mathcal{S}$. Consequently, the map $\boldsymbol{\mu} \mapsto \mathcal{A}(\boldsymbol{\mu})$ is also locally constant around $\boldsymbol{\mu}_0 \in \mathcal{S}$, which implies it is lower hemicontinuous at $\boldsymbol{\mu}_0$.

Since all conditions of Theorem \ref{thm:berge_adapted_inner} are met for $\boldsymbol{\mu} \in \mathcal{S}$, we conclude that $f^*_{\boldsymbol{\mu}}(\mathbf{w})$ is continuous with respect to $\boldsymbol{\mu}$ for any fixed $\mathbf{w}$.

%% file: appendices/D_convergence/3_extension1_cor_proof.tex
\subsection{Proof of Corollary \ref{cor:opt_convergnce_m_sbfc}}
\label{appendix:thm_conv_ext1_proof}

The proof of this corollary follows the same structure as that of Theorem \ref{thm:opt_convergnce} in Appendix \ref{appendix:thm_conv_proof}. The key distinction lies in the definition of the counter set and the internal optimization problem: under the generalized constraint set $\mathcal{M}$, the counter set is $\mathcal{A}(\boldsymbol{\mu}, \mathcal{M})$ and the internal optimization uses $f^*_{\boldsymbol{\mu},\mathcal{M}}(\mathbf{w})$ instead of $f^*_{\boldsymbol{\mu}}(\mathbf{w})$.

This change affects the algorithm and convergence analysis in the following ways:
\begin{enumerate}
    \item {Stopping statistic}: The GLR statistic becomes $$Z(t)=\inf_{{\boldsymbol{\lambda}}\in \mathcal{A}(\hat{\boldsymbol{\mu}}(t), \mathcal{M})}\sum_{k,l}N_{k,l}(t)d(\hat\mu_{k,l}(t),\lambda_{k,l}),$$ where the infimum is taken over the generalized counter set.
    \item {Weight optimization}: The projected subgradient ascent computes $\mathbf{w}_t \in \mathbf{w}^*(\hat{\boldsymbol{\mu}}(t), \mathcal{M})$ using subgradients of $f^*_{\hat{\boldsymbol{\mu}}(t),\mathcal{M}}(\mathbf{w})$, which involves solving the constrained optimization problems (\ref{eq:f_star_k_0_ext1}), (\ref{eq:f_mu_opt_ext1}), and (\ref{eq:f_infeas_ext1}) instead of their linear constraint counterparts.
\end{enumerate}

We now verify that all steps in Appendix \ref{appendix:thm_conv_proof} remain valid under these modifications.

{Step 1 (Convergence of Estimates):} The C-Tracking rule and the Strong Law of Large Numbers apply identically, ensuring $\hat{\boldsymbol{\mu}}(t) \to \boldsymbol{\mu}$ a.s. This step is independent of the constraint form.

{Step 2 (Properties of Optimal Weight Map):} 
\begin{itemize}
    \item \emph{Convexity}: Since $f^*_{\boldsymbol{\mu},\mathcal{M}}(\mathbf{w})$ is the infimum of linear functions in $\mathbf{w}$, it is concave, hence $\mathbf{w}^*(\boldsymbol{\mu}, \mathcal{M})$ is convex.
    
    \item \emph{Hemicontinuity}: To apply Berge's Maximum Theorem (Theorem \ref{thm:berge_adapted_outer}), we verify that $f^*_{\boldsymbol{\mu},\mathcal{M}}(\mathbf{w})$ is jointly continuous in $(\boldsymbol{\mu}, \mathbf{w})$. Continuity in $\mathbf{w}$ follows from concavity. For continuity in $\boldsymbol{\mu}$, we verify that the structural requirements stated in Section \ref{subsec:first_extension} are sufficient. Under these requirements, the optimal policy $k^*(\boldsymbol{\mu}, \mathcal{M})$ is locally constant: for any $\boldsymbol{\mu}_n \to \boldsymbol{\mu}_0 \in \mathcal{S}$, since $\boldsymbol{\mu}_{k^*}$ is in the interior of $\mathcal{M}$ and the optimality gap is strict, continuity ensures that $k^*(\boldsymbol{\mu}_n, \mathcal{M}) = k^*(\boldsymbol{\mu}_0, \mathcal{M})$ for sufficiently large $n$. Therefore, the counter set $\mathcal{A}(\boldsymbol{\mu}, \mathcal{M})$ is locally constant, hence lower hemicontinuous. By Theorem \ref{thm:berge_adapted_inner}, this ensures $f^*_{\boldsymbol{\mu},\mathcal{M}}(\mathbf{w})$ is continuous in $\boldsymbol{\mu}$, and by Theorem \ref{thm:berge_adapted_outer}, the map $\boldsymbol{\mu} \mapsto \mathbf{w}^*(\boldsymbol{\mu}, \mathcal{M})$ is upper hemicontinuous.
\end{itemize}

{Step 3 (Convergence of Empirical Weights):} Since $\hat{\boldsymbol{\mu}}(t) \to \boldsymbol{\mu}$ a.s. and $\boldsymbol{\mu} \mapsto \mathbf{w}^*(\boldsymbol{\mu}, \mathcal{M})$ is upper hemicontinuous and convex-valued, Lemma 6 of \cite{degenne2019pure} applies identically, yielding convergence of both average tracked weights and empirical weights to $\mathbf{w}^*(\boldsymbol{\mu}, \mathcal{M})$. This step does not depend on the specific form of the counter set.

{Step 4 (Application of Asymptotic Optimality Theorem):} The $\mathcal{M}$-SBFC problem is a single-answer pure exploration problem with convex optimal weight set $\mathbf{w}^*(\boldsymbol{\mu}, \mathcal{M})$. The $\mathcal{M}$-T-a-S-CS algorithm uses C-Tracking and a GLR-based stopping rule with threshold $\beta(t,\delta) \sim \ln(1/\delta)$. Therefore, Theorem 7 of \cite{degenne2019pure} applies, yielding:
$$
\lim_{\delta\to 0}\frac{\mathbb{E}_{\boldsymbol{\mu}}[\tau_\delta]}{\ln(1/\delta)}=T^*({\boldsymbol{\mu}}, \mathcal{M}).
$$

Combining the $\delta$-PAC guarantee from the stopping rule and the asymptotic optimality from the tracking rule completes the proof of Corollary \ref{cor:opt_convergnce_m_sbfc}.

%% file: appendices/D_convergence/4_extension2_cor_proof.tex
\subsection{Proof of Corollary \ref{cor:opt_convergnce_ext2}}
\label{appendix:thm_conv_ext2_proof}

The proof of this corollary follows the same structure as that of Theorem \ref{thm:opt_convergnce} in Appendix \ref{appendix:thm_conv_proof}. The key distinction lies in the definition of optimality and the internal optimization problem: under soft fairness constraints with penalty parameters $\boldsymbol{\gamma}$, the optimal policy is determined by maximizing penalized performance rather than satisfying hard feasibility constraints, and the counter set is $\mathcal{A}(\boldsymbol{\mu}, \boldsymbol{\gamma})$ with internal optimization using $f^*_{\boldsymbol{\mu},\boldsymbol{\gamma}}(\mathbf{w})$.

This change affects the algorithm and convergence analysis in the following ways:
\begin{enumerate}
    \item {Stopping statistic}: The GLR statistic becomes $$Z(t)=\inf_{{\boldsymbol{\lambda}}\in \mathcal{A}(\hat{\boldsymbol{\mu}}(t), \boldsymbol{\gamma})}\sum_{k,l}N_{k,l}(t)d(\hat\mu_{k,l}(t),\lambda_{k,l}),$$ where the counter set is defined using penalized performance.
    \item {Weight optimization}: The projected subgradient ascent computes $\mathbf{w}_t \in \mathbf{w}^*(\hat{\boldsymbol{\mu}}(t), \boldsymbol{\gamma})$ using subgradients of $f^*_{\hat{\boldsymbol{\mu}}(t),\boldsymbol{\gamma}}(\mathbf{w})$, which involves solving the penalized optimization problem (\ref{eq:f_mu_opt_ext2}) with constraint $\lambda_k^\gamma \ge \lambda_1^\gamma$.
\end{enumerate}

We now verify that all steps in Appendix \ref{appendix:thm_conv_proof} remain valid under these modifications.

{Step 1 (Convergence of Estimates):} The C-Tracking rule and the Strong Law of Large Numbers apply identically, ensuring $\hat{\boldsymbol{\mu}}(t) \to \boldsymbol{\mu}$ a.s. This step is independent of the optimality criterion.

{Step 2 (Properties of Optimal Weight Map):} 
\begin{itemize}
    \item \emph{Convexity}: Since $f^*_{\boldsymbol{\mu},\boldsymbol{\gamma}}(\mathbf{w})$ is the infimum of linear functions in $\mathbf{w}$, it is concave, hence $\mathbf{w}^*(\boldsymbol{\mu}, \boldsymbol{\gamma})$ is convex.
    
    \item \emph{Hemicontinuity}: To apply Berge's Maximum Theorem (Theorem \ref{thm:berge_adapted_outer}), we verify that $f^*_{\boldsymbol{\mu},\boldsymbol{\gamma}}(\mathbf{w})$ is jointly continuous in $(\boldsymbol{\mu}, \mathbf{w})$. Continuity in $\mathbf{w}$ follows from concavity. For continuity in $\boldsymbol{\mu}$, we verify that the structural requirements stated in Section \ref{subsec:second_extension} are sufficient. Since the penalized performance $\mu_k^\gamma = \sum_{l=1}^Lq_l\left(\mu_{k,l}+\gamma_{l}\min(\mu_{k,l}-C_\text{min},0)\right)$ is continuous in $\boldsymbol{\mu}_k$, the strict gap condition ensures that the optimal policy $k^{\gamma,*}(\boldsymbol{\mu})$ is locally constant: for any $\boldsymbol{\mu}_n \to \boldsymbol{\mu}_0 \in \mathcal{S}$, we have $\mu_{k,n}^\gamma \to \mu_{k,0}^\gamma$ for all $k$, so the strict ranking is preserved for sufficiently large $n$. Therefore, the counter set $\mathcal{A}(\boldsymbol{\mu}, \boldsymbol{\gamma})$ is locally constant, hence lower hemicontinuous. By Theorem \ref{thm:berge_adapted_inner}, this ensures $f^*_{\boldsymbol{\mu},\boldsymbol{\gamma}}(\mathbf{w})$ is continuous in $\boldsymbol{\mu}$, and by Theorem \ref{thm:berge_adapted_outer}, the map $\boldsymbol{\mu} \mapsto \mathbf{w}^*(\boldsymbol{\mu}, \boldsymbol{\gamma})$ is upper hemicontinuous.
\end{itemize}

{Step 3 (Convergence of Empirical Weights):} Since $\hat{\boldsymbol{\mu}}(t) \to \boldsymbol{\mu}$ a.s. and $\boldsymbol{\mu} \mapsto \mathbf{w}^*(\boldsymbol{\mu}, \boldsymbol{\gamma})$ is upper hemicontinuous and convex-valued, Lemma 6 of \cite{degenne2019pure} applies identically, yielding convergence of both average tracked weights and empirical weights to $\mathbf{w}^*(\boldsymbol{\mu}, \boldsymbol{\gamma})$. This step does not depend on the specific form of the counter set or optimality criterion.

{Step 4 (Application of Asymptotic Optimality Theorem):} The $\gamma$-SBFC problem is a single-answer pure exploration problem with convex optimal weight set $\mathbf{w}^*(\boldsymbol{\mu}, \boldsymbol{\gamma})$. The $\gamma$-T-a-S-CS algorithm uses C-Tracking and a GLR-based stopping rule with threshold $\beta(t,\delta) \sim \ln(1/\delta)$. Therefore, Theorem 7 of \cite{degenne2019pure} applies, yielding:
$$
\lim_{\delta\to 0}\frac{\mathbb{E}_{\boldsymbol{\mu}}[\tau_\delta]}{\ln(1/\delta)}=T^*({\boldsymbol{\mu}}, \boldsymbol{\gamma}).
$$

Combining the $\delta$-PAC guarantee from the stopping rule and the asymptotic optimality from the tracking rule completes the proof of Corollary \ref{cor:opt_convergnce_ext2}.

%% file: appendices/E_numerical/1_gaussian_sbfc.tex
\subsection{Optimization Subproblems under Gaussian KL-Divergence}
\label{appendix:gaussian_sbfc}

This appendix provides the detailed formulation of the optimization subproblems for the basic SBFC problem when $\mathcal{P}_{k,l}$ belongs to the Gaussian family with known variance $\sigma^2=1$, as referenced in Section~\ref{subsec:stop_rule}.

\textbf{Internal Optimization.} With the Gaussian KL-divergence $\mathrm{d}(\mu,\lambda)=\frac{1}{2}(\mu-\lambda)^2$, the internal optimization subproblems in (\ref{eq:f_mu_opt}) become:
\begin{equation}
\label{eq:gaussian_internal_opt}
\begin{aligned}
\min_{\substack{\lambda_{i,l}\in\mathbb{R}\\ i\in\{1,k\},~l\in[L]}}~
&\frac{1}{2}\sum_{i\in\{1,k\}}\sum_{l\in[L]}w_{i,l}(\mu_{i,l}-\lambda_{i,l})^2,\\
\text{s.t.}\quad &\lambda_k=\sum_{l\in[L]}q_l\lambda_{k,l}\geq \lambda_1= \sum_{l\in[L]}q_l\lambda_{1,l},\\
&\lambda_{k,l}\geq C_{\min},~\forall~l\in[L],
\end{aligned}
\end{equation}
for each $k\in[K]$, $k\geq 2$. Each of the $K-1$ subproblems is a \emph{convex quadratic program (QP)} with linear constraints, which can be efficiently solved using standard optimization methods such as interior-point methods or the Lagrangian multiplier method (see, e.g., Lemma 5 of \cite{russac2021b}).

\textbf{External Optimization.} The external optimization problem (\ref{eq:external_opt}) is a concave maximization problem with linear simplex constraints on $\mathbf{w}$. As analyzed in Section \ref{subsec: opt_weights}, we employ the projected subgradient ascent method where each subgradient evaluation requires solving the internal optimization problem (\ref{eq:gaussian_internal_opt}).

\textbf{Other Exponential Families.} The algorithm can be extended to other exponential families including binomial distributions (with fixed parameter $n$), Poisson distributions, and gamma distributions (with fixed shape parameter). For these families, the internal optimization problems remain convex, though no longer QPs, requiring interior-point methods for general convex programs. The computational cost increases but convergence is still guaranteed. Appendix~\ref{appendix:bernoulli_subproblem} provides the detailed derivation for the Bernoulli case.

%% file: appendices/E_numerical/2_bernoulli_subproblem.tex
\subsection{Derivation of Optimization Subproblems under Bernoulli KL-Divergence}
\label{appendix:bernoulli_subproblem}

This appendix provides the detailed derivation of the optimization subproblems when using the Bernoulli KL-divergence, as referenced in Section \ref{subsubsec:robustness}.

\textbf{Bernoulli KL-Divergence.} For Bernoulli-distributed observations with mean $\mu \in (0,1)$, the KL-divergence from $\mu$ to $\lambda$ is
\begin{equation}
\mathrm{d}_{\text{B}}(\mu, \lambda) = \mu \log\frac{\mu}{\lambda} + (1-\mu) \log\frac{1-\mu}{1-\lambda}.
\end{equation}

\textbf{Subproblem Structure.} Following the framework in Section \ref{subsec:counter_set}, the internal optimization problem \eqref{eq:internal_opt} decomposes into subproblems. When $\mathcal{C}(\boldsymbol{\mu}) \neq \emptyset$ and w.l.o.g.\ $k^*(\boldsymbol{\mu}) = 1$, we have $f^*_{\boldsymbol{\mu}}(\mathbf{w}) = \min\{f^{\mathrm{opt}}_{\boldsymbol{\mu}}(\mathbf{w}), f^{\mathrm{feas}}_{\boldsymbol{\mu}}(\mathbf{w})\}$ as given in Lemma \ref{lemma: counter_set_calculation}.

\textbf{Optimality Subproblem.} For each competing policy $k \geq 2$, the optimality subproblem \eqref{eq:f_mu_opt} under Bernoulli KL-divergence becomes:
\begin{equation}
\label{eq:bernoulli_opt_subproblem}
\begin{aligned}
f^{\mathrm{opt}(k)}_{\boldsymbol{\mu}}(\mathbf{w}) = \inf_{\boldsymbol{\lambda}} \quad &\sum_{i \in \{1,k\}} \sum_{l \in [L]} w_{i,l} \cdot \mathrm{d}_{\text{B}}(\mu_{i,l}, \lambda_{i,l}), \\
\text{s.t.} \quad & \sum_{l \in [L]} q_l \lambda_{k,l} > \sum_{l \in [L]} q_l \lambda_{1,l}, \\
& \lambda_{k,l} \geq C_{\min}, \quad \forall\, l \in [L],
\end{aligned}
\end{equation}
where the decision variables are $\boldsymbol{\lambda} = \{\lambda_{i,l}\}_{i \in \{1,k\}, l \in [L]}$.

\textbf{Feasibility Subproblem.} The feasibility subproblem \eqref{eq:f_infeas} has a closed-form solution:
\begin{equation}
f^{\mathrm{feas}}_{\boldsymbol{\mu}}(\mathbf{w}) = \min_{l \in [L]} w_{1,l} \cdot \mathrm{d}_{\text{B}}(\mu_{1,l}, C_{\min}).
\end{equation}

\textbf{Convexity Analysis.} Although the Bernoulli KL-divergence $\mathrm{d}_{\text{B}}(\mu, \lambda)$ is nonlinear in $\lambda$, the subproblem \eqref{eq:bernoulli_opt_subproblem} remains a convex optimization problem. To see this, we compute the first and second derivatives of $\mathrm{d}_{\text{B}}(\mu, \lambda)$ with respect to $\lambda$:
\begin{equation}
\frac{\partial \mathrm{d}_{\text{B}}(\mu, \lambda)}{\partial \lambda} = -\frac{\mu}{\lambda} + \frac{1-\mu}{1-\lambda} = \frac{\lambda - \mu}{\lambda(1-\lambda)},
\end{equation}
\begin{equation}
\frac{\partial^2 \mathrm{d}_{\text{B}}(\mu, \lambda)}{\partial \lambda^2} = \frac{\mu}{\lambda^2} + \frac{1-\mu}{(1-\lambda)^2} > 0, \quad \forall\, \lambda \in (0,1).
\end{equation}
Since the second derivative is strictly positive for all $\lambda \in (0,1)$, the function $\mathrm{d}_{\text{B}}(\mu, \lambda)$ is strictly convex in $\lambda$. Consequently, the objective function in \eqref{eq:bernoulli_opt_subproblem}, being a nonnegative weighted sum of strictly convex functions, is also strictly convex. Combined with linear constraints, this yields a convex optimization problem.

However, unlike the Gaussian case where $\mathrm{d}_{\text{G}}(\mu, \lambda) = (\mu - \lambda)^2 / (2\sigma^2)$ leads to quadratic programming (QP) subproblems solvable by efficient QP solvers such as Gurobi, the logarithmic terms in $\mathrm{d}_{\text{B}}$ require general nonlinear optimization solvers.

\textbf{Numerical Implementation.} We solve the Bernoulli subproblems using the interior-point optimizer IPOPT \citep{wachter2006implementation} with the following considerations:
\begin{itemize}
    \item \textbf{Domain constraints}: To avoid numerical issues at boundary values where $\mathrm{d}_{\text{B}}(\mu, \lambda) \to \infty$, we constrain $\lambda_{i,l} \in [\epsilon, 1-\epsilon]$ with $\epsilon = 10^{-6}$.
    \item \textbf{Gradient and Hessian}: The analytical gradient $\frac{\partial \mathrm{d}_{\text{B}}}{\partial \lambda} = \frac{\lambda - \mu}{\lambda(1-\lambda)}$ and Hessian $\frac{\partial^2 \mathrm{d}_{\text{B}}}{\partial \lambda^2} = \frac{\mu}{\lambda^2} + \frac{1-\mu}{(1-\lambda)^2}$ are provided to the solver for improved convergence.
    \item \textbf{Initialization}: We initialize $\lambda_{i,l} = \mu_{i,l}$ (projected to the feasible domain if necessary).
\end{itemize}

%% file: appendices/E_numerical/3_variance_constraint.tex
\subsection{$\mathcal{M}$-SBFC with Variance Constraints: Detailed Formulation}
\label{appendix:variance_constraint}

This appendix provides the detailed optimization formulations for the $\mathcal{M}$-SBFC framework (Section~\ref{subsec:first_extension}) under a representative variance-based fairness constraint.

\begin{example}
\label{example:variance_fairness}
Consider the constraint set
\begin{equation}
\label{eq:variance_constraint}
\mathcal{M}=\left\{\boldsymbol{\mu}_k\in\mathbb{R}^L: \sum_{l\in[L]}(\mu_{k,l}-\mu_k)^2\leq C_{\text{var}}\right\},
\end{equation}
where $\mu_k=\sum_{l\in[L]}q_l\mu_{k,l}$ is the overall performance and $C_{\text{var}}>0$ is the maximum allowed variance. This constraint ensures that no subpopulation's performance deviates too far from the average, promoting equity.
\end{example}

For demonstration, we specify $\mathcal{P}$ as the Gaussian family. The internal optimization subproblems for the $\mathcal{M}$-SBFC algorithm become:

\textbf{Subproblem 1} (Activate alternative policy, from \eqref{eq:f_mu_opt_ext1}):
\begin{equation}
\label{eq:variance_opt1}
\begin{aligned}
\min_{\substack{\lambda_{i,l}\in\mathbb{R}\\ i\in\{1,k\},~l\in[L]}}
&\frac{1}{2}\sum_{i\in\{1,k\}}\sum_{l\in[L]}w_{i,l}(\mu_{i,l}-\lambda_{i,l})^2,\\
\text{s.t.}\quad &\lambda_k=\sum_{l\in[L]}q_l\lambda_{k,l}\geq\lambda_1= \sum_{l\in[L]}q_l\lambda_{1,l},\\
&\sum_{l\in[L]}(\lambda_{k,l}-\lambda_k)^2\leq C_{\text{var}}.
\end{aligned}
\end{equation}

\textbf{Subproblem 2} (Make current optimal infeasible, from \eqref{eq:f_infeas_ext1}):
\begin{equation}
\label{eq:variance_opt2}
\begin{aligned}
\min_{\lambda_{1,1},\ldots,\lambda_{1,L}\in\mathbb{R}}&~\frac{1}{2}\sum_{l\in[L]}w_{1,l}(\mu_{1,l}-\lambda_{1,l})^2,\\
\text{s.t.}\quad&\sum_{l\in[L]}(\lambda_{1,l}-\lambda_1)^2\geq C_{\text{var}}.
\end{aligned}
\end{equation}

\textbf{Computational Complexity.} Problem (\ref{eq:variance_opt1}) is a \emph{convex quadratically constrained quadratic program (QCQP)}, which can be efficiently solved using standard methods such as interior-point algorithms or second-order cone programming (SOCP) reformulations \citep{boyd2004convex}. Problem (\ref{eq:variance_opt2}) is non-convex due to the reverse inequality constraint. However, it remains a QCQP with a single quadratic constraint, which admits polynomial-time solutions via semidefinite programming (SDP) relaxations or specialized algorithms for trust-region subproblems \citep{boyd2004convex}.

This example demonstrates that the $\mathcal{M}$-SBFC framework maintains computational tractability even with non-trivial constraint structures, provided the constraint sets $\mathcal{M}$ are defined by polynomial inequalities.

%% file: appendices/E_numerical/4_soft_constraint.tex
\subsection{$\gamma$-SBFC with Soft Constraints: MIQP Formulation}
\label{appendix:soft_constraint}

This appendix provides the detailed mixed-integer reformulation of the $\gamma$-SBFC internal optimization (Section~\ref{subsec:second_extension}) under the Gaussian family.

\textbf{MIQP Reformulation.} The internal optimization subproblem (\ref{eq:f_mu_opt_ext2}) involves the constraint $\lambda_k^\gamma\geq\lambda_1^\gamma$, where the penalized performance contains the non-smooth $\min(\cdot,0)$ function. For the Gaussian family, this problem can be reformulated as a \emph{mixed-integer quadratic program (MIQP)} by introducing auxiliary variables $s_{i,l}$ to represent $\min(\lambda_{i,l}-C_{\min},0)$ and binary variables $z_{i,l}\in\{0,1\}$, with standard Big-M linearization constraints:
\begin{equation}
\label{eq:gamma_miqp}
\begin{aligned}
\min~&\sum_{i\in\{1,k\}}\sum_{l\in[L]}w_{i,l}(\mu_{i,l}-\lambda_{i,l})^2,\\
\text{s.t.}~&\sum_{l=1}^Lq_l(\lambda_{k,l}+\gamma_ls_{k,l})\geq\sum_{l=1}^Lq_l(\lambda_{1,l}+\gamma_ls_{1,l}),\\
&s_{i,l}\leq 0,\quad s_{i,l}\leq\lambda_{i,l}-C_{\min},\\
&s_{i,l}\geq\lambda_{i,l}-C_{\min}-Mz_{i,l},\quad s_{i,l}\geq -M(1-z_{i,l}),\\
&z_{i,l}\in\{0,1\},\quad\forall~i\in\{1,k\},~l\in[L],
\end{aligned}
\end{equation}
where the decision variables are $\lambda_{i,l}\in\mathbb{R}$, $s_{i,l}\in\mathbb{R}$, and $z_{i,l}\in\{0,1\}$ for $i\in\{1,k\}$ and $l\in[L]$. Problem (\ref{eq:gamma_miqp}) is a convex MIQP with $2L$ binary variables, which can be efficiently solved using standard solvers such as Gurobi.

%% file: appendices/E_numerical/5_asymptotic_details.tex
\subsection{Detailed Results for the Asymptotic Scaling Experiment (Section~\ref{subsubsec:asymptotic_optimality})}
\label{appendix:asymptotic_details}

This appendix provides supplementary numerical results for the asymptotic scaling experiment described in Section~\ref{subsubsec:asymptotic_optimality}. We present two complementary analyses: first, the full table of stopping times across the $\delta$-scaling sweep (Appendix~\ref{appendix:asymptotic_scaling_details}); second, an instance-sensitivity study that varies the feasibility gap to examine how the algorithm tracks the theoretical complexity constant $T^*(\boldsymbol{\mu})$ (Appendix~\ref{appendix:instance_sensitivity}).

\subsubsection{Problem Instance Detail}
\label{appendix: scaling_problem_instance}
The instance matrix is given in \ref{table:asymptotic_instance}: policy $a_1$ is the optimal feasible policy ($\bar{\mu}_1 = 0.70$); policy $a_2$ is feasible but suboptimal (gap $\Delta_{\mathrm{opt}} = 0.20$); policy $a_3$ is infeasible due to $\mu_{3,1} = -0.6 < 0$ but has high overall performance; policies $a_4$ and $a_5$ are easily eliminated (deeply infeasible and dominated, respectively).

\begin{table}[ht]
\centering
\small
\renewcommand{\arraystretch}{1.2}
\begin{tabular}{ccccccl}
\toprule
Policy & $\mu_{k,1}$ & $\mu_{k,2}$ & $\mu_{k,3}$ & $\bar{\mu}_k$ & Feasible & Role \\
\midrule
$a_1$ & 0.8 & 0.6 & 0.7 & 0.70 & \checkmark & Optimal feasible \\
$a_2$ & 0.5 & 0.6 & 0.4 & 0.50 & \checkmark & Suboptimal feasible \\
$a_3$ & $-0.6$ & 1.2 & 1.1 & $0.57$ & $\times$ & Infeasible, high performance \\
$a_4$ & $-1.5$ & 0.8 & 0.9 & 0.07 & $\times$ & Deeply infeasible \\
$a_5$ & 0.4 & 0.2 & 0.1 & 0.23 & \checkmark & Feasible but dominated \\
\bottomrule
\end{tabular}
\caption{Mean matrix for the asymptotic scaling experiment.}
\label{table:asymptotic_instance}
\end{table}

\subsubsection{$\delta$-Scaling Stopping Times}
\label{appendix:asymptotic_scaling_details}

Table~\ref{table:exp_a_results} reports the empirical stopping time $\hat{\tau}_\delta$, its standard deviation, the normalized ratio $\hat{\tau}_\delta / \ln(1/\delta)$, and the empirical correctness probability $\hat{P}_{\boldsymbol{\mu}}$ for each value of $\delta$ tested.

As $\delta$ decreases from 0.2 to 0.0005, the normalized ratio $\hat{\tau}_\delta / \ln(1/\delta)$ decreases from approximately 166 to 101, progressively approaching the asymptotic constant $T^*(\boldsymbol{\mu}) \approx 57.32$. All configurations achieve $\hat{P}_{\boldsymbol{\mu}} \ge 0.9997$, confirming the over-conservativeness discussed in Section~\ref{subsubsec:asymptotic_finite_sample}: the stopping threshold ensures a much lower error probability than the nominal $\delta$, at the cost of additional samples.

\begin{table}[ht]
\centering
\small
\renewcommand{\arraystretch}{1.2}
\begin{tabular}{ccccc}
\toprule
$\delta$ & $\hat{\tau}_\delta$ (mean $\pm$ SE) & Std & $\hat{\tau}_\delta / \ln(1/\delta)$ & $\hat{P}_{\boldsymbol{\mu}}$ \\
\midrule
0.200 & $267 \pm 2.9$ & 156.3 & 166.12 & 1.0000 \\
0.103 & $311 \pm 3.2$ & 174.8 & 136.53 & 1.0000 \\
0.053 & $360 \pm 3.4$ & 187.9 & 122.28 & 1.0000 \\
0.027 & $415 \pm 3.7$ & 203.5 & 115.08 & 1.0000 \\
0.014 & $468 \pm 3.8$ & 210.3 & 109.56 & 0.9997 \\
0.007 & $525 \pm 4.2$ & 229.5 & 106.40 & 1.0000 \\
0.004 & $587 \pm 4.3$ & 236.4 & 104.82 & 1.0000 \\
0.002 & $636 \pm 4.6$ & 253.5 & 101.48 & 1.0000 \\
0.001 & $693 \pm 4.8$ & 265.2 & 99.97 & 1.0000 \\
0.0005 & $766 \pm 5.0$ & 274.0 & 100.73 & 1.0000 \\
\bottomrule
\end{tabular}
\caption{Empirical stopping times across precision levels for the asymptotic scaling experiment (Section~\ref{subsubsec:asymptotic_optimality}). The asymptotic constant is $T^*(\boldsymbol{\mu}) \approx 57.32$; the ratio $\hat{\tau}_\delta / \ln(1/\delta)$ approaches this value as $\delta \to 0$. SE denotes standard error; Std denotes standard deviation across 3000 replications.}
\label{table:exp_a_results}
\end{table}

\subsubsection{Instance-Specific Complexity Sensitivity}
\label{appendix:instance_sensitivity}

This section complements the $\delta$-scaling analysis by providing evidence along a second dimension: the algorithm's ability to track the instance-specific constant $T^*(\boldsymbol{\mu})$ as the problem difficulty varies.

	extbf{Setup.} We use the same base instance described in Table~\ref{table:asymptotic_instance}, fixing $\delta = 0.01$ and varying the feasibility gap by setting $\mu_{3,1} = -\varepsilon$ for $\varepsilon \in \{0.1, 0.3, 0.6, 0.9, 1.2\}$, with 3000 replications per instance. As $\varepsilon$ increases, the KL divergence $d(\mu_{3,1}, C_{\min})$ grows, making feasibility verification easier and driving $T^*(\boldsymbol{\mu}, \varepsilon)$ downward. The optimality gap ($a_1$ vs.\ $a_2$) is held fixed, so the reduction is driven purely by the feasibility component.

\textbf{Results.} Figure~\ref{fig:exp_b_tracking} plots the empirical ratio $\hat{\tau}_\delta / \ln(1/\delta)$ against the theoretical $T^*(\boldsymbol{\mu}, \varepsilon)$ for each instance. The empirical ratios exceed $T^*$ by a factor of roughly $1.7$--$1.9\times$, consistent with the finite-sample conservativeness discussed in Section~\ref{subsubsec:asymptotic_finite_sample} and observed in the $\delta$-scaling experiment. Crucially, the points follow the same monotone trend as $T^*(\boldsymbol{\mu})$: instances with larger $T^*$ consistently incur larger empirical stopping times, confirming that the algorithm tracks the instance-specific difficulty across varying $\varepsilon$. Table~\ref{table:exp_b_results} provides numerical details. 

\begin{figure}[htbp]
  \centering
  \includegraphics[width=0.55\textwidth]{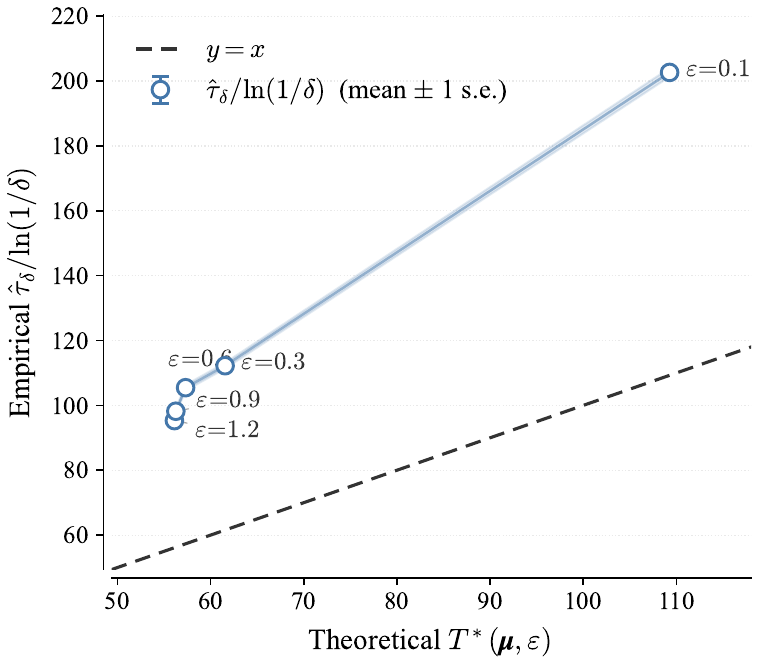}
  \caption{Empirical $\hat{\tau}_\delta / \ln(1/\delta)$ versus theoretical $T^*(\boldsymbol{\mu}, \varepsilon)$ for $\varepsilon \in \{0.1, 0.3, 0.6, 0.9, 1.2\}$ at $\delta = 0.01$. The dashed line is $y = x$. Error bars show $\pm 1$ SE over 3000 replications.}
  \label{fig:exp_b_tracking}
\end{figure}

\begin{table}[ht]
\centering
\small
\renewcommand{\arraystretch}{1.2}
\begin{tabular}{ccccc}
\toprule
$\varepsilon$ & $\mu_{3,1}$ & $T^*(\boldsymbol{\mu},\varepsilon)$ & $\hat{\tau}_\delta / \ln(1/\delta)$ (mean $\pm$ SE) & Ratio to $T^*$ \\
\midrule
0.1 & $-0.1$ & 109.27 & $202.66 \pm 1.34$ & 1.85 \\
0.3 & $-0.3$ & 61.56 & $112.23 \pm 0.90$ & 1.82 \\
0.6 & $-0.6$ & 57.32 & $105.46 \pm 0.87$ & 1.84 \\
0.9 & $-0.9$ & 56.28 & $98.20 \pm 0.86$ & 1.75 \\
1.2 & $-1.2$ & 56.12 & $95.31 \pm 0.83$ & 1.70 \\
\bottomrule
\end{tabular}
\caption{Theoretical $T^*(\boldsymbol{\mu},\varepsilon)$ versus empirical $\hat{\tau}_\delta / \ln(1/\delta)$ at $\delta = 0.01$ for varying feasibility gap $\varepsilon$.}
\label{table:exp_b_results}
\end{table}

%% file: appendices/E_numerical/6_extension_details.tex
\subsection{Detailed Results for the Hard vs.\ Soft Constraint Experiment (Section~\ref{subsubsec:extension_comparison})}
\label{appendix:extension_details}

This appendix provides supplementary materials for the hard vs.\ soft constraint comparison described in Section~\ref{subsubsec:extension_comparison}. We present the full problem instance specification (Appendix~\ref{appendix:extension_instance}) and the empirical sampling allocations (Appendix~\ref{appendix:extension_weights}).

\subsubsection{Problem Instance}
\label{appendix:extension_instance}

Table~\ref{table:extension_instance} reports the mean matrix for the hard vs.\ soft constraint comparison experiment. We consider $K=10$ policies and $L=3$ subpopulations with equal weights $q_\ell = 1/3$. The fairness constraint requires $\mu_{k,\ell} \geq C_{\min} = 0.2$. Of the ten policies, five are feasible (policies 1--3, 9, 10) and five are infeasible (policies 4--8). Policy~1 is the optimal feasible policy with $\bar\mu_1 = 0.550$. Policy~4 achieves the highest overall performance $\bar\mu_4 = 0.700$ but is infeasible due to $\mu_{4,1} = 0.05 < 0.2$. Observations are Gaussian with $\sigma = 0.5$.

\begin{table}[ht]
\centering
\small
\renewcommand{\arraystretch}{1.2}
\begin{tabular}{cccccl}
\toprule
Policy & $\mu_{k,1}$ & $\mu_{k,2}$ & $\mu_{k,3}$ & $\bar{\mu}_k$ & Role \\
\midrule
$a_1$ & 0.55 & 0.60 & 0.50 & 0.550 & Optimal feasible \\
$a_2$ & 0.35 & 0.40 & 0.30 & 0.350 & Feasible suboptimal \\
$a_3$ & 0.30 & 0.35 & 0.25 & 0.300 & Feasible suboptimal \\
$a_4$ & 0.05 & 1.05 & 1.00 & 0.700 & Infeasible, highest $\bar\mu$ \\
$a_5$ & 0.00 & 0.80 & 0.75 & 0.517 & Infeasible, moderate $\bar\mu$ \\
$a_6$ & 0.05 & 0.40 & 0.35 & 0.267 & Infeasible, low $\bar\mu$ \\
$a_7$ & 0.00 & 0.35 & 0.30 & 0.217 & Infeasible, low $\bar\mu$ \\
$a_8$ & 0.10 & 0.30 & 0.25 & 0.217 & Infeasible, low $\bar\mu$ \\
$a_9$ & 0.25 & 0.30 & 0.20 & 0.250 & Feasible, dominated \\
$a_{10}$ & 0.20 & 0.25 & 0.25 & 0.233 & Feasible, dominated \\
\bottomrule
\end{tabular}
\caption{Mean matrix for the hard vs.\ soft constraint comparison ($K=10$, $L=3$, $C_{\min}=0.2$, $\sigma=0.5$).}
\label{table:extension_instance}
\end{table}

\subsubsection{Weight Allocation}
\label{appendix:extension_weights}

Figure~\ref{fig:weights_extension} presents the empirical sampling proportions for the SBFC and $\gamma$-SBFC formulations described in Section~\ref{subsubsec:extension_comparison}.

\begin{figure}[ht]
    \centering
    \includegraphics[width=0.75\textwidth]{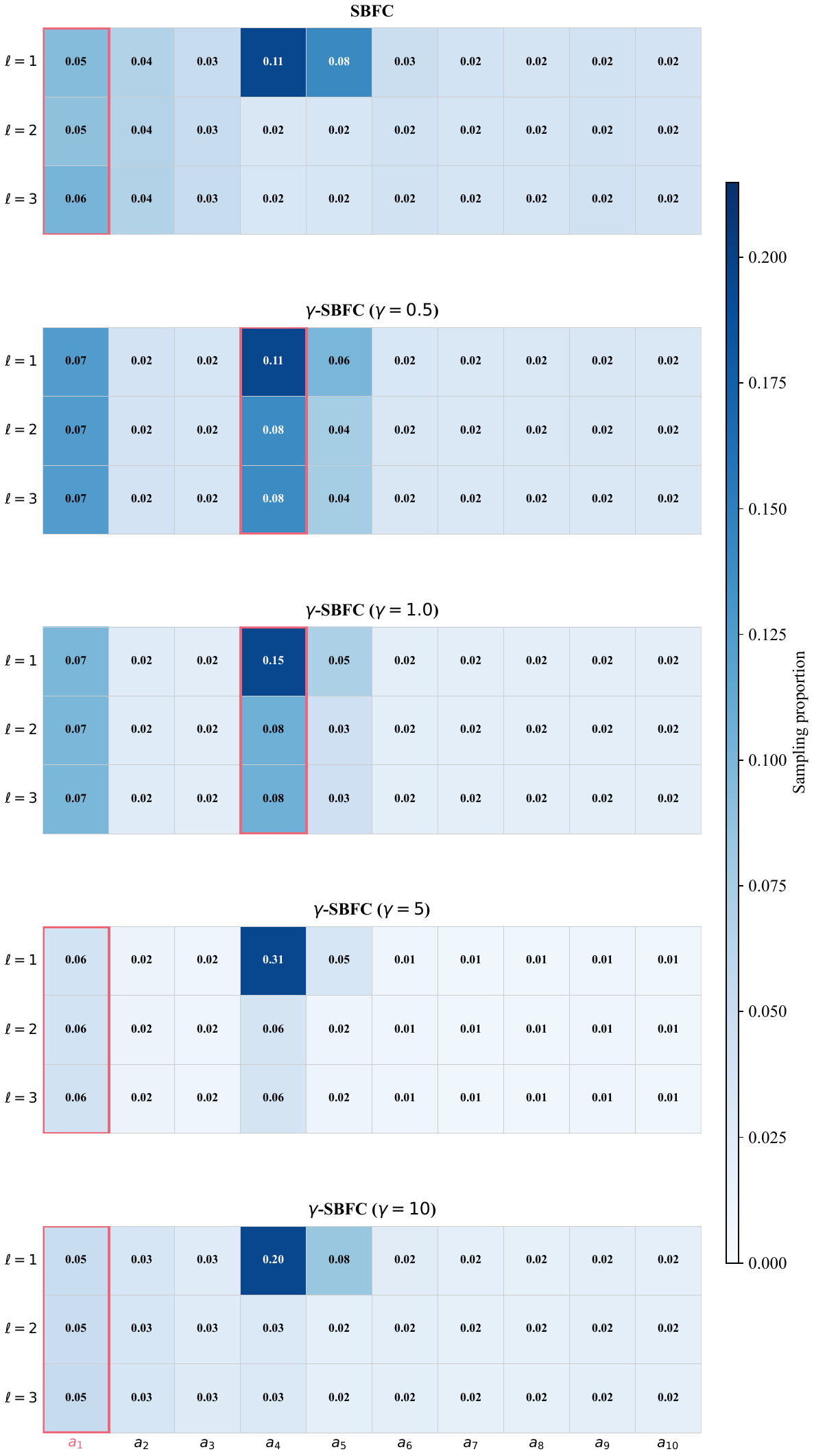}
    \caption{Empirical sampling proportions by policy and subpopulation for SBFC and $\gamma$-SBFC formulations ($\delta = 0.2$). For $\gamma \in \{0.5, 1.0\} < \gamma^*$, the optimal penalized policy is $a_4$ (infeasible under hard constraints); for $\gamma \in \{5, 10\} > \gamma^*$, it is $a_1$ (same as SBFC).}
    \label{fig:weights_extension}
\end{figure}

%% file: appendices/F_real_scenario/IST_supplementary.tex
\section{Supplementary Details for the IST Case Study}
\label{appendix:ist}

This appendix documents the data construction, cell-level summary statistics, feasibility gaps, allocation trajectory, and sensitivity analyses supporting the IST case study in Section~\ref{subsec:real_scenario}.

\subsection{Data source}
\label{appendix:ist_data}
We use the publicly released subject-level dataset from the International Stroke Trial (IST)~\citep{ist1997lancet,sandercock2011ist}, which enrolled $N = 19{,}435$ patients with acute ischaemic stroke across 36 countries between 1991 and 1996. The dataset contains randomisation variables, pre-treatment covariates, in-hospital events, and 14-day outcomes. We work with the full set of randomised patients; no exclusions are applied.

\subsection{Subpopulation definitions}
Subpopulations are defined by age at randomisation (variable \texttt{AGE} in the IST release), giving $L = 2$: Age~$<$~80 and Age~$\geq$~80. Age is recorded prior to randomisation, so the two subpopulations are well-defined pre-treatment strata satisfying the framework's assumption. Subpopulation weights are $q_1 = 0.7177$ and $q_2 = 0.2823$, matching sample proportions in the full cohort.

\subsection{Outcome construction}
The binary outcome $X_{k,\ell} \in \{0, 1\}$ is a 14-day composite: $X_{k,\ell} = 1$ if and only if the patient is alive and free of all four adverse events recorded in the trial's 14-day follow-up:
\begin{itemize}
    \item \texttt{ID14} $= 0$: not dead by day 14,
    \item \texttt{ISC14} $= 0$: no recurrent ischaemic stroke within 14 days,
    \item \texttt{H14}  $= 0$: no haemorrhagic stroke within 14 days,
    \item \texttt{TRAN14} $= 0$: no symptomatic pulmonary embolism / transfusion-requiring bleed,
    \item \texttt{NCB14} $= 0$: no non-cerebral major bleed within 14 days.
\end{itemize}
This composite integrates survival, stroke-recurrence avoidance, and bleeding avoidance on a 14-day horizon compatible with sequential sampling.

\subsection{Cell-level summary statistics}
\label{appendix:ist_cells}
Table~\ref{table:ist_cells_appendix} reports, for each $(k, \ell)$ cell, the pool size $N_{k,\ell}$, the cell mean $\hat\mu_{k,\ell}$, and the implied cell-level standard deviation $\hat\sigma_{k,\ell} = \sqrt{\hat\mu_{k,\ell}(1 - \hat\mu_{k,\ell})}$. Pool sizes are large---at least $1{,}363$ per cell---so cell-level empirical distributions are well-estimated and the bootstrap sampling model introduces no material additional noise.

\begin{table}[h]
\centering
\small
\renewcommand{\arraystretch}{1.15}
\begin{tabular}{lcccc}
\toprule
 & \multicolumn{2}{c}{Age $<$ 80} & \multicolumn{2}{c}{Age $\geq$ 80} \\
\cmidrule(lr){2-3}\cmidrule(lr){4-5}
Policy & $N_{k,1}$ & $\hat\mu_{k,1}\,(\hat\sigma_{k,1})$ & $N_{k,2}$ & $\hat\mu_{k,2}\,(\hat\sigma_{k,2})$ \\
\midrule
Aspirin only & $3470$ & $0.8925\;(0.3098)$ & $1388$ & $0.8429\;(0.3640)$ \\
Heparin only & $3484$ & $0.8751\;(0.3306)$ & $1371$ & $0.7980\;(0.4015)$ \\
Both         & $3499$ & $0.8705\;(0.3357)$ & $1363$ & $0.7843\;(0.4113)$ \\
Neither      & $3496$ & $0.9013\;(0.2983)$ & $1364$ & $0.8248\;(0.3801)$ \\
\bottomrule
\end{tabular}
\caption{Cell-level pool sizes and Bernoulli summary statistics on IST. Total sample sizes: $13{,}949$ (Age~$<$~80) and $5{,}486$ (Age~$\geq$~80).}
\label{table:ist_cells_appendix}
\end{table}

\begin{figure}[ht]
\centering
\includegraphics[width=\linewidth]{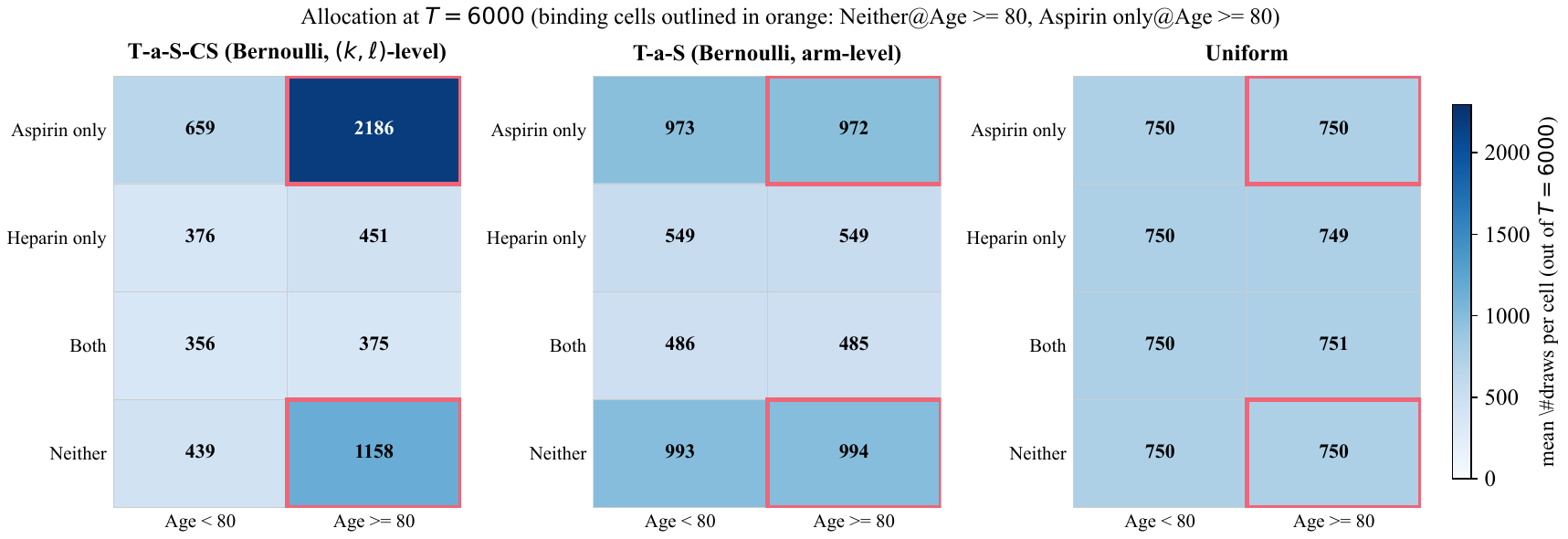}
\caption{Mean sample allocations. T-a-S-CS significantly oversamples the two critical subpopulation cells (outlined in orange) where the feasibility threshold is tightly constrained.}
\label{fig:ist_allocation}
\end{figure}

\subsection{Scope and limitations}
Three caveats. (i) The IST outcome is a 14-day composite; longer-horizon outcomes (6-month disability, mortality at one year) exhibit different signal-to-noise ratios and may yield different feasibility patterns. (ii) Age $\geq 80$ is one of several possible vulnerable subgroups; stratifying instead by baseline severity (\texttt{RCONSC}, consciousness level at randomisation) or by atrial-fibrillation status produces different Case/feasibility structures. (iii) Bootstrap-with-replacement preserves the marginal cell distribution exactly but does not capture any hospital/country clustering in the original trial; a multilevel generative model would be a natural extension. None of these caveats affects the algorithmic conclusion---subpopulation-level allocation dominates both policy-level adaptive allocation and uniform fixed allocation at moderate budgets on genuine clinical data in the Case \ref{example:2} regime.